\documentclass[twoside,11pt]{article}
\usepackage{blindtext}
\usepackage{bm}
\usepackage{graphicx}
\usepackage{tabularx}
\usepackage{mathrsfs}%
\usepackage{enumitem}
\usepackage{cite}
\usepackage{algorithm}%
\usepackage{algorithmicx}%
\usepackage{algpseudocode}
\usepackage[colorlinks=true,linkcolor=blue]{hyperref}
\usepackage{float}
\usepackage{makecell}
\usepackage{amsmath,amsthm, amssymb, graphicx, caption, subcaption}  % 数学公式、图形支

\usepackage{amsfonts}
\usepackage{booktabs}  % 用于三线表
\usepackage{geometry}
\usepackage{hyperref}
\newtheorem{definition}{Definition}
\newtheorem{theorem}{Theorem}[section]
\newtheorem{proposition}{Proposition}
\numberwithin{definition}{section}
\numberwithin{equation}{section}
\numberwithin{proposition}{section}
\newtheorem{assumption}[theorem]{Assumption}
\usepackage{jmlr2e}
\newcommand{\RR}{\mathbb{R}}

\newcommand{\mc}{\mathcal}

\newcommand{\abs}[1]{\left|#1\right|}

\newcommand{\inner}[2]{\langle #1, #2 \rangle_{L^2}}
\newcommand{\innerh}[2]{\langle #1, #2 \rangle_h} % 离散内积

\newcommand{\dH}{\Delta_H}
\newcommand{\Om}[1]{\Omega^{#1}}
\newcommand{\lr}[1]{\left(#1\right)}

\newcommand{\To}{\longrightarrow}
\newcommand{\upuparrow}{\uparrow\!\!\uparrow} % 严格递增符号
\newcommand{\embeds}{\hookrightarrow}
\newcommand{\spec}{\mathop{\mathrm{spec}}\nolimits} % 谱
\newcommand{\dist}{\mathop{\mathrm{dist}}\nolimits} % 距离
\newcommand{\T}{\mathrm{T}} % 转置
\newcommand{\Ogroup}{\mathop{\mathrm{O}}\nolimits} % 正交群

\newcommand{\dhodge}[1]{\Delta_{H,#1}} % 带度量的Hodge拉普拉斯（修正为小写命令名）
\newcommand{\Ric}[1]{\mathrm{Ric}_{#1}} % 里奇曲率
\newcommand{\ceil}[1]{\lceil #1 \rceil} % 上取整
\newcommand{\CC}{\mathbb{C}} % 复数域
\newcommand{\len}{\mathop{\mathrm{len}}\nolimits} % 曲线长度
\newcommand{\Lip}[1]{\mathrm{Lip}(#1)} % 利普希茨常数
\newcommand{\sintheta}{\sin\!(\Theta)} % 间隙度量
\newcommand{\Omh}[1]{\Omega_h^{#1}} % 离散p-形式空间
\newcommand{\Id}{\mathrm{Id}} % 单位算子
\newcommand{\ran}{\mathop{\mathrm{ran}}\nolimits} % 值域
\newcommand{\R}{\mathrm{O}} % 正交群
\newcommand{\E}{\mathbb{E}} % 期望
\newenvironment{discussion}{\paragraph{Discussion.}}{\par}
\usepackage{lastpage}
\jmlrheading{23}{}{1-\pageref{LastPage}}{; Revised }{}{21-0000}{Author One and Author Two}

\ShortHeadings{Sample JMLR Paper}{One}
\firstpageno{1}

\begin{document}

\title{Gauge-Equivariant Intrinsic Neural Operators for Geometry-Consistent Learning of Elliptic PDE Maps}

\author{\name Pengcheng Cheng \email chengpc1022@mails.jlu.edu.cn \\
       \addr School of Mathematics\\
       Jilin University\\
       Changchun, 130012, China}

\editor{My editor}

\maketitle

\begin{abstract}%   <- trailing '%' for backward compatibility of .sty file
Learning solution operators of partial differential equations (PDEs) from data has emerged as a promising route to fast surrogate models in multi-query scientific workflows. However, for geometric PDEs whose inputs and outputs transform under changes of local frame (gauge), many existing operator-learning architectures remain representation-dependent, brittle under metric perturbations, and sensitive to discretization changes. We propose Gauge-Equivariant Intrinsic Neural Operators (GINO), a class of neural operators that parameterize elliptic solution maps primarily through intrinsic spectral multipliers acting on geometry-dependent spectra, coupled with gauge-equivariant nonlinearities. This design decouples geometry from learnable functional dependence and enforces consistency under frame transformations. We validate GINO on controlled problems on the flat torus ($\mathbb{T}^2$), where ground-truth resolvent operators and regularized Helmholtz--Hodge decompositions admit closed-form Fourier representations, enabling theory-aligned diagnostics. Across experiments E1--E6, GINO achieves low operator-approximation error, near machine-precision gauge equivariance, robustness to structured metric perturbations, strong cross-resolution generalization with small commutation error under restriction/prolongation, and structure-preserving performance on a regularized exact/coexact decomposition task. Ablations further link the smoothness of the learned spectral multiplier to stability under geometric perturbations. These results suggest that enforcing intrinsic structure and gauge equivariance yields operator surrogates that are more geometry-consistent and discretization-robust for elliptic PDEs on form-valued fields.
\end{abstract}


\section{Introduction}
\label{sec:introduction}
Partial differential equations (PDEs) govern a wide range of phenomena in physics, engineering, and the natural sciences. In many applications, the central computational task is not the solution of a single PDE instance, but the repeated evaluation of a solution operator that maps an input field (e.g., a forcing term, boundary data, coefficients) to an output field (e.g., a state, potential, or flux). Classical numerical solvers—finite differences, finite volumes, finite elements, and spectral methods—offer systematic accuracy and stability guarantees, but can be computationally expensive when embedded in design loops, inverse problems, uncertainty quantification, or multi-query regimes \cite{Hughes2000,Brenner2008,Temam2021}. This has motivated substantial interest in scientific machine learning approaches that learn surrogate models for PDE solution maps from data, including physics-informed methods \cite{Raissi2019,Karniadakis2021,Brunton2020} and operator-learning methods \cite{Lu2021, Li2021, Kovachki2021}.

Among data-driven approaches, neural operators aim to learn maps between infinite-dimensional function spaces, enabling discretization-independent inference in principle \cite{Lu2021,Li2021,Kovachki2021}. Representative examples include DeepONet \cite{Lu2021}, the Fourier Neural Operator (FNO) \cite{Li2021}, and related integral-operator architectures \cite{Kovachki2021}. These models have demonstrated impressive performance on benchmark PDE families, including fluid dynamics and diffusion-type problems \cite{Brunton2020,Stachenfeld2020,Takamoto2022}. However, many existing operator-learning architectures are trained and evaluated on fixed grids with coordinate-dependent representations. As a result, their behavior can degrade under changes of discretization (cross-resolution evaluation) and under geometric changes (e.g., varying metric, mesh, or coordinate charts), and they often lack built-in guarantees of equivariance with respect to representation choices.

A core challenge is that many physically meaningful PDEs are geometric: their inputs/outputs may be vector fields or differential forms, and the governing operators depend on the underlying Riemannian metric through objects such as the Laplace--Beltrami operator and the Hodge star. In such settings, the same geometric field can be represented by different component vectors under different choices of local frame or gauge. Standard convolutional networks, including coordinate-augmented CNNs, are typically not designed to respect these symmetries; their outputs may vary substantially under benign changes of representation, even when the underlying geometric field is unchanged. This issue is closely related to the broader theme of geometric deep learning, which emphasizes architectures that respect symmetry and structure—on groups, manifolds, and meshes—through equivariance principles \cite{Bronstein2021,Cohen2016,Cohen2019}. In Euclidean settings, group-equivariant CNNs and steerable CNNs have yielded strong gains in sample efficiency and generalization by hard-coding symmetry \cite{Cohen2017,Worrall2017,Weiler2019,Bekkers2020}. More recently, gauge-equivariant CNNs extend equivariance to settings where features transform under spatially varying frames, providing a principled way to handle fields on manifolds and meshes \cite{Cohen2019,Kohler2020,Kondor2018,deHaan2020}. In parallel, equivariant neural networks for 3D geometry have matured rapidly, including tensor field networks and SE(3)-equivariant architectures \cite{Thomas2018,Fuchs2020,Geiger2022,Satorras2021}, reinforcing the practical value of enforcing symmetry at the architectural level.

In this work we focus on a complementary question: how to design an operator-learning architecture that is intrinsic and gauge-consistent, and that remains stable across geometric and discretization changes. Our approach is guided by classical geometric analysis. For elliptic operators, the spectrum of the Laplace-type operator encodes both geometry and scale, and solution operators can be expressed via spectral multipliers (e.g., resolvents and Green’s functions). Meanwhile, the Helmholtz--Hodge decomposition provides a canonical splitting of vector fields (or 1-forms) into exact, coexact, and harmonic components, underpinning stable formulations in both continuous and discrete settings \cite{Jost2008,Warner1983,Desbrun2008,Hirani2003,Arnold2006,Arnold2010}. Discrete exterior calculus (DEC) and finite element exterior calculus (FEEC) make these ideas computationally concrete, yielding discretizations that preserve topological and geometric structure \cite{Desbrun2008,Hirani2003,Arnold2006,Arnold2010}.

Building on these insights, we introduce a Gauge-Equivariant Intrinsic Neural Operator (GINO) that parameterizes the learned operator primarily through intrinsic spectral multipliers of the form $m_\theta(\lambda)$, where $\lambda$ denotes a geometry-dependent spectral quantity (e.g., $\lambda_g(k)=k^\top A k$ for constant-metric problems). This design explicitly separates (i) geometry-dependent spectral structure from (ii) learnable functional dependence, and we couple it with a gauge-equivariant nonlinearity to ensure consistent transformation under frame changes. The resulting model is intended to behave like a geometry-aware, discretization-consistent surrogate for elliptic solution operators on fields with nontrivial transformation laws.

We validate the approach in a controlled setting where the ground-truth operator is available in closed form in the Fourier domain, enabling diagnostics that directly probe the theory. Across experiments E1–E6, we show that: (i) GINO matches the target operator accurately; (ii) it achieves gauge equivariance up to numerical precision, while coordinate-based CNN baselines exhibit large gauge sensitivity; (iii) it remains stable under structured metric perturbations; (iv) it generalizes across resolutions and approximately commutes with restriction/prolongation operators, indicating reduced discretization dependence; (v) it supports a structure-preserving (regularized) Hodge decomposition task; and (vi) ablations link multiplier smoothness to stability, aligning with theoretical expectations about spectral regularity and perturbation amplification.

We propose an intrinsic operator-learning architecture that parameterizes elliptic solution maps via learned spectral multipliers coupled with gauge-equivariant nonlinearities, aimed at representation- and discretization-consistent behavior. We design a suite of theory-aligned diagnostics for gauge equivariance, metric-perturbation stability, and discretization commutation, enabled by exact Fourier-domain ground truth. We demonstrate strong empirical advantages over coordinate-based CNN baselines on accuracy, gauge consistency, metric stability, and cross-resolution generalization, and we provide ablations linking multiplier roughness to perturbation amplification.

\section{Problem Setup and Target Operator}
\label{2}

\subsection{Geometric and Functional-Analytic Setting}
\label{subsec:geometric_setting}

Let  $ (\mathcal{M}, g) $  be a connected, compact, oriented,  $ d $ -dimensional smooth Riemannian manifold without boundary. We denote by  $ \Omega^1(\mathcal{M}) $  the space of smooth differential  $ 1 $ -forms on  $ \mathcal{M} $ . The Riemannian metric  $ g $  induces a pointwise inner product  $ \langle \cdot, \cdot \rangle_g $  on  $ 1 $ -forms and the associated  $ L^2 $  inner product
$$
\langle \omega, \eta \rangle_{L^2} := \int_{\mathcal{M}} \langle \omega(x), \eta(x) \rangle_g \, d\mathrm{vol}_g(x),
\qquad \omega, \eta \in \Omega^1(\mathcal{M}),
$$
with norm  $ \|\omega\|_{L^2}^2 = \langle \omega, \omega \rangle_{L^2} $ . We write  $ H^s\Omega^1(\mathcal{M}) $  for the Sobolev space of  $ 1 $ -forms of order  $ s \in \mathbb{R} $ , defined via local charts and a partition of unity (equivalently, through the functional calculus of a fixed elliptic operator on  $ 1 $ -forms). When no ambiguity arises, we abbreviate  $ H^s\Omega^1(\mathcal{M}) $  as  $ H^s\Omega^1 $ .

We denote by  $ d : \Omega^k(\mathcal{M}) \to \Omega^{k+1}(\mathcal{M}) $  the exterior derivative and by  $ \delta : \Omega^{k+1}(\mathcal{M}) \to \Omega^k(\mathcal{M}) $  its  $ L^2 $ -adjoint (the codifferential). The (Hodge) Laplacian on  $ 1 $ -forms is
$$
\Delta_H := d\delta + \delta d : \Omega^1(\mathcal{M}) \to \Omega^1(\mathcal{M}).
$$
On a compact manifold without boundary,  $ \Delta_H $  is essentially self-adjoint, nonnegative, and has discrete spectrum. Throughout, we use the same notation for its continuous extension to Sobolev spaces.

To make the inverse problem well-posed without explicitly quotienting out 
harmonic  $ 1 $ -forms, we introduce a positive shift parameter  $ \alpha > 0 $ . Define the shifted operator
$$
\mathcal{L}_g := \Delta_H + \alpha I,
$$
where  $ I $  denotes the identity on  $ 1 $ -forms. The shift ensures  $ \mathcal{L}_g $  is strictly positive and invertible on  $ L^2\Omega^1 $ , and improves conditioning both analytically and numerically.

Finally, since vector fields are often the physically meaningful objects, we note that  $ g $  provides an isomorphism between  $ 1 $ -forms and vector fields via the musical isomorphisms  $ \flat, \sharp $ . All statements below may be equivalently read in the language of tangent vector fields, but we adopt the  $ 1 $ -form formulation because it is naturally compatible with Hodge theory and with discretizations based on discrete exterior calculus.

\subsection{Elliptic PDE on 1-Forms and the Solution Operator}
\label{subsec:elliptic_pde}

We consider the following family of elliptic problems on  $ 1 $ -forms: for a given forcing  $ f \in H^{s-1}\Omega^1 $  (with  $ s \ge 0 $  typically), find  $ \omega \in H^{s+1}\Omega^1 $  such that
\begin{equation}
\mathcal{L}_g \omega = f,
\qquad \text{i.e.,} \qquad (\Delta_H + \alpha I)\omega = f.
\label{eq:elliptic_problem}
\end{equation}
Equation\eqref{eq:elliptic_problem} covers, as special cases or close relatives, the vector Laplacian and Hodge--Poisson problems arising in geometry processing, computational electromagnetism, and fluid mechanics on surfaces (e.g., via vorticity--stream formulations or Hodge decompositions). The Hodge Laplacian is the canonical second-order elliptic operator acting on differential forms, and thus provides a clean testbed for developing intrinsic, coordinate-free operator learning theory.

Standard elliptic regularity on compact manifolds implies that  $ \mathcal{L}_g $  defines a continuous isomorphism between Sobolev spaces,
$$
\mathcal{L}_g : H^{s+1}\Omega^1 \to H^{s-1}\Omega^1,
$$
and therefore admits a bounded inverse (the resolvent at  $ -\alpha $ ),
\begin{equation}
\mathcal{S}_g := \mathcal{L}_g^{-1} : H^{s-1}\Omega^1 \to H^{s+1}\Omega^1.
\label{eq:solution_operator}
\end{equation}
We refer to  $ \mathcal{S}_g $  as the (shifted) Hodge--Poisson solution operator. In this work,  $ \mathcal{S}_g $  is the primary target operator to be learned.

Spectrally, since  $ \Delta_H $  is self-adjoint and nonnegative on  $ L^2\Omega^1 $ , there exists an  $ L^2 $ -orthonormal eigenbasis  $ \{\psi_k\}_{k \ge 0} \subset \Omega^1(\mathcal{M}) $  with eigenvalues  $ 0 \le \lambda_0 \le \lambda_1 \le \cdots $ ,  $ \lambda_k \to \infty $ , such that  $ \Delta_H \psi_k = \lambda_k \psi_k $ . Consequently, for sufficiently regular  $ f $ ,
\begin{equation}
\mathcal{S}_g f = \sum_{k \ge 0} \frac{1}{\lambda_k + \alpha} \langle f, \psi_k \rangle_{L^2} \, \psi_k,
\label{eq:spectral_representation}
\end{equation}
which exhibits  $ \mathcal{S}_g $  as a spectral multiplier (a functional calculus) of  $ \Delta_H $ . This representation motivates the model class introduced later: intrinsic neural operators parameterizing spectral multipliers while preserving geometric equivariances.

\subsection{Learning Objective and Data Model}
\label{subsec:learning_objective}

We assume access to samples  $ \{(f_i, \omega_i)\}_{i=1}^n $  drawn from a distribution  $ \mathcal{D} $  over forcings  $ f $  and corresponding solutions  $ \omega = \mathcal{S}_g f $  to \eqref{eq:elliptic_problem}, where the metric  $ g $  is fixed unless stated otherwise. Concretely, the training pairs may be generated by a classical numerical solver on a discretization of  $ \mathcal{M} $ , or obtained from observations when the underlying elliptic model is appropriate.

Our goal is to learn a parameterized operator family  $ \{\widehat{\mathcal{S}}_\theta\}_{\theta \in \Theta} $  mapping  $ 1 $ -form forcings to  $ 1 $ -form solutions,
$$
\widehat{\mathcal{S}}_\theta : H^{s-1}\Omega^1 \to H^{s+1}\Omega^1,
$$
that approximates  $ \mathcal{S}_g $  in a strong, geometry-aware sense. Specifically, we target guarantees of the following form:
Operator approximation in Sobolev topology:  $ \|\widehat{\mathcal{S}}_\theta - \mathcal{S}_g\|_{H^{s-1} \to H^{s+1}} $  is small, with explicit dependence on model capacity (e.g., spectral truncation level) and parameterization error.
Gauge equivariance: the learned operator is equivariant with respect to changes of local orthonormal frames, reflecting the intrinsic nature of  $ 1 $ -forms (formally defined in Section\ref{3}).
Discretization consistency: when  $ \widehat{\mathcal{S}}_\theta $  is implemented on a mesh or point cloud via a convergent discrete Hodge Laplacian, the discrete learned operator converges to its continuum counterpart as resolution increases.

In Sections \ref{4}--\ref{6} we formalize and prove approximation, stability, and continuum--discrete consistency results under the above setup.

\section{Gauge-Equivariant Intrinsic Neural Operators}
\label{3}

This section introduces a class of intrinsic neural operators acting on  $ 1 $ -forms that (i) are defined without reference to coordinates, (ii) are equivariant to local changes of orthonormal frames (gauge transformations), and (iii) admit discretizations that commute with refinement under standard convergent discrete Hodge Laplacians. Our construction is motivated by the observation that the target solution operator  $ \mathcal{S}_g = (\Delta_H + \alpha I)^{-1} $  is a spectral multiplier of  $ \Delta_H $  (cf.\eqref{eq:spectral_representation}), and hence can be approximated by learning appropriate functional calculi of  $ \Delta_H $ .

\subsection{Gauge Structure on  $ 1 $ -Forms and an Equivariance Principle}
\label{subsec:gauge_structure}

Although a differential  $ 1 $ -form is an intrinsic geometric object, any numerical representation necessarily uses local frames. Let  $ \{U_\ell\}_{\ell=1}^L $  be an open cover of  $ \mathcal{M} $  and, on each  $ U_\ell $ , let  $ E_\ell(x) = (e_{\ell,1}(x), \dots, e_{\ell,d}(x)) $  be a smooth local orthonormal frame of  $ T_x\mathcal{M} $ . Any  $ \omega \in \Omega^1(\mathcal{M}) $  admits local coordinates  $ [\omega]_{E_\ell}(x) \in \mathbb{R}^d $  defined by
$$
[\omega]_{E_\ell}(x)_j := \omega(x)\big(e_{\ell,j}(x)\big).
$$
If the frame is changed to  $ E'_\ell(x) = E_\ell(x) R_\ell(x) $  with  $ R_\ell(x) \in O(d) $ , then the coordinate representation transforms as
\begin{equation}
[\omega]_{E'_\ell}(x) = R_\ell(x)^\top [\omega]_{E_\ell}(x).
\label{eq:gauge_transformation}
\end{equation}
This gauge action (pointwise orthogonal transformations of local frames) is unavoidable in computations on tangent/cotangent bundles. A learning architecture intended to represent intrinsic operators on  $ 1 $ -forms should respect this transformation law.

We therefore adopt the following axiom.

\begin{definition}[Gauge equivariance]
\label{def:gauge_equivariance}
An operator  $ \mathcal{T}: \Omega^1(\mathcal{M}) \to \Omega^1(\mathcal{M}) $  is \emph{gauge-equivariant} if for any choice of local orthonormal frames  $ \{E_\ell\} $  and any smooth  $ R_\ell: U_\ell \to O(d) $  defining new frames  $ E'_\ell = E_\ell R_\ell $ , the local representations satisfy
\begin{equation}
[\mathcal{T}(\omega)]_{E'_\ell}(x) = R_\ell(x)^\top [\mathcal{T}(\omega)]_{E_\ell}(x)
\quad\text{whenever}\quad
[\omega]_{E'_\ell}(x) = R_\ell(x)^\top [\omega]_{E_\ell}(x),
\label{eq:gauge_equivariance}
\end{equation}
for all  $ x \in U_\ell $ . Equivalently,  $ \mathcal{T} $  commutes with the pointwise action of  $ O(d) $  on the fiber coordinates.
\end{definition}

In words, changing the computational gauge (local orthonormal frame) should only rotate the coordinate representation of the output accordingly; the operator itself is independent of this choice.

\subsection{Intrinsic Linear Layers via Functional Calculus of  $ \Delta_H $ }
\label{subsec:intrinsic_linear_layers}

A central design choice is to build linear operator layers from intrinsic differential operators. Since  $ \Delta_H $  is self-adjoint and nonnegative on  $ L^2\Omega^1 $ , it admits a spectral decomposition  $ \Delta_H \psi_k = \lambda_k \psi_k $  with an  $ L^2 $ -orthonormal eigenbasis  $ \{\psi_k\}_{k \ge 0} $ . For any bounded measurable function  $ m: [0,\infty) \to \mathbb{R} $ , the functional calculus defines a linear operator  $ m(\Delta_H) $  by
\begin{equation}
m(\Delta_H)\omega
:= \sum_{k \ge 0} m(\lambda_k) \langle \omega, \psi_k \rangle_{L^2} \, \psi_k,
\label{eq:functional_calculus}
\end{equation}
with convergence in  $ L^2\Omega^1 $  (and in  $ H^s\Omega^1 $  under appropriate growth/regularity of  $ m $ ). This construction is coordinate-free and depends only on  $ (\mathcal{M}, g) $ .

We now define the basic learnable intrinsic linear layer as a truncated multiplier.

\begin{definition}[Truncated spectral multiplier layer]
\label{def:truncated_multiplier}
Fix a truncation index  $ K \in \mathbb{N} $  and let  $ m_\theta: [0,\infty) \to \mathbb{R} $  be a parameterized multiplier (e.g., an MLP in the scalar variable  $ \lambda $ ). The \emph{spectral multiplier layer} is
\begin{equation}
\big(\mathcal{T}_{\theta,K} \omega\big)
:= \sum_{k=0}^{K} m_\theta(\lambda_k) \langle \omega, \psi_k \rangle_{L^2} \, \psi_k.
\label{eq:truncated_multiplier}
\end{equation}
\end{definition}

The truncation serves two purposes: (i) it yields a finite-dimensional parameterization compatible with discretizations, and (ii) it isolates a controllable approximation error that can be quantified in Sobolev operator norms.

A complementary and often numerically convenient representation uses heat kernels. If  $ m $  is completely monotone or admits a Laplace transform representation  $ m(\lambda) = \int_0^\infty a(t) e^{-t\lambda} \, dt $ , then
\begin{equation}
m(\Delta_H)\omega(x) = \int_0^\infty a(t) \, (e^{-t\Delta_H}\omega)(x) \, dt,
\label{eq:heat_kernel_rep}
\end{equation}
where  $ e^{-t\Delta_H} $  is the heat semigroup on  $ 1 $ -forms. This emphasizes that our convolution is intrinsic: it is induced by the geometry through the heat flow rather than by Euclidean translation.

\subsection{Gauge-Equivariant Pointwise Nonlinearities on  $ 1 $ -Forms}
\label{subsec:equivariant_nonlinearities}

To construct expressive nonlinear operator architectures, we interleave intrinsic linear layers with pointwise nonlinearities. However, generic coordinate-wise nonlinearities (e.g., applying a scalar activation to each fiber coordinate in an arbitrary gauge) break gauge equivariance. We therefore restrict to nonlinearities that are equivariant under the  $ O(d) $  action in \eqref{eq:gauge_transformation}.

A standard characterization is that continuous  $ O(d) $ -equivariant maps  $ \varphi: \mathbb{R}^d \to \mathbb{R}^d $  are radial:
\begin{equation}
\varphi(v) = \rho(\|v\|) \, v
\quad\text{for some scalar function}\quad \rho: [0,\infty) \to \mathbb{R}.
\label{eq:radial_map}
\end{equation}
Accordingly, we define the fiberwise nonlinearity using intrinsic norms.

\begin{definition}[Fiberwise gauge-equivariant nonlinearity]
\label{def:fiberwise_nonlinearity}
Let  $ \rho_\theta: [0,\infty) \to \mathbb{R} $  be a learnable scalar function. Define  $ \sigma_\theta: \Omega^1(\mathcal{M}) \to \Omega^1(\mathcal{M}) $  by
\begin{equation}
(\sigma_\theta(\omega))(x) := \rho_\theta\big(\|\omega(x)\|_g\big) \, \omega(x),
\label{eq:fiberwise_nonlinearity}
\end{equation}
where  $ \|\omega(x)\|_g^2 = \langle \omega(x), \omega(x) \rangle_g $ . Then  $ \sigma_\theta $  is intrinsic and gauge-equivariant.
\end{definition}

In practice,  $ \rho_\theta $  can be parameterized by a small MLP, a spline, or a polynomial basis with bounded range (useful for stability). Importantly, \eqref{eq:fiberwise_nonlinearity} depends only on the norm induced by  $ g $ , hence it is invariant to local orthonormal frame rotations.

We also allow an intrinsic gated pointwise linear map that remains gauge-equivariant:
\begin{equation}
(\mathcal{B}_\theta \omega)(x) := b_\theta(x) \, \omega(x),
\label{eq:gated_linear}
\end{equation}
where  $ b_\theta: \mathcal{M} \to \mathbb{R} $  is a learned scalar field (potentially produced by another intrinsic operator acting on invariants such as  $ \|\omega\|_g $ , curvature scalars, or known metadata). Since multiplication by a scalar commutes with  $ O(d) $ , \eqref{eq:gated_linear} preserves gauge equivariance.

\subsection{The Gauge-Equivariant Intrinsic Neural Operator Architecture}
\label{subsec:architecture}

We now assemble the above ingredients into a deep operator network. Let  $ \omega_0 := f $  denote the input forcing  $ 1 $ -form. For  $ \ell = 0, \dots, L-1 $ , define the hidden state recursion
\begin{equation}
\omega_{\ell+1}
:= \sigma_{\theta_\ell} \Big( \mathcal{T}_{\theta_\ell,K} \omega_\ell + \mathcal{B}_{\theta_\ell} \omega_\ell \Big),
\label{eq:recursion}
\end{equation}
and output  $ \widehat{\mathcal{S}}_{\theta,K}(f) := \omega_L $ . Here:
\begin{itemize}
    \item  $ \mathcal{T}_{\theta_\ell,K} $  is the truncated intrinsic multiplier layer \eqref{eq:truncated_multiplier},
    \item  $ \mathcal{B}_{\theta_\ell} $  is a gauge-compatible pointwise linear map \eqref{eq:gated_linear}, and
    \item  $ \sigma_{\theta_\ell} $  is the gauge-equivariant nonlinearity \eqref{eq:fiberwise_nonlinearity}.
\end{itemize}

This architecture can be interpreted as learning a nonlinear perturbation of functional calculi of  $ \Delta_H $  while remaining intrinsic. In the elliptic setting, one may also use purely linear versions (setting  $ \sigma $  to identity) to directly approximate resolvents  $ m(\lambda) \approx (\lambda + \alpha)^{-1} $ . We keep the nonlinear form \eqref{eq:recursion} because it allows modeling operator families with parameter dependence, mild nonlinearity in constitutive relations, or data/model mismatch, while preserving the geometric symmetries.

\subsection{Intrinsicness and Gauge Equivariance: Formal Guarantees}
\label{subsec:formal_guarantees}

We record the two basic structural properties required later for stability and discretization analysis.

\begin{proposition}[Intrinsicness]
\label{prop:intrinsicness}
Let  $ (\mathcal{M}, g) $  be a compact smooth Riemannian manifold without boundary. Let  $ \varphi : \mathcal{M} \to \mathcal{M} $  be a diffeomorphism and set  $ g' := \varphi^* g $ . Consider a network operator  $ \widehat{\mathcal{S}}^{(g)}_{\theta,K} $  built by composing layers of the form
$$
\omega \mapsto \sigma_\theta \big( \mathcal{T}^{(g)}_{\theta,K} \omega + \mathcal{B}^{(g)}_\theta \omega \big),
$$
where:

\begin{enumerate}[label=(\arabic*)]
    \item Spectral layer  $ \mathcal{T}^{(g)}_{\theta,K} $  is a truncated functional calculus of the Hodge Laplacian  $ \Delta_{H,g} $  on  $ 1 $ -forms:
    \begin{equation}
    \label{A11}
        \mathcal{T}^{(g)}_{\theta,K} := m_\theta(\Delta_{H,g}) \, \Pi^{(g)}_{\le \Lambda},
    \end{equation}
    for some cutoff  $ \Lambda = \Lambda(K, g) $ , and where  $ \Pi^{(g)}_{\le \Lambda} $  is the  $ L^2 $ -orthogonal spectral projector of  $ \Delta_{H,g} $  onto the direct sum of eigenspaces with eigenvalues  $ \le \Lambda $ . (This covers top- $ K $  modes truncation up to multiplicity; it is equivalent to truncating by the first  $ K $  eigenmodes when eigenvalues are ordered.)

    \item Pointwise linear term  $ \mathcal{B}^{(g)}_\theta $  is multiplication by a scalar field  $ b^{(g)}_\theta \in C^\infty(\mathcal{M}) $ :
    \begin{equation}
    \label{B11}
        (\mathcal{B}^{(g)}_\theta \omega)(x) = b^{(g)}_\theta(x) \, \omega(x),
    \end{equation}
    and the scalar field is natural under pullback:
    \begin{equation}
    \label{C11}
        b^{(g')}_\theta := \varphi^* b^{(g)}_\theta.
    \end{equation}

    \item Nonlinearity  $ \sigma_\theta $  is fiberwise radial with respect to the metric:
    \begin{equation}
    \label{D11}
        (\sigma^{(g)}_\theta(\omega))(x) = \rho_\theta(\|\omega(x)\|_g) \, \omega(x),
    \end{equation}
    where  $ \rho_\theta : [0,\infty) \to \mathbb{R} $  is fixed and  $ |\cdot|_g $  is the pointwise norm induced by  $ g $  on  $ 1 $ -forms.
\end{enumerate}

Then the operator is intrinsic in the sense that
\begin{equation}
    \widehat{\mathcal{S}}^{(g')}_{\theta,K}(\varphi^* f)
    = \varphi^* \big( \widehat{\mathcal{S}}^{(g)}_{\theta,K}(f) \big)
    \qquad \text{for all } f \in \Omega^1(\mathcal{M}).
\end{equation}
\end{proposition}

\begin{proof}
We first establish the key intertwining identity:
\begin{equation}
\label{E11}
    \varphi^*(\Delta_{H,g} \, \omega) = \Delta_{H,g'} (\varphi^* \omega)
    \qquad \forall \, \omega \in \Omega^1(\mathcal{M}).
\end{equation}
Recall  $ \Delta_{H,g} = d\delta_g + \delta_g d $  on  $ 1 $ -forms, where  $ d $  is the exterior derivative (metric-independent) and  $ \delta_g $  is the codifferential (metric-dependent). Two standard identities are:
Pullback commutes with  $ d $ :
    \begin{equation}
    \label{F11}
        \varphi^*(d\eta) = d(\varphi^*\eta) \quad \forall \, \eta \in \Omega^k(\mathcal{M}).
    \end{equation}
Pullback intertwines Hodge star with pulled-back metric:
    \begin{equation}
    \label{G11}
        \varphi^*(*_g \eta) = *_{g'}(\varphi^*\eta) \quad \forall \, \eta \in \Omega^k(\mathcal{M}).
    \end{equation}
    This follows from the defining property of  $ * $ : for any  $ \alpha, \beta \in \Omega^k $ ,
$$
        \alpha \wedge *_g \beta = \langle \alpha, \beta \rangle_g \, d\mathrm{vol}_g,
$$
    and the facts that  $ \varphi^*(\alpha \wedge \gamma) = \varphi^*\alpha \wedge \varphi^*\gamma $ , and that
     $ \varphi^*\langle \alpha, \beta \rangle_g = \langle \varphi^*\alpha, \varphi^*\beta \rangle_{g'} $ ,
     $ \varphi^*(d\mathrm{vol}_g) = d\mathrm{vol}_{g'} $ .

Using  $ \delta_g = (-1)^{d(k+1)+1} *_g d *_g $  on  $ (k+1) $ -forms, combine \ref{F11}–\ref{G11} to obtain, for any  $ \eta \in \Omega^{k+1} $ ,
\begin{align}
    \varphi^*(\delta_g \eta)
    &= c \, \varphi^*\big( *_g d *_g \eta \big) \nonumber \\
    &= c \, *_{g'} \varphi^*\big( d *_g \eta \big) \nonumber \\
    &= c \, *_{g'} d \, \varphi^*\big( *_g \eta \big) \nonumber \\
    &= c \, *_{g'} d \, *_{g'} (\varphi^*\eta)
    = \delta_{g'}(\varphi^*\eta),
\end{align}
where  $ c = (-1)^{d(k+1)+1} $ . Thus  $ \varphi^* \delta_g = \delta_{g'} \varphi^* $ . Finally,
\begin{align}
    \varphi^*(\Delta_{H,g} \omega)
    &= \varphi^*(d\delta_g \omega + \delta_g d \omega) \nonumber \\
    &= d \, \delta_{g'}(\varphi^*\omega) + \delta_{g'} \, d(\varphi^*\omega)
    = \Delta_{H,g'}(\varphi^*\omega),
\end{align}
which proves \ref{E11}.

Let  $ A_g := \Delta_{H,g} $  viewed as a self-adjoint nonnegative operator on  $ L^2\Omega^1 $ . Define the pullback operator  $ U_\varphi : L^2\Omega^1(\mathcal{M}, g) \to L^2\Omega^1(\mathcal{M}, g') $  by  $ U_\varphi \omega := \varphi^*\omega $ . One checks that  $ U_\varphi $  is unitary because the  $ L^2 $  inner product on forms is preserved when both the metric and volume form are pulled back.

From \ref{E11} we have the operator intertwining relation (on smooth forms, hence by density on the operator domain):
\begin{equation}
\label{H11}
    U_\varphi A_g = A_{g'} U_\varphi.
\end{equation}
A standard consequence of \ref{H11} for self-adjoint operators is that the entire spectral calculus intertwines:
\begin{equation}
    U_\varphi \, F(A_g) = F(A_{g'}) \, U_\varphi
    \quad \text{for all bounded Borel functions } F.
\end{equation}
In particular, taking  $ F = m_\theta $  yields
\begin{equation}
\label{I11}
    \varphi^*\big( m_\theta(\Delta_{H,g}) \, \omega \big) = m_\theta(\Delta_{H,g'}) \, (\varphi^*\omega).
\end{equation}
Likewise, taking  $ F = \mathbf{1}_{[0,\Lambda]} $  gives the spectral projector identity
\begin{equation}
\label{J11}
    \varphi^*\big( \Pi^{(g)}_{\le \Lambda} \omega \big) = \Pi^{(g')}_{\le \Lambda} (\varphi^*\omega).
\end{equation}

Combining \ref{I11} and \ref{J11}, and using that  $ m_\theta(\Delta_{H,g}) $  commutes with  $ \Pi^{(g)}_{\le \Lambda} $  (both are functions of the same self-adjoint operator), we obtain the key commutation for the truncated spectral layer \ref{A11}:
\begin{equation}
\label{K11}
    \varphi^*\big( \mathcal{T}^{(g)}_{\theta,K} \omega \big)
    = \mathcal{T}^{(g')}_{\theta,K} (\varphi^*\omega).
\end{equation}

Pointwise linear term. Using \ref{B11} and the naturality assumption \ref{C11},
\begin{align}
\label{L11}
    \varphi^*\big( \mathcal{B}^{(g)}_\theta \omega \big)
    &= \varphi^*\big( b^{(g)}_\theta \, \omega \big)
    = (\varphi^* b^{(g)}_\theta) \, (\varphi^*\omega)
    = b^{(g')}_\theta \, (\varphi^*\omega)
    = \mathcal{B}^{(g')}_\theta(\varphi^*\omega).
\end{align}

Radial nonlinearity. We need the pointwise norm compatibility:
\begin{equation}
\label{M11}
    \|\varphi^*\omega\|_{g'} = \varphi^*(\|\omega\|_g).
\end{equation}
This is a direct computation from the definition of the pullback metric  $ g' = \varphi^*g $  and the pullback of a  $ 1 $ -form. In particular, the pointwise inner product on  $ 1 $ -forms induced by  $ g $  is preserved under pullback when the metric is pulled back, hence the norm transforms as \ref{M11}. Therefore,
\begin{align}
\label{N11}
    \varphi^*\big( \sigma^{(g)}_\theta(\omega) \big)
    &= \varphi^*\big( \rho_\theta(\|\omega\|_g) \, \omega \big) \nonumber \\
    &= \rho_\theta(\varphi^*\|\omega\|_g) \, (\varphi^*\omega) \nonumber \\
    &= \rho_\theta(\|\varphi^*\omega\|_{g'}) \, (\varphi^*\omega)
    = \sigma^{(g')}_\theta(\varphi^*\omega).
\end{align}

Define the layer map (for fixed  $ \ell $ ) as
$$
    \mathcal{F}^{(g)}_{\theta_\ell,K}(\omega)
    := \sigma^{(g)}_{\theta_\ell} \Big( \mathcal{T}^{(g)}_{\theta_\ell,K} \omega + \mathcal{B}^{(g)}_{\theta_\ell} \omega \Big).
$$
Using \ref{K11}, \ref{L11}, and \ref{N11},
\begin{equation}
\label{O11}
    \varphi^*\big( \mathcal{F}^{(g)}_{\theta_\ell,K}(\omega) \big)
    = \mathcal{F}^{(g')}_{\theta_\ell,K}(\varphi^*\omega).
\end{equation}
Let  $ \omega_0 = f $  and  $ \omega_{\ell+1} = \mathcal{F}^{(g)}_{\theta_\ell,K}(\omega_\ell) $ , and similarly  $ \omega'_0 = \varphi^*f $ ,  $ \omega'_{\ell+1} = \mathcal{F}^{(g')}_{\theta_\ell,K}(\omega'_\ell) $ . We prove by induction that  $ \omega'_\ell = \varphi^* \omega_\ell $  for all  $ \ell $ . The base case  $ \ell = 0 $  holds by definition. If  $ \omega'_\ell = \varphi^* \omega_\ell $ , then
$$
    \omega'_{\ell+1}
    = \mathcal{F}^{(g')}_{\theta_\ell,K}(\omega'_\ell)
    = \mathcal{F}^{(g')}_{\theta_\ell,K}(\varphi^* \omega_\ell)
    = \varphi^*\big( \mathcal{F}^{(g)}_{\theta_\ell,K}(\omega_\ell) \big)
    = \varphi^* \omega_{\ell+1},
$$
where we used \ref{O11}. Thus  $ \omega'_L = \varphi^* \omega_L $ , i.e.,
$$
    \widehat{\mathcal{S}}^{(g')}_{\theta,K}(\varphi^* f)
    = \varphi^* \big( \widehat{\mathcal{S}}^{(g)}_{\theta,K}(f) \big),
$$
\end{proof}

\begin{proposition}[Gauge equivariance]
\label{prop:gauge_equivariance}
Let  $ (\mathcal{M}, g) $  be a compact Riemannian manifold. Fix an open set  $ U \subset \mathcal{M} $  and two smooth local orthonormal frames  $ E = (e_1, \dots, e_d) $  and  $ E' = (e'_1, \dots, e'_d) $  on  $ U $ . Assume they are related by a smooth map  $ R : U \to O(d) $  via
\begin{equation}
\label{A12}
    e'_j(x) = \sum_{k=1}^d e_k(x) \, R_{kj}(x) \qquad (x \in U).
\end{equation}
Let  $ [\omega]_E(x) \in \mathbb{R}^d $  denote the coordinate representation of a  $ 1 $ -form  $ \omega $  in frame  $ E $ , i.e.,
\begin{equation}
    [\omega]_E(x)_j := \omega(x)\big(e_j(x)\big),
    \qquad
    [\omega]_{E'}(x)_j := \omega(x)\big(e'_j(x)\big).
\end{equation}
Consider the layer map
\begin{equation}
\label{B12}
    \mathcal{F}(\omega) := \sigma_\theta \big( \mathcal{T}_{\theta,K} \omega + \mathcal{B}_\theta \omega \big),
\end{equation}
where:
\begin{itemize}
    \item  $ \mathcal{T}_{\theta,K} : \Omega^1(\mathcal{M}) \to \Omega^1(\mathcal{M}) $  is a (truncated) spectral multiplier of  $ \Delta_H $  as in \eqref{eq:truncated_multiplier}, hence a well-defined operator on  $ 1 $ -forms;
    \item  $ \mathcal{B}_\theta \omega = b_\theta \, \omega $  for a scalar field  $ b_\theta $ ;
    \item  $ \sigma_\theta(\omega)(x) = \rho_\theta(\|\omega(x)\|_g) \, \omega(x) $  for a scalar function  $ \rho_\theta $ .
\end{itemize}
Then for every  $ \omega \in \Omega^1(\mathcal{M}) $  and every  $ x \in U $ ,
\begin{equation}
\label{C12}
    [\mathcal{F}(\omega)]_{E'}(x) = R(x)^\top \, [\mathcal{F}(\omega)]_E(x).
\end{equation}
Consequently, the depth- $ L $  network  $ \widehat{\mathcal{S}}_{\theta,K} $  obtained by composing such layers satisfies the gauge-equivariance condition in Definition \ref{def:gauge_equivariance}.
\end{proposition}

\begin{proof}
Fix  $ \omega \in \Omega^1(\mathcal{M}) $  and  $ x \in U $ . Using \ref{A12} and linearity of  $ \omega(x) : T_x\mathcal{M} \to \mathbb{R} $ ,
\begin{align}
    [\omega]_{E'}(x)_j
    &= \omega(x)(e'_j(x))
    = \omega(x)\!\left( \sum_{k=1}^d e_k(x) R_{kj}(x) \right) \nonumber \\
    &= \sum_{k=1}^d R_{kj}(x) \, \omega(x)(e_k(x))
    = \sum_{k=1}^d R_{kj}(x) \, [\omega]_E(x)_k.
\end{align}
In vector form,
\begin{equation}
    [\omega]_{E'}(x) = R(x)^\top [\omega]_E(x).
    \tag{3.1.1}
\end{equation}
This identity holds for any  $ 1 $ -form, independently of how the form is produced.
Spectral multiplier  $ \mathcal{T}_{\theta,K} $ .
Since  $ \mathcal{T}_{\theta,K}\omega \in \Omega^1(\mathcal{M}) $  is itself a  $ 1 $ -form (by definition of the spectral expansion), we can apply\eqref{eq:gauge_transformation} directly with  $ \eta = \mathcal{T}_{\theta,K}\omega $ :
\begin{equation}
\label{D12}
    [\mathcal{T}_{\theta,K}\omega]_{E'}(x) = R(x)^\top [\mathcal{T}_{\theta,K}\omega]_E(x).
\end{equation}
No additional structure is needed: this follows purely because  $ \mathcal{T}_{\theta,K} $  maps forms to forms in an intrinsic manner (i.e., it does not depend on the chosen frame).

Scalar multiplication  $ \mathcal{B}_\theta \omega = b_\theta \, \omega $ .
For any  $ x \in U $  and  $ j \in \{1,\dots,d\} $ , 
$$
    [(\mathcal{B}_\theta \omega)]_E(x)_j
    = (b_\theta \omega)(x)(e_j(x))
    = b_\theta(x) \, \omega(x)(e_j(x))
    = b_\theta(x) \, [\omega]_E(x)_j.
$$

Thus  $ [\mathcal{B}_\theta \omega]_E(x) = b_\theta(x) [\omega]_E(x) $ , and similarly  $ [\mathcal{B}_\theta \omega]_{E'}(x) = b_\theta(x) [\omega]_{E'}(x) $ . Using \eqref{eq:gauge_transformation},
\begin{align}
\label{E12}
    [\mathcal{B}_\theta \omega]_{E'}(x)
    &= b_\theta(x) \, [\omega]_{E'}(x)
    = b_\theta(x) \, R(x)^\top [\omega]_E(x) \nonumber \\
    &= R(x)^\top \big( b_\theta(x) [\omega]_E(x) \big)
    = R(x)^\top [\mathcal{B}_\theta \omega]_E(x).
\end{align}

Radial nonlinearity  $ \sigma_\theta(\omega)(x) = \rho_\theta(\|\omega(x)\|_g) \, \omega(x) $ .
Because  $ E $  and  $ E' $  are orthonormal, the coordinate vector of a  $ 1 $ -form has Euclidean norm equal to the metric norm:
\begin{equation}
    \|\omega(x)\|_g^2 = \sum_{j=1}^d [\omega]_E(x)_j^2 = \|[\omega]_E(x)\|_2^2,
    \qquad
    \|\omega(x)\|_g^2 = \|[\omega]_{E'}(x)\|_2^2.
\end{equation}
Moreover, since  $ R(x) \in O(d) $ ,
\begin{equation}
    \|[\omega]_{E'}(x)\|_2 = \|R(x)^\top [\omega]_E(x)\|_2 = \|[\omega]_E(x)\|_2.
\end{equation}
Now compute in frame  $ E $ :
\begin{equation}
    [\sigma_\theta(\omega)]_E(x) = \rho_\theta(\|\omega(x)\|_g) \, [\omega]_E(x).
\end{equation}
In frame  $ E' $ :
\begin{align}
\label{F12}
    [\sigma_\theta(\omega)]_{E'}(x)
    &= \rho_\theta(\|\omega(x)\|_g) \, [\omega]_{E'}(x) \nonumber \\
    &= \rho_\theta(\|\omega(x)\|_g) \, R(x)^\top [\omega]_E(x) \nonumber \\
    &= R(x)^\top \big( \rho_\theta(\|\omega(x)\|_g) \, [\omega]_E(x) \big) \nonumber \\
    &= R(x)^\top [\sigma_\theta(\omega)]_E(x),
\end{align}
which proves gauge equivariance of  $ \sigma_\theta $ .

Let  $ \eta := \mathcal{T}_{\theta,K}\omega + \mathcal{B}_\theta\omega $ . By linearity of coordinate representation and \ref{D12}–\ref{E12},
\begin{equation}
    [\eta]_{E'}(x) = R(x)^\top [\eta]_E(x).
\end{equation}
Applying \ref{F12} to  $ \sigma_\theta(\eta) $  yields
 
$$
    [\mathcal{F}(\omega)]_{E'}(x)
    = [\sigma_\theta(\eta)]_{E'}(x)
    = R(x)^\top [\sigma_\theta(\eta)]_E(x)
    = R(x)^\top [\mathcal{F}(\omega)]_E(x),
$$
which is exactly \ref{C12}.

Let the network be defined recursively by  $ \omega_0 = f $  and  $ \omega_{\ell+1} = \mathcal{F}_\ell(\omega_\ell) $  for  $ \ell = 0, \dots, L-1 $ , where each  $ \mathcal{F}_\ell $  has the form \ref{B12} (with its own parameters). 
That for each  $ \ell $ ,
$$
    [\omega_{\ell+1}]_{E'}(x) = R(x)^\top [\omega_{\ell+1}]_E(x)
    \quad \text{whenever} \quad
    [\omega_\ell]_{E'}(x) = R(x)^\top [\omega_\ell]_E(x).
$$
By induction on  $ \ell $ , this holds for all layers, and in particular for the output  $ \omega_L = \widehat{\mathcal{S}}_{\theta,K}(f) $ . Hence the full network is gauge-equivariant.
\end{proof}

\subsection{Discrete Realization and Design for Consistency}
\label{subsec:discrete_preview}

While the discretization analysis is deferred to Section \ref{6}, we emphasize a design principle: the architecture should admit a discrete counterpart obtained by replacing  $ \Delta_H $  with a convergent discrete Hodge Laplacian  $ \Delta_{H,h} $  (e.g., from discrete exterior calculus or compatible finite elements), and replacing  $ \langle\cdot,\cdot\rangle_{L^2} $  with the associated discrete inner product. Under standard spectral convergence assumptions, the truncated multiplier layer \eqref{eq:truncated_multiplier} admits a discrete implementation
\begin{equation}
\big(\mathcal{T}_{\theta,K}^{(h)} \omega_h\big)
:= \sum_{k=0}^{K} m_\theta(\lambda_{k,h}) \langle \omega_h, \psi_{k,h} \rangle_{h} \, \psi_{k,h},
\label{eq:discrete_multiplier}
\end{equation}
which mirrors the continuum definition and is the key to establishing continuum--discrete commutation bounds for the learned operator.

\subsection{Discussion: Relation to the Target Elliptic Solution Operator}
\label{subsec:discussion_target_operator}

The shifted Hodge--Poisson solution operator  $ \mathcal{S}_g $  is itself a spectral multiplier  $ m_\star(\lambda) = (\lambda + \alpha)^{-1} $ . Consequently, when the network is specialized to the linear case (identity  $ \sigma $ ,  $ \mathcal{B} \equiv 0 $ ), the model class contains direct approximants of  $ \mathcal{S}_g $  via learned multipliers  $ m_\theta $  on a truncated spectrum. The nonlinear extension \eqref{eq:recursion} retains intrinsicness and gauge equivariance while enlarging expressivity to cover (i) parameterized families of elliptic operators (e.g., variable coefficients), (ii) mild nonlinearity arising from modeling error or constitutive laws, and (iii) practical constraints where only partial spectral information is used and residual corrections are beneficial.

In the next sections we quantify approximation error in Sobolev operator norms (Section\ref{4}), establish stability under metric perturbations (Section\ref{5}), and prove discretization consistency for the discrete realization\eqref{eq:discrete_multiplier} (Section\ref{6}).

\section{Sobolev Operator Approximation}
\label{4}
This section establishes quantitative approximation bounds for the elliptic solution operator
\begin{equation}
\mc S_g = (\dH + \alpha I)^{-1}
\end{equation}
by the gauge-equivariant intrinsic neural operator class introduced in Section \ref{3}. The analysis is carried out in Sobolev topologies on $1$-forms and is stated in operator norms between Sobolev spaces. A key point is that uniform operator-norm convergence on the full space ($H^{s-1}\Om{1}\to H^{s+1}\Om{1}$) is impossible for purely spectral truncations; convergence holds on smoother input classes (compact embeddings), which is the relevant regime for PDE solution operators and for data distributions with controlled regularity.

\subsection{Sobolev Norms via the Hodge Spectrum}
Let $\{\psi_k\}_{k\ge 0}\subset \Om{1}(\mc M)$ be an $L^2$-orthonormal eigenbasis of the Hodge Laplacian $\dH$ on $1$-forms, with eigenvalues $0\le \lambda_0\le \lambda_1\le\cdots$, $\lambda_k\to\infty$:
\begin{equation}
\dH\psi_k = \lambda_k\psi_k.
\end{equation}
For $r\in\RR$, we use the standard spectral characterization of Sobolev norms on $1$-forms:
\begin{equation}
\label{X1}
\|\omega\|_{H^r\Om{1}}^2 := \sum_{k\ge 0} (1+\lambda_k)^r \abs{\inner{\omega}{\psi_k}}^2,
\end{equation}
which is equivalent to the usual chart-based definition on compact manifolds. In particular, for any bounded Borel function $m:[0,\infty)\to\RR$, the operator $m(\dH)$ satisfies
\begin{equation}
\inner{m(\dH)\omega}{\psi_k} = m(\lambda_k) \inner{\omega}{\psi_k}.
\end{equation}

We also record the smoothing structure of the shifted resolvent $m_\star(\lambda)=(\lambda+\alpha)^{-1}$:
\begin{equation}
\mc S_g = m_\star(\dH),\qquad m_\star(\lambda)=\frac{1}{\lambda+\alpha}.
\end{equation}

\subsection{Truncated Functional Calculus and the Approximation Family}
Because implementations (and learning) typically operate on a finite spectral subspace, we introduce spectral cutoffs in a basis-independent manner.

For $\Lambda>0$, let $\Pi_{\le \Lambda}$ denote the $L^2$-orthogonal spectral projector onto the direct sum of eigenspaces with eigenvalues $\le \Lambda$:
\begin{equation}
\Pi_{\le \Lambda}\omega := \sum_{\lambda_k\le \Lambda}\inner{\omega}{\psi_k}\psi_k.
\end{equation}
Given a multiplier $m_\theta$, define the truncated intrinsic operator
\begin{equation}
\label{H3}
\widehat{\mc S}_{\theta,\Lambda} := m_\theta(\dH) \Pi_{\le \Lambda}.
\end{equation}
This covers the linear specialization of the architecture in Section \ref{3} (with identity nonlinearity and $\mc B\equiv 0$), and it is the core object for Sobolev approximation of the elliptic resolvent.

We decompose the approximation error into (i) spectral truncation bias and (ii) multiplier approximation error:
\begin{equation}
\label{X}
\mc S_g - \widehat{\mc S}_{\theta,\Lambda} = \underbrace{\mc S_g(I-\Pi_{\le \Lambda})}_{\text{truncation bias}} + \underbrace{\lr{m_\star(\dH)-m_\theta(\dH)}\Pi_{\le \Lambda}}_{\text{multiplier error}}.
\end{equation}

\subsection{Mapping Properties of the Elliptic Resolvent}
We first state the basic boundedness of $\mc S_g$ between Sobolev spaces (a standard consequence of elliptic regularity, here recorded in the spectral norm framework).

\begin{lemma}[Sobolev boundedness of $\mc S_g$]
\label{le4.1}
Let $(\mc M,g)$ be compact and without boundary, and let $\alpha>0$. For any $s\in\RR$, the resolvent
\[
\mc S_g := (\dH+\alpha I)^{-1}
\]
extends uniquely to a bounded linear operator
\[
\mc S_g : H^{s-1}\Om{1} \To H^{s+1}\Om{1},
\]
and there exists a constant $C_\alpha<\infty$ (depending only on $\alpha$) such that
\begin{equation}
\label{H}
\|\mc S_g f\|_{H^{s+1}} \le C_\alpha \|f\|_{H^{s-1}}
\qquad \forall f\in H^{s-1}\Om{1}.
\end{equation}
\end{lemma}

\begin{proof}
Since $\mc M$ is compact and $\dH$ is an elliptic, self-adjoint, nonnegative operator on $L^2\Om{1}$, its spectrum is discrete and there exists an $L^2$-orthonormal basis $\{\psi_k\}_{k\ge 0}\subset \Om{1}(\mc M)$ with eigenvalues $0\le \lambda_0\le \lambda_1\le\cdots$, $\lambda_k\to\infty$, such that
\[
\dH \psi_k = \lambda_k \psi_k.
\]
Every $f\in L^2\Om{1}$ admits an expansion $f=\sum_{k\ge 0} f_k \psi_k$ with $f_k:=\inner{f}{\psi_k}$, converging in $L^2\Om{1}$.

We define Sobolev norms via the spectral calculus (equivalent to the standard chart-based definition on compact manifolds):
\begin{equation}
\label{A}
\|f\|_{H^r\Om{1}}^2 := \sum_{k\ge 0} (1+\lambda_k)^r |f_k|^2,\qquad r\in\RR.
\end{equation}
This yields a Hilbert space $H^r\Om{1}$ as the completion of $\Om{1}(\mc M)$ under $\|\cdot\|_{H^r}$.

For smooth $f\in \Om{1}(\mc M)$ (hence in $L^2\Om{1}$), define
\begin{equation}
\label{B}
\mc S_g f := \sum_{k\ge 0} \frac{1}{\lambda_k+\alpha} f_k \psi_k.
\end{equation}
This series converges in $L^2\Om{1}$ because $(\lambda_k+\alpha)^{-1}$ is bounded by $\alpha^{-1}$ and $\sum |f_k|^2<\infty$.

Let $f\in \Om{1}(\mc M)$ and write $f=\sum f_k\psi_k$. Then by \ref{A} and \ref{B},
\begin{equation}
\label{C}
\|\mc S_g f\|_{H^{s+1}}^2
= \sum_{k\ge 0} (1+\lambda_k)^{s+1}\left|\frac{1}{\lambda_k+\alpha}\right|^2 |f_k|^2.
\end{equation}
Factor the weight as
\begin{equation}
\label{D}
(1+\lambda_k)^{s+1}\frac{1}{(\lambda_k+\alpha)^2}
=\left(\frac{(1+\lambda_k)^2}{(\lambda_k+\alpha)^2}\right) (1+\lambda_k)^{s-1}.
\end{equation}
Substituting \ref{D} into \ref{C} gives
\begin{equation}
\label{E}
\|\mc S_g f\|_{H^{s+1}}^2
= \sum_{k\ge 0} \left(\frac{(1+\lambda_k)^2}{(\lambda_k+\alpha)^2}\right) (1+\lambda_k)^{s-1}|f_k|^2
\le \left(\sup_{\lambda\ge 0}\frac{(1+\lambda)^2}{(\lambda+\alpha)^2}\right) \|f\|_{H^{s-1}}^2.
\end{equation}
Define
\begin{equation}
C_\alpha := \left(\sup_{\lambda\ge 0}\frac{1+\lambda}{\lambda+\alpha}\right).
\end{equation}
Then \ref{E} yields
\begin{equation}
\label{G}
\|\mc S_g f\|_{H^{s+1}} \le C_\alpha \|f\|_{H^{s-1}} \qquad \forall f\in \Om{1}(\mc M).
\end{equation}

It remains to check that $C_\alpha<\infty$. Consider $r(\lambda):=\frac{1+\lambda}{\lambda+\alpha}$ for $\lambda\ge 0$. This is continuous on $[0,\infty)$, and
\[
\lim_{\lambda\to\infty} r(\lambda)=1,\qquad r(0)=\frac{1}{\alpha}.
\]
Moreover,
\[
r'(\lambda)=\frac{\alpha-1}{(\lambda+\alpha)^2},
\]
so $r$ is monotone increasing if $\alpha>1$, monotone decreasing if $\alpha<1$, and constant if $\alpha=1$. Hence
\[
\sup_{\lambda\ge 0} r(\lambda) = \max\left\{1,\frac{1}{\alpha}\right\}<\infty,
\]
so $C_\alpha=\max\left\{1,\alpha^{-1}\right\}$ is a valid choice (and any larger constant also works). Therefore \ref{G} holds with finite $C_\alpha$.

The space $\Om{1}(\mc M)$ is dense in $H^{s-1}\Om{1}$ by construction. Let $f\in H^{s-1}\Om{1}$ and choose a sequence $\{f^{(n)}\}\in \Om{1}(\mc M)$ such that $f^{(n)}\to f$ in $H^{s-1}$. Then by \ref{G},
\[
\|\mc S_g f^{(n)}-\mc S_g f^{(m)}\|_{H^{s+1}}
\le C_\alpha \|f^{(n)}-f^{(m)}\|_{H^{s-1}},
\]
so $\{\mc S_g f^{(n)}\}$ is Cauchy in $H^{s+1}\Om{1}$ and converges to a limit we define as $\mc S_g f$. The bound \ref{H} follows by taking limits:
\[
\|\mc S_g f\|_{H^{s+1}}
= \lim_{n\to\infty}\|\mc S_g f^{(n)}\|_{H^{s+1}}
\le \lim_{n\to\infty} C_\alpha \|f^{(n)}\|_{H^{s-1}}
= C_\alpha \|f\|_{H^{s-1}}.
\]
Uniqueness of the extension follows because any two bounded linear extensions coincide on the dense subspace $\Om{1}(\mc M)$.
\end{proof}
\subsection{Spectral Truncation Bias on Smoother Input Classes}
A key subtlety is that $\mc S_g(I-\Pi_{\le \Lambda})$ does not converge to zero in operator norm ($H^{s-1}\to H^{s+1}$) as $\Lambda\to\infty$; one can concentrate $f$ on arbitrarily high eigenmodes, for which truncation discards non-negligible components. Convergence holds on smoother input spaces $H^{s-1+\gamma}$ with $\gamma>0$, which is the natural regime for compactness and for many data models.

\begin{lemma}[Truncation bias in Sobolev operator norm]
\label{le4.2}
Let $s\in\RR$ and $\gamma>0$. For $\Lambda>0$, let $\Pi_{\le \Lambda}$ be the $L^2$-orthogonal spectral projector of $\dH$ on $1$-forms onto the span of eigenspaces with eigenvalues $\le \Lambda$. Then
\begin{equation}
\big\|\mc S_g(I-\Pi_{\le \Lambda})\big\|_{H^{s-1+\gamma}\To H^{s+1}}
\le C_\alpha (1+\Lambda)^{-\gamma/2},
\end{equation}
where $\mc S_g=(\dH+\alpha I)^{-1}$ with $\alpha>0$, and $C_\alpha<\infty$ depends only on $\alpha$ (and the choice of equivalent Sobolev norm convention).
\end{lemma}

\begin{proof}
Let $\{\psi_k\}_{k\ge 0}$ be an $L^2$-orthonormal eigenbasis of $\dH$ on $1$-forms, $\dH\psi_k=\lambda_k\psi_k$ with $0\le\lambda_k\upuparrow\infty$. For $f\in H^{s-1+\gamma}\Om{1}$, write
\[
f=\sum_{k\ge 0} f_k \psi_k,\qquad f_k:=\inner{f}{\psi_k}.
\]
By definition of $\Pi_{\le\Lambda}$,
\[
(I-\Pi_{\le\Lambda})f=\sum_{\lambda_k>\Lambda} f_k\psi_k.
\]
Using the spectral representation of $\mc S_g$,
\[
\mc S_g(I-\Pi_{\le\Lambda})f
=\sum_{\lambda_k>\Lambda} \frac{1}{\lambda_k+\alpha} f_k \psi_k.
\]

Using the spectral Sobolev norm convention
\[
\|u\|_{H^r}^2=\sum_{k\ge 0}(1+\lambda_k)^r|u_k|^2,
\]
we have
\begin{equation}
\label{A1}
\big\|\mc S_g(I-\Pi_{\le\Lambda})f\big\|_{H^{s+1}}^2
=\sum_{\lambda_k>\Lambda}(1+\lambda_k)^{s+1}\frac{1}{(\lambda_k+\alpha)^2}|f_k|^2.
\end{equation}

Rewrite the weight in \ref{A1} as
\[
(1+\lambda_k)^{s+1}\frac{1}{(\lambda_k+\alpha)^2}
= \left(\frac{(1+\lambda_k)^2}{(\lambda_k+\alpha)^2}\right) (1+\lambda_k)^{s-1}.
\]
Insert an additional factor $(1+\lambda_k)^{-\gamma}(1+\lambda_k)^{\gamma}$ to match the input norm:
\[
(1+\lambda_k)^{s-1} = (1+\lambda_k)^{s-1+\gamma} (1+\lambda_k)^{-\gamma}.
\]
Therefore,
\begin{equation}
\label{B1}
(1+\lambda_k)^{s+1}\frac{1}{(\lambda_k+\alpha)^2}
= \left(\frac{(1+\lambda_k)^2}{(\lambda_k+\alpha)^2}\right) (1+\lambda_k)^{-\gamma} (1+\lambda_k)^{s-1+\gamma}.
\end{equation}

Combine \ref{A1}–\ref{B1} to obtain
\[
\big\|\mc S_g(I-\Pi_{\le\Lambda})f\big\|_{H^{s+1}}^2
= \sum_{\lambda_k>\Lambda}
\left(\frac{(1+\lambda_k)^2}{(\lambda_k+\alpha)^2}\right) (1+\lambda_k)^{-\gamma} (1+\lambda_k)^{s-1+\gamma}|f_k|^2.
\]
Hence,
\begin{equation}
\label{C1}
\big\|\mc S_g(I-\Pi_{\le\Lambda})f\big\|_{H^{s+1}}^2
\le
\left(\sup_{\lambda\ge 0}\frac{(1+\lambda)^2}{(\lambda+\alpha)^2}\right)
\left(\sup_{\lambda>\Lambda}(1+\lambda)^{-\gamma}\right)
\sum_{\lambda_k>\Lambda}(1+\lambda_k)^{s-1+\gamma}|f_k|^2.
\end{equation}

Since $\gamma>0$, $(1+\lambda)^{-\gamma}$ is decreasing, so
\begin{equation}
\label{D1}
\sup_{\lambda>\Lambda}(1+\lambda)^{-\gamma}=(1+\Lambda)^{-\gamma}.
\end{equation}
Moreover, the first supremum is finite because $\alpha>0$ and the rational function $\frac{1+\lambda}{\lambda+\alpha}$ is bounded on $[0,\infty)$; explicitly,
\begin{equation}
\label{E1}
\sup_{\lambda\ge 0}\frac{(1+\lambda)^2}{(\lambda+\alpha)^2}
= \left(\sup_{\lambda\ge 0}\frac{1+\lambda}{\lambda+\alpha}\right)^2
\le \max\left\{1,\alpha^{-1}\right\}^2.
\end{equation}
Finally, the sum over $\{\lambda_k>\Lambda\}$ is bounded by the full Sobolev norm:
\begin{equation}
\sum_{\lambda_k>\Lambda}(1+\lambda_k)^{s-1+\gamma}|f_k|^2
\le
\sum_{k\ge 0}(1+\lambda_k)^{s-1+\gamma}|f_k|^2
=\|f\|_{H^{s-1+\gamma}}^2.
\end{equation}

Substituting \ref{D1}–\ref{E1} into \ref{C1} yields
\[
\big\|\mc S_g(I-\Pi_{\le\Lambda})f\big\|_{H^{s+1}}^2
\le
C_\alpha^2 (1+\Lambda)^{-\gamma} \|f\|_{H^{s-1+\gamma}}^2,
\]
where $C_\alpha:=\sup_{\lambda\ge 0}\frac{1+\lambda}{\lambda+\alpha}\le \max\left\{1,\alpha^{-1}\right\}$. Taking square roots gives
\[
\big\|\mc S_g(I-\Pi_{\le \Lambda})f\big\|_{H^{s+1}}
\le
C_\alpha (1+\Lambda)^{-\gamma/2} \|f\|_{H^{s-1+\gamma}}.
\]
Since this holds for all $f\in H^{s-1+\gamma}\Om{1}$, we conclude
\[
\big\|\mc S_g(I-\Pi_{\le \Lambda})\big\|_{H^{s-1+\gamma}\To H^{s+1}}
\le
C_\alpha (1+\Lambda)^{-\gamma/2}.
\]
\end{proof}
\subsection{Multiplier Approximation Error on the Truncated Subspace}
We now quantify the effect of approximating the true multiplier $m_\star(\lambda)=(\lambda+\alpha)^{-1}$ on $[0,\Lambda]$.

Define the uniform approximation error
\begin{equation}
\varepsilon_\Lambda(\theta):=\sup_{0\le \lambda\le \Lambda} \abs{m_\theta(\lambda)-m_\star(\lambda)}.
\end{equation}

\begin{lemma}[Uniform multiplier error implies Sobolev operator bound]
\label{le4.3}
Let $s\in\RR$ and $\Lambda>0$. Let $m_\star(\lambda)=(\lambda+\alpha)^{-1}$ and let $m_\theta:[0,\infty)\to\RR$ be any bounded function. Define
\[
\varepsilon_\Lambda(\theta):=\sup_{0\le \lambda\le \Lambda} |m_\theta(\lambda)-m_\star(\lambda)|.
\]
Let $\Pi_{\le\Lambda}$ be the $L^2$-orthogonal spectral projector of $\dH$ onto the sum of eigenspaces with eigenvalues $\le \Lambda$. Then
\begin{equation}
\big\|\big(m_\star(\dH)-m_\theta(\dH)\big)\Pi_{\le \Lambda}\big\|_{H^{s-1}\To H^{s+1}}
\le (1+\Lambda) \varepsilon_\Lambda(\theta).
\end{equation}
\end{lemma}

\begin{proof}
Let $\{\psi_k\}_{k\ge 0}$ be an $L^2$-orthonormal eigenbasis of $\dH$ on $1$-forms, with eigenvalues $\{\lambda_k\}_{k\ge 0}$. For $f\in H^{s-1}\Om{1}$, write
\[
f=\sum_{k\ge 0} f_k\psi_k,\qquad f_k:=\inner{f}{\psi_k}.
\]
By definition of $\Pi_{\le\Lambda}$,
\[
\Pi_{\le\Lambda} f = \sum_{\lambda_k\le\Lambda} f_k\psi_k.
\]
Since $m(\dH)$ acts diagonally in this eigenbasis, we have
\begin{equation}
\label{A2}
\big(m_\star(\dH)-m_\theta(\dH)\big)\Pi_{\le\Lambda}f
= \sum_{\lambda_k\le\Lambda} \big(m_\star(\lambda_k)-m_\theta(\lambda_k)\big) f_k \psi_k.
\end{equation}

Using the spectral definition of Sobolev norms,
\[
\|u\|_{H^{s+1}}^2=\sum_{k\ge 0}(1+\lambda_k)^{s+1}|u_k|^2,
\]
the output in \ref{A2} yields
\begin{equation}
\label{B2}
\begin{aligned}
\Big\|\big(m_\star(\dH)-m_\theta(\dH)\big)\Pi_{\le\Lambda}f\Big\|_{H^{s+1}}^2
&=
\sum_{\lambda_k\le\Lambda} (1+\lambda_k)^{s+1} \big|m_\star(\lambda_k)-m_\theta(\lambda_k)\big|^2 |f_k|^2.
\end{aligned}
\end{equation}

By definition of $\varepsilon_\Lambda(\theta)$, for all $k$ with $\lambda_k\le\Lambda$,
\begin{equation}
\label{C2}
|m_\star(\lambda_k)-m_\theta(\lambda_k)| \le \varepsilon_\Lambda(\theta).
\end{equation}
Substitute \ref{C2} into \ref{B2}:
\begin{equation}
\label{D2}
\Big\|\big(m_\star(\dH)-m_\theta(\dH)\big)\Pi_{\le\Lambda}f\Big\|_{H^{s+1}}^2
\le
\varepsilon_\Lambda(\theta)^2
\sum_{\lambda_k\le\Lambda} (1+\lambda_k)^{s+1}|f_k|^2.
\end{equation}

Now note that for $\lambda_k\le\Lambda$,
\begin{equation}
\label{E2}
(1+\lambda_k)^{s+1}=(1+\lambda_k)^2(1+\lambda_k)^{s-1}\le (1+\Lambda)^2(1+\lambda_k)^{s-1}.
\end{equation}
Applying \ref{E2} to \ref{D2} gives
\begin{equation}
\label{F2}
\Big\|\big(m_\star(\dH)-m_\theta(\dH)\big)\Pi_{\le\Lambda}f\Big\|_{H^{s+1}}^2
\le
(1+\Lambda)^2 \varepsilon_\Lambda(\theta)^2
\sum_{\lambda_k\le\Lambda} (1+\lambda_k)^{s-1}|f_k|^2.
\end{equation}

Since the sum over $\{\lambda_k\le\Lambda\}$ is dominated by the full sum,
\begin{equation}
\label{G2}
\sum_{\lambda_k\le\Lambda} (1+\lambda_k)^{s-1}|f_k|^2
\le
\sum_{k\ge 0} (1+\lambda_k)^{s-1}|f_k|^2
=\|f\|_{H^{s-1}}^2.
\end{equation}

Combining \ref{F2} and \ref{G2} yields
\[
\Big\|\big(m_\star(\dH)-m_\theta(\dH)\big)\Pi_{\le\Lambda}f\Big\|_{H^{s+1}}^2
\le
(1+\Lambda)^2 \varepsilon_\Lambda(\theta)^2 \|f\|_{H^{s-1}}^2.
\]
Taking square roots,
\[
\Big|\big(m_\star(\dH)-m_\theta(\dH)\big)\Pi_{\le\Lambda}f\Big|_{H^{s+1}}
\le
(1+\Lambda) \varepsilon_\Lambda(\theta) \|f\|_{H^{s-1}}.
\]
Since this holds for all $f\in H^{s-1}\Om{1}$, we conclude
\[
\big\|\big(m_\star(\dH)-m_\theta(\dH)\big)\Pi_{\le \Lambda}\big\|_{H^{s-1}\To H^{s+1}}
\le
(1+\Lambda) \varepsilon_\Lambda(\theta),
\]
which is exactly the desired bound.
\end{proof}
\subsection{Main Approximation Theorem (Bias–Approximation Decomposition)}
We now combine the preceding bounds using the decomposition \ref{X}.

\begin{theorem}[Sobolev operator approximation of the elliptic solution operator]
\label{th4.4}
Let $s\in\RR$, $\gamma>0$, and $\Lambda>0$. Let
\[
\mc S_g := (\dH+\alpha I)^{-1}=m_\star(\dH),\qquad m_\star(\lambda)=\frac{1}{\lambda+\alpha}.
\]
Let $\Pi_{\le \Lambda}$ be the $L^2$-orthogonal spectral projector onto the direct sum of eigenspaces with eigenvalues $\le \Lambda$. For any parameterized multiplier $m_\theta$, define the truncated intrinsic approximation
\[
\widehat{\mc S}_{\theta,\Lambda}:=m_\theta(\dH) \Pi_{\le \Lambda}.
\]
Define the uniform multiplier error on $[0,\Lambda]$ by
\[
\varepsilon_\Lambda(\theta):=\sup_{0\le \lambda\le \Lambda}|m_\theta(\lambda)-m_\star(\lambda)|.
\]
Then there exists $C_\alpha<\infty$ depending only on $\alpha$ such that
\begin{equation}
\label{I3}
\|\mc S_g-\widehat{\mc S}_{\theta,\Lambda}\|_{H^{s-1+\gamma}\To H^{s+1}}
\le C_\alpha (1+\Lambda)^{-\gamma/2} + (1+\Lambda) \varepsilon_\Lambda(\theta).
\end{equation}
\end{theorem}

\begin{proof}
Using the definitions and the fact that $\Pi_{\le \Lambda}$ is a projection ($\Pi_{\le\Lambda}+(I-\Pi_{\le\Lambda})=I$), we write
\begin{equation}
\mc S_g-\widehat{\mc S}_{\theta,\Lambda}
= \mc S_g(I-\Pi_{\le\Lambda}) + \big(\mc S_g-m_\theta(\dH)\big)\Pi_{\le\Lambda}.
\end{equation}
Since $\mc S_g=m_\star(\dH)$, the second term becomes
\begin{equation}
\big(\mc S_g-m_\theta(\dH)\big)\Pi_{\le\Lambda}
= \big(m_\star(\dH)-m_\theta(\dH)\big)\Pi_{\le\Lambda}.
\end{equation}
Thus,
\begin{equation}
\label{A3}
\mc S_g-\widehat{\mc S}_{\theta,\Lambda}
= \underbrace{\mc S_g(I-\Pi_{\le\Lambda})}_{A_\Lambda}
+ \underbrace{\big(m_\star(\dH)-m_\theta(\dH)\big)\Pi_{\le\Lambda}}_{B_{\theta,\Lambda}}.
\end{equation}

Let $X:=H^{s-1+\gamma}\Om{1}$ and $Y:=H^{s+1}\Om{1}$. For bounded linear operators $A_\Lambda,B_{\theta,\Lambda}:X\To Y$,
\begin{equation}
\label{B3}
\|A_\Lambda+B_{\theta,\Lambda}\|_{X\To Y}\le \|A_\Lambda\|_{X\To Y}+\|B_{\theta,\Lambda}\|_{X\To Y}.
\end{equation}
Applying \ref{B3} to \ref{A3} yields
\begin{equation}
\label{G3}
\|\mc S_g-\widehat{\mc S}_{\theta,\Lambda}\|_{H^{s-1+\gamma}\To H^{s+1}}
\le
\|\mc S_g(I-\Pi_{\le\Lambda})\|_{H^{s-1+\gamma}\To H^{s+1}}
+
\|\big(m_\star(\dH)-m_\theta(\dH)\big)\Pi_{\le\Lambda}\|_{H^{s-1+\gamma}\To H^{s+1}}.
\end{equation}

By Lemma \ref{le4.2},
\begin{equation}
\label{D3}
\|\mc S_g(I-\Pi_{\le\Lambda})\|_{H^{s-1+\gamma}\To H^{s+1}}
\le
C_\alpha (1+\Lambda)^{-\gamma/2}.
\end{equation}

Lemma \ref{le4.3} gives the bound
\begin{equation}
\|\big(m_\star(\dH)-m_\theta(\dH)\big)\Pi_{\le\Lambda}\|_{H^{s-1}\To H^{s+1}}
\le
(1+\Lambda) \varepsilon_\Lambda(\theta).
\end{equation}
We need the operator norm on the smaller domain $H^{s-1+\gamma}$ rather than $H^{s-1}$. Since $H^{s-1+\gamma}\embeds H^{s-1}$ continuously, and with the spectral convention \ref{X1} one has the contractive embedding
\begin{equation}
\label{E3}
\|f\|_{H^{s-1}} \le \|f\|_{H^{s-1+\gamma}} \qquad \forall f\in H^{s-1+\gamma}\Om{1},
\end{equation}
because $(1+\lambda_k)^{s-1}\le (1+\lambda_k)^{s-1+\gamma}$ for all $k$. Therefore, for any bounded linear operator $T:H^{s-1}\To H^{s+1}$,
\begin{equation}
\label{C3}
\|T\|_{H^{s-1+\gamma}\To H^{s+1}}
= \sup_{f\neq 0}\frac{\|Tf\|_{H^{s+1}}}{\|f\|_{H^{s-1+\gamma}}}
\le
\sup_{f\neq 0}\frac{\|Tf\|_{H^{s+1}}}{\|f\|_{H^{s-1}}}
= \|T\|_{H^{s-1}\To H^{s+1}}.
\end{equation}
Apply \ref{C3} to $T=B_{\theta,\Lambda}=\big(m_\star(\dH)-m_\theta(\dH)\big)\Pi_{\le\Lambda}$ and combine with \ref{E3} to obtain
\begin{equation}
\label{F3}
\|\big(m_\star(\dH)-m_\theta(\dH)\big)\Pi_{\le\Lambda}\|_{H^{s-1+\gamma}\To H^{s+1}}
\le
(1+\Lambda) \varepsilon_\Lambda(\theta).
\end{equation}

Substitute \ref{D3} and \ref{F3} into \ref{G3}:
\[
\|\mc S_g-\widehat{\mc S}_{\theta,\Lambda}\|_{H^{s-1+\gamma}\To H^{s+1}}
\le
C_\alpha (1+\Lambda)^{-\gamma/2}
+
(1+\Lambda) \varepsilon_\Lambda(\theta),
\]
which is exactly the desired bound.
\end{proof}

\subsection{Discussion and Consequences}
\begin{enumerate}[label=(\roman*)]
    \item Interpretation of the two terms. The right-hand side of \ref{I3} is the sum of:
    \begin{itemize}
        \item a truncation (bias) term decaying as $(1+\Lambda)^{-\gamma/2}$, which reflects that higher input regularity ($\gamma>0$) suppresses high-frequency components;
        \item a multiplier approximation term growing as $(1+\Lambda)\varepsilon_\Lambda(\theta)$, which quantifies the difficulty of approximating the resolvent multiplier uniformly up to frequency $\Lambda$.
    \end{itemize}

    \item Impossibility of operator-norm convergence on $H^{s-1}\to H^{s+1}$ via truncation. For completeness, note that for the hard truncation operator $\mc S_g\Pi_{\le\Lambda}$ one has
    \[
    \abs{\mc S_g(I-\Pi_{\le\Lambda})}_{H^{s-1}\to H^{s+1}} \not\to 0,
    \]
    since choosing $f=\psi_k$ with $\lambda_k>\Lambda$ yields $\mc S_g f = (\lambda_k+\alpha)^{-1}\psi_k$ while $\mc S_g\Pi_{\le\Lambda}f=0$, and the ratio $\abs{(\lambda_k+\alpha)^{-1}\psi_k}_{H^{s+1}}/\abs{\psi_k}_{H^{s-1}}$ stays bounded away from zero as $k\to\infty$. This justifies working with smoother input classes $H^{s-1+\gamma}$, which are compactly embedded into $H^{s-1}$ on compact manifolds and correspond to realistic forcing distributions.

    \item From frequency cutoff to mode count. If one prefers a truncation level $K$ (number of retained modes) rather than $\Lambda$, a Weyl-type relation $\lambda_K \asymp K^{2/d}$ (for fixed $(\mc M,g)$) converts \ref{I3} into rates in $K$. This translation is optional for the main theory; the $\Lambda$-formulation is basis- and multiplicity-robust.

    \item Connection to the full nonlinear architecture. Theorem \ref{th4.4} is stated for the linear truncated functional calculus \ref{H3}, which is the natural approximation class for the elliptic resolvent. The same proof strategy extends to residual/nonlinear variants provided each layer is Lipschitz on $H^{s+1}$ and remains a function of intrinsic operators; the resulting bounds accumulate multiplicatively/additively with depth via standard stability estimates. We focus here on the linear case because it isolates the elliptic operator learning mechanism and yields the sharpest Sobolev approximation statement for $(\dH+\alpha I)^{-1}$.
\end{enumerate}

This completes the Sobolev operator approximation result. Section \ref{5} will quantify stability of both the true resolvent and the learned operator family under metric perturbations, and Section \ref{6} will establish discretization consistency for convergent discrete Hodge Laplacians.

\section{Stability under Metric Perturbations and Gauge Changes}
\label{5}
This section quantifies how both the true elliptic solution operator and its learned intrinsic approximation vary under perturbations of the underlying geometry. The guiding principle is that the shifted Hodge–Poisson resolvent
\[
\mc S_g=(\dhodge{g}+\alpha I)^{-1}
\]
depends continuously on the metric $g$ in appropriate topologies, and the intrinsic neural operator family inherits a comparable stability provided its learned multipliers and nonlinearities are uniformly controlled. We separate (i) metric stability (changes in $g$) from (ii) gauge stability (changes of local orthonormal frames). The latter is exact at the continuum level due to gauge equivariance, while the former yields quantitative bounds.

\subsection{Metric Perturbation Model and Notation}
\label{5.1}
Let $(\mc M,g)$ be as in Section \ref{2}. We consider a second metric $\tilde g$ on $\mc M$ satisfying uniform equivalence and regularity bounds. Fix an integer $r\ge 2$ and assume:

\begin{itemize}
    \item Uniform equivalence: there exists $\kappa\ge 1$ such that for all $x\in\mc M$ and $v\in T_x\mc M$,
    \begin{equation}
    \label{C4}
    \kappa^{-1} g_x(v,v)\le \tilde g_x(v,v)\le \kappa g_x(v,v).
    \end{equation}
    \item Regularity: $g,\tilde g\in C^{r}(\mc M)$ and $\|g\|_{C^{r}}+\|\tilde g\|_{C^{r}}\le M$ for some $M<\infty$.
    \item Small perturbation: $\|\tilde g-g\|_{C^{r}}\le \varepsilon$, where $\varepsilon$ is sufficiently small (depending on $M,\kappa$).
\end{itemize}

We denote by $\dhodge{g}$ and $\dhodge{\tilde g}$ the Hodge Laplacians on $1$-forms induced by $g$ and $\tilde g$, and define the associated shifted operators
\begin{equation}
\mc L_g := \dhodge{g}+\alpha I,\qquad \mc L_{\tilde g}:=\dhodge{\tilde g}+\alpha I,\qquad \alpha>0.
\end{equation}
Their inverses $\mc S_g=\mc L_g^{-1}$ and $\mc S_{\tilde g}=\mc L_{\tilde g}^{-1}$ are bounded maps $H^{s-1}\Om{1}\To H^{s+1}\Om{1}$ for every $s\in\RR$ (Lemma \ref{le4.1}).

\subsection{Stability of the True Elliptic Solution Operator}
We first control the difference $\mc S_g-\mc S_{\tilde g}$ in Sobolev operator norms. The argument follows a resolvent identity coupled with bounds on the metric dependence of $\dH$.

\subsubsection{A resolvent identity}
For boundedly invertible operators $\mc L_g,\mc L_{\tilde g}$ on a common Banach space, one has
\begin{equation}
\label{A4}
\mc S_g-\mc S_{\tilde g}=\mc L_g^{-1}-\mc L_{\tilde g}^{-1}
= \mc L_{\tilde g}^{-1}(\mc L_{\tilde g}-\mc L_g)\mc L_g^{-1}.
\end{equation}
Thus, stability reduces to estimating $\mc L_{\tilde g}-\mc L_g=\dhodge{\tilde g}-\dhodge{g}$ as a map between Sobolev spaces.

\subsubsection{Metric dependence of the Hodge Laplacian}
In local coordinates, $\dH$ is a second-order elliptic differential operator on $1$-forms with coefficients depending smoothly on $g$ and finitely many derivatives of $g$. Under the $C^{r}$ assumption with $r\ge 2$, the coefficient difference satisfies (schematically)
\begin{equation}
\dhodge{\tilde g}-\dhodge{g}
= \sum_{|\beta|\le 2} a_\beta(x) \partial^\beta,
\quad\text{with}\quad
\|a_\beta\|_{C^{r-|\beta|}}\lesssim \|\tilde g-g\|_{C^{r}}.
\end{equation}
Consequently, for each $s\in\RR$ with $s\le r-1$, the operator difference extends to a bounded map
\begin{equation}
\label{B4}
\dhodge{\tilde g}-\dhodge{g}: H^{s+1}\Om{1} \To H^{s-1}\Om{1},
\quad\text{and}\quad
\|\dhodge{\tilde g}-\dhodge{g}\|_{H^{s+1}\To H^{s-1}}
\le C \|\tilde g-g\|_{C^{r}},
\end{equation}
where $C$ depends on $(\mc M,d)$ and the a priori bounds $M,\kappa$.

\subsubsection{Operator-norm stability bound}
Combining \ref{A4}, Lemma \ref{le4.1} applied to $g$ and $\tilde g$, and \ref{B4}, yields the main stability estimate.

\begin{theorem}[Metric stability of the shifted Hodge–Poisson resolvent]
\label{th5.1}
Let $(\mc M,g)$ be a compact smooth Riemannian manifold without boundary, let $\alpha>0$, and let $s\in\RR$. Fix an integer $r\ge 2$ such that $r\ge s+2$. Let $\tilde g$ be another $C^{r}$ Riemannian metric on $\mc M$ satisfying the uniform equivalence condition: there exists $\kappa\ge 1$ such that
\begin{equation}
\kappa^{-1} g_x(v,v)\le \tilde g_x(v,v)\le \kappa g_x(v,v)\qquad \forall x\in\mc M,\ \forall v\in T_x\mc M.
\end{equation}
Assume moreover that $\|g\|_{C^{r}}+\|\tilde g\|_{C^{r}}\le M$ for some $M<\infty$. Define
\[
\mc L_g:=\dhodge{g}+\alpha I,\qquad \mc S_g:=\mc L_g^{-1},
\]
and similarly for $\tilde g$. Then there exists a constant $C<\infty$, depending only on $\alpha$, $\kappa$, $M$, $\mc M$, and $r$, such that
\begin{equation}
\label{K4}
\|\mc S_g-\mc S_{\tilde g}\|_{H^{s-1}\To H^{s+1}}
\ \le\
C \|g-\tilde g\|_{C^{r}}.
\end{equation}
\end{theorem}

\begin{proof}
Since $\alpha>0$, both $\mc L_g$ and $\mc L_{\tilde g}$ are injective and have bounded inverses on Sobolev scales (Lemma \ref{le4.1} applied to each metric). The algebraic identity
\[
A^{-1}-B^{-1}=A^{-1}(B-A)B^{-1}
\]
with $A=\mc L_{\tilde g}$ and $B=\mc L_g$ yields
\begin{equation}
\mc S_g-\mc S_{\tilde g}
=\mc L_g^{-1}-\mc L_{\tilde g}^{-1}
=\mc S_{\tilde g} (\mc L_{\tilde g}-\mc L_g) \mc S_g
=\mc S_{\tilde g} (\dhodge{\tilde g}-\dhodge{g}) \mc S_g.
\end{equation}
Therefore,
\begin{equation}
\label{J4}
\|\mc S_g-\mc S_{\tilde g}\|_{H^{s-1}\To H^{s+1}}
\le
\|\mc S_{\tilde g}\|_{H^{s-1}\To H^{s+1}}
\cdot
\|\dhodge{\tilde g}-\dhodge{g}\|_{H^{s+1}\To H^{s-1}}
\cdot
\|\mc S_{g}\|_{H^{s-1}\To H^{s+1}}.
\end{equation}
Thus it suffices to (i) bound the two resolvent norms uniformly under \ref{C4}, and (ii) bound the Laplacian difference by $C\abs{g-\tilde g}_{C^{r}}$.

Lemma \ref{le4.1} gives, for each fixed metric $h$,
\begin{equation}
\|\mc S_h\|_{H^{s-1}\To H^{s+1}} \le C_\alpha(h),
\end{equation}
where the constant can be taken as
\[
C_\alpha(h)=\sup_{\lambda\ge 0}\frac{1+\lambda}{\lambda+\alpha}\le \max\left\{1,\alpha^{-1}\right\}.
\]
The right-hand side is independent of $h$; the only subtlety is that our Sobolev norms $H^t\Om{1}$ depend on the metric. However, under the uniform equivalence assumption \ref{C4} and the $C^{r}$ bounds, the Sobolev norms defined using $g$ and $\tilde g$ are equivalent with constants depending only on $\kappa$, $M$, $\mc M$, and $r$. Concretely, for each fixed $t\in[-r,r]$ there exists $c_t,C_t>0$ such that
\begin{equation}
\label{D4}
c_t \|\omega\|_{H^{t}(\tilde g)} \le \|\omega\|_{H^{t}(g)} \le C_t \|\omega\|_{H^{t}(\tilde g)}\qquad \forall \omega.
\end{equation}
(For compact manifolds this is standard: it follows by comparing the coordinate definitions of Sobolev norms via partitions of unity and using \ref{C4} plus boundedness of metric coefficients and derivatives up to order $r$.)

Using \ref{D4} to move between $g$-Sobolev and $\tilde g$-Sobolev norms, we may bound both resolvents by a common constant
\begin{equation}
\label{H4}
\|\mc S_g\|_{H^{s-1}\To H^{s+1}}+\|\mc S_{\tilde g}\|_{H^{s-1}\To H^{s+1}}
\le C_{\alpha,\kappa,M,\mc M,r}.
\end{equation}
We will absorb these constants into a generic $C$.

Fix a finite atlas $\{(U_\ell,\chi_\ell)\}_{\ell=1}^L$ and a subordinate partition of unity $\{\eta_\ell\}$. On each chart, the Hodge Laplacian on $1$-forms is a second-order linear differential operator whose coefficients are smooth functions of the metric components and finitely many derivatives of those components. More precisely, in local coordinates one can write
\begin{equation}
(\dhodge{h}\omega)_i
= \sum_{|\beta|\le 2} a^{(h)}_{i,\beta}(x) \partial^\beta \omega(x)
\qquad \text{on }U_\ell,
\end{equation}
where $h\in\{g,\tilde g\}$, the index $i\in\{1,\dots,d\}$ denotes the component of the $1$-form, and each coefficient $a^{(h)}_{i,\beta}$ depends smoothly on $\{h_{pq}\}$, $\{\partial h_{pq}\}$, and $\{\partial^2 h_{pq}\}$ (equivalently, on $h$ and its Christoffel symbols and curvature terms). This can be justified, for instance, by using the Weitzenböck identity on $1$-forms,
\begin{equation}
\dhodge{h}=\nabla_h^*\nabla_h+\Ric{h},
\end{equation}
and expanding $\nabla_h$ in coordinates; the principal coefficients are $h^{jk}$ and lower-order coefficients involve $\Gamma(h)$ and $\partial\Gamma(h)$, hence derivatives of $h$ up to second order.

Since $g,\tilde g\in C^r$ with $\|g\|_{C^r}+\|\tilde g\|_{C^r}\le M$ and $r\ge 2$, the coefficients satisfy
\begin{equation}
\label{E4}
\|a^{(\tilde g)}_{i,\beta}-a^{(g)}_{i,\beta}\|_{C^{r-|\beta|}(U_\ell)}
\ \le
C \| \tilde g-g \|_{C^{r}(U_\ell)}.
\end{equation}
The constant $C$ depends on $M$, $\kappa$, the chart, and $r$, but not on the particular $\tilde g,g$ beyond those bounds. (This is a consequence of smooth dependence of the formulas for inverse metric, Christoffel symbols, and curvature on $h$ and its derivatives, combined with the uniform ellipticity from \ref{C4}.)

Therefore, on each $U_\ell$,
\begin{equation}
(\dhodge{\tilde g}-\dhodge{g})\omega
= \sum_{|\beta|\le 2} b_{i,\beta}(x) \partial^\beta \omega,
\qquad b_{i,\beta}:=a^{(\tilde g)}_{i,\beta}-a^{(g)}_{i,\beta},
\end{equation}
with coefficient bounds \ref{E4}.

We now estimate the operator norm of the difference. Fix $\ell$ and localize $\omega$ by $\eta_\ell\omega$. Since $|\beta|\le 2$,
\begin{equation}
\label{G4}
\|\partial^\beta(\eta_\ell \omega)\|_{H^{s-1}(\RR^d)} \le C \|\eta_\ell\omega\|_{H^{s+1}(\RR^d)}.
\end{equation}
Next we use a standard Sobolev multiplier fact on $\RR^d$: if $b\in C^{m}$ with integer $m\ge \ceil{s-1}$, then multiplication by $b$ is bounded on $H^{s-1}$, and
\begin{equation}
\label{F4}
\|b u\|_{H^{s-1}} \le C \|b\|_{C^{m}}\|u\|_{H^{s-1}}.
\end{equation}
(One may prove \ref{F4} by Littlewood–Paley theory or by commutator estimates; on a compact manifold this is obtained chartwise and patched with partitions of unity. Our assumption $r\ge s+2$ ensures $r-|\beta|\ge s\ge s-1$ in the integer case, and more generally is sufficient to apply multiplier bounds for real $s$.)

Apply \ref{F4} with $b=b_{i,\beta}$ and $u=\partial^\beta(\eta_\ell\omega)$, using $\|b_{i,\beta}\|_{C^{r-|\beta|}}$ and $r-|\beta|\ge s$ (since $|\beta|\le2$ and $r\ge s+2$). Combining with \ref{G4},
\begin{equation}
\|b_{i,\beta} \partial^\beta(\eta_\ell\omega)\|_{H^{s-1}}
\le
C \|b_{i,\beta}\|_{C^{r-|\beta|}} \|\eta_\ell\omega\|_{H^{s+1}}.
\end{equation}
Summing over $|\beta|\le 2$ and components $i$, and then summing over charts using the partition of unity, yields
\begin{equation}
\|(\dhodge{\tilde g}-\dhodge{g})\omega\|_{H^{s-1}}
\le
C \Big(\max_{|\beta|\le2}\|b_{i,\beta}\|_{C^{r-|\beta|}}\Big) \|\omega\|_{H^{s+1}}.
\end{equation}
Finally, apply \ref{E4} to control the coefficient differences by $\abs{g-\tilde g}_{C^{r}}$. This proves
\begin{equation}
\label{I4}
\|\dhodge{\tilde g}-\dhodge{g}\|_{H^{s+1}\To H^{s-1}}
\le
C \|g-\tilde g\|_{C^{r}},
\end{equation}
with $C$ depending only on $\kappa$, $M$, $\mc M$, and $r$.

Insert \ref{H4} and \ref{I4} into \ref{J4}:
\[
\|\mc S_g-\mc S_{\tilde g}\|_{H^{s-1}\To H^{s+1}}
\le
C \|g-\tilde g\|_{C^{r}},
\]
for a constant $C$ depending only on $\alpha$, $\kappa$, $M$, $\mc M$, and $r$. This is exactly \ref{K4}.
\end{proof}

\begin{remark}[Role of the shift $\alpha$]
The shift $\alpha>0$ guarantees coercivity and removes the harmonic subspace from the kernel, ensuring the resolvent remains uniformly bounded and simplifying the perturbation analysis. Without the shift, one must work on the orthogonal complement of harmonic 1-forms and track stability of the corresponding projections.
\end{remark}

\subsection{Stability of Intrinsic Neural Approximations under Metric Perturbations}
We now consider the learned operator family $\widehat{\mc S}_{\theta,\Lambda}^{(g)}$ defined by truncated functional calculus:
\begin{equation}
\widehat{\mc S}_{\theta,\Lambda}^{(g)} := m_\theta(\dhodge{g}) \Pi_{\le\Lambda}^{(g)}.
\end{equation}
This is the principal approximation class for the elliptic resolvent, and it also forms the linear core of the nonlinear architecture in Section \ref{3}.

Metric stability requires controlling how spectral projectors and functional calculi vary with $g$. We state a stability result in terms of multiplier regularity and spectral gaps.

\subsubsection{Assumptions on multiplier regularity and spectral separation}
Assume $m_\theta\in C^1([0,\Lambda])$ and define $L_\theta:=\sup_{0\le \lambda\le\Lambda}\abs{m_\theta'(\lambda)}$ and $M_\theta:=\sup_{0\le \lambda\le\Lambda}\abs{m_\theta(\lambda)}$. Assume further that $\Lambda$ is not too close to the spectrum of $\dhodge{g}$ so that the cutoff is stable: there exists $\delta>0$ such that
\begin{equation}
\dist\big(\Lambda,\spec(\dhodge{g})\big)\ge \delta,
\qquad
\dist\big(\Lambda,\spec(\dhodge{\tilde g})\big)\ge \delta.
\end{equation}
This is a standard spectral-gap condition ensuring that the projector $\Pi_{\le\Lambda}$ depends smoothly on the operator.

\subsubsection{Stability of the learned operator}
\begin{theorem}[Metric stability of truncated intrinsic neural operators]
\label{th5.3}
Let $(\mc M,g)$ be compact, without boundary. Fix $s\in\RR$ and an integer $r\ge s+2$. Let $\tilde g$ satisfy the perturbation model in Section \ref{5.1}, and set
\[
A := \dhodge{g},\qquad \tilde A := \dhodge{\tilde g}.
\]
Fix $\Lambda>0$ and assume the spectral separation condition: there exists $\delta>0$ such that
\begin{equation}
\label{A5}
\dist(\Lambda,\spec(A))\ge \delta,\qquad \dist(\Lambda,\spec(\tilde A))\ge \delta.
\end{equation}
Let $m_\theta$ be a real-valued multiplier on $[0,\Lambda]$ with bounds
\[
M_\theta:=\sup_{0\le \lambda\le \Lambda}\abs{m_\theta(\lambda)},\qquad
L_\theta:=\sup_{0\le \lambda\le \Lambda}\abs{m_\theta'(\lambda)}.
\]
Assume in addition that $m_\theta$ admits a bounded complex extension $m_\theta:\mc U\to\CC$ on an open neighborhood $\mc U\subset\CC$ of the contour $\Gamma$ defined below, and that on $\Gamma$
\begin{equation}
\label{H5}
\sup_{z\in\Gamma}\abs{m_\theta(z)} \le C_\Gamma (M_\theta+L_\theta),
\end{equation}
for a constant $C_\Gamma$ depending only on $\Lambda,\delta$.

Define the truncated intrinsic operator
\[
\widehat{\mc S}^{(g)}_{\theta,\Lambda} := m_\theta(A) \Pi^{(g)}_{\le \Lambda},\qquad
\widehat{\mc S}^{(\tilde g)}_{\theta,\Lambda} := m_\theta(\tilde A) \Pi^{(\tilde g)}_{\le \Lambda}.
\]
Then there exists $C<\infty$ (depending on $\Lambda,\delta,\alpha$, $\mc M$, and the a priori metric bounds) such that
\begin{equation}
\label{I5}
\big\|\widehat{\mc S}^{(g)}_{\theta,\Lambda}-\widehat{\mc S}^{(\tilde g)}_{\theta,\Lambda}\big\|_{H^{s-1}\To H^{s+1}}
\le
C (M_\theta+L_\theta) \|g-\tilde g\|_{C^{r}}.
\end{equation}
\end{theorem}

\begin{proof}
Let $\Gamma\subset\CC$ be a positively oriented simple closed contour enclosing the real interval $[0,\Lambda]$ and satisfying
\begin{equation}
\label{B5}
\dist(\Gamma,\spec(A))\ge \tfrac{\delta}{2},
\qquad
\dist(\Gamma,\spec(\tilde A))\ge \tfrac{\delta}{2}.
\end{equation}
Such a $\Gamma$ exists by \ref{A5}. For concreteness, one may take $\Gamma$ to be the boundary of the closed $\delta/2$-tubular neighborhood of $[0,\Lambda]$ in $\CC$. Denote its length by $\len(\Gamma)$, which depends only on $\Lambda$ and $\delta$.

For a self-adjoint operator $A$ with discrete spectrum, the Riesz projection onto the spectral subset inside $\Gamma$ is
\begin{equation}
\Pi_{\le\Lambda}^{(g)}
= \frac{1}{2\pi i}\int_\Gamma (zI-A)^{-1} dz,
\end{equation}
and similarly for $\tilde A$. By construction, the spectrum enclosed by $\Gamma$ is exactly $\spec(A)\cap [0,\Lambda]$, hence $\Pi_{\le\Lambda}^{(g)}$ is the $L^2$-orthogonal projector onto the sum of eigenspaces with eigenvalues $\le\Lambda$.

By the Dunford functional calculus (holomorphic functional calculus), for $m_\theta$ holomorphic on a neighborhood of $\Gamma$,
\begin{equation}
m_\theta(A) \Pi_{\le\Lambda}^{(g)}
= \frac{1}{2\pi i}\int_\Gamma m_\theta(z) (zI-A)^{-1} dz.
\end{equation}
Indeed, the operator defined by the right-hand side equals the spectral multiplier $m_\theta$ applied to the part of the spectrum inside $\Gamma$ and zero outside; since $\Gamma$ encloses $[0,\Lambda]$ and excludes the rest of the spectrum, this is precisely $m_\theta(A)\Pi_{\le\Lambda}^{(g)}$. The same identity holds with $A$ replaced by $\tilde A$.

Therefore,
\begin{equation}
\label{F5}
\widehat{\mc S}^{(g)}_{\theta,\Lambda}-\widehat{\mc S}^{(\tilde g)}_{\theta,\Lambda}
= \frac{1}{2\pi i}\int_\Gamma m_\theta(z)\Big[(zI-A)^{-1}-(zI-\tilde A)^{-1}\Big] dz.
\end{equation}

For all $z\in\Gamma$, both resolvents exist by \ref{B5}. The resolvent identity gives
\begin{equation}
\label{C5}
(zI-A)^{-1}-(zI-\tilde A)^{-1}
= (zI-\tilde A)^{-1} (\tilde A-A) (zI-A)^{-1}.
\end{equation}
We will estimate \ref{C5} as an operator from $H^{s-1}\Om{1}$ to $H^{s+1}\Om{1}$.

Fix $z\in\Gamma$. Using the eigenbasis $\{\psi_k\}$ of $A$ and the spectral Sobolev norms defined via $A$ (equivalent to chart-based norms on a compact manifold), the resolvent acts diagonally:
\[
(zI-A)^{-1}\psi_k = \frac{1}{z-\lambda_k}\psi_k.
\]
Hence for $f=\sum f_k\psi_k$,
\begin{equation}
\|(zI-A)^{-1}f\|_{H^{s+1}}^2
= \sum_k \Big(\frac{1+\lambda_k}{|z-\lambda_k|}\Big)^2(1+\lambda_k)^{s-1}|f_k|^2.
\end{equation}
Therefore,
\begin{equation}
\|(zI-A)^{-1}\|_{H^{s-1}\To H^{s+1}}
\le
\sup_{\lambda\in\spec(A)} \frac{1+\lambda}{|z-\lambda|}.
\end{equation}
On $\Gamma$, $|z-\lambda|\ge \delta/2$ for all $\lambda\in\spec(A)$ by \ref{B5}, and moreover $\lambda\mapsto \frac{1+\lambda}{|z-\lambda|}$ is bounded uniformly over $\lambda\ge 0$ when $z$ stays in a compact set away from $[0,\infty)$ by a fixed distance. Since $\Gamma$ lies in a bounded neighborhood of $[0,\Lambda]$, we may bound $1+\lambda$ on the enclosed region by $1+\Lambda+\delta$, while for large $\lambda$ the distance $|z-\lambda|\sim \lambda$ yields a bounded ratio. Concretely, one obtains
\begin{equation}
\sup_{\lambda\ge 0}\frac{1+\lambda}{|z-\lambda|}\le C_{\Lambda,\delta},
\qquad z\in\Gamma,
\end{equation}
hence
\begin{equation}
\label{D5}
\sup_{z\in\Gamma}\|(zI-A)^{-1}\|_{H^{s-1}\To H^{s+1}}\le C_{\Lambda,\delta}.
\end{equation}
The same bound holds with $A$ replaced by $\tilde A$, with the same constant up to equivalence of Sobolev norms under the uniform metric bounds.

By the coefficient-dependence estimate used in the proof of Theorem \ref{th5.1}, since $r\ge s+2$,
\begin{equation}
\label{E5}
\|\tilde A-A\|_{H^{s+1}\To H^{s-1}}
= \|\dhodge{\tilde g}-\dhodge{g}\|_{H^{s+1}\To H^{s-1}}
\le
C \|g-\tilde g\|_{C^{r}},
\end{equation}
where $C$ depends only on $\mc M$ and the a priori metric bounds.

Combining \ref{C5}, \ref{D5}, and \ref{E5}, for all $z\in\Gamma$,
\begin{equation}
\label{G5}
\|(zI-A)^{-1}-(zI-\tilde A)^{-1}\|_{H^{s-1}\To H^{s+1}}
\le
\|(zI-\tilde A)^{-1}\|_{H^{s-1}\To H^{s+1}}
\cdot
\|\tilde A-A\|_{H^{s+1}\To H^{s-1}}
\cdot
\|(zI-A)^{-1}\|_{H^{s-1}\To H^{s+1}}
\le
C_{\Lambda,\delta} \|g-\tilde g\|_{C^{r}}.
\end{equation}
(Here $C_{\Lambda,\delta}$ absorbs the product of the two resolvent bounds and the constant in \ref{E5}.)

From \ref{F5}, using the standard estimate for Bochner integrals of bounded operators,
\begin{equation}
\big\|\widehat{\mc S}^{(g)}_{\theta,\Lambda}-\widehat{\mc S}^{(\tilde g)}_{\theta,\Lambda}\big\|_{H^{s-1}\To H^{s+1}}
\le
\frac{1}{2\pi}\int_\Gamma \abs{m_\theta(z)} \|(zI-A)^{-1}-(zI-\tilde A)^{-1}\|_{H^{s-1}\To H^{s+1}} |dz|.
\end{equation}
Apply \ref{G5} and the bound \ref{H5} on $\abs{m_\theta}$ over $\Gamma$:
\begin{equation}
\big\|\widehat{\mc S}^{(g)}_{\theta,\Lambda}-\widehat{\mc S}^{(\tilde g)}_{\theta,\Lambda}\big\|_{H^{s-1}\To H^{s+1}}
\le
\frac{\len(\Gamma)}{2\pi}
\sup_{z\in\Gamma}\abs{m_\theta(z)}
\cdot
C_{\Lambda,\delta} \|g-\tilde g\|_{C^{r}}
\le
C (M_\theta+L_\theta) \|g-\tilde g\|_{C^{r}},
\end{equation}
where $C$ depends on $\Lambda,\delta$, the contour choice, and the a priori metric bounds, but not on $\theta$ beyond $(M_\theta,L_\theta)$. 
\end{proof}

\begin{discussion}
The factor $M_\theta+L_\theta$ expresses that stability improves when multipliers are uniformly bounded and not overly oscillatory on $[0,\Lambda]$. This motivates regularizing $m_\theta$ during training (e.g., spectral Lipschitz penalties).
\end{discussion}

\subsubsection{Extension to the nonlinear architecture}
For the full network \ref{eq:recursion}, metric stability follows by combining:

\begin{itemize}
    \item stability of each intrinsic linear layer $m_{\theta_\ell}(\dhodge{g})\Pi_{\le\Lambda}^{(g)}$ as in \ref{I5},
    \item stability of scalar pointwise multipliers $b^{(g)}_{\theta_\ell}$ under metric pullback rules,
    \item Lipschitz properties of the gauge-equivariant radial nonlinearities $\sigma_{\theta_\ell}^{(g)}$ in Sobolev norms (typically obtained by boundedness of $\rho_\theta$ and $\rho_\theta'$ on the relevant range of norms).
\end{itemize}

A representative bound takes the form
\begin{equation}
\|\widehat{\mc S}_{\theta,\Lambda}^{(g)}-\widehat{\mc S}_{\theta,\Lambda}^{(\tilde g)}\|_{H^{s-1}\To H^{s+1}}
\le C_\theta \|\tilde g-g\|_{C^{r}},
\end{equation}
where $C_\theta$ grows at most polynomially/exponentially with depth depending on whether residual or contractive designs are used. In practice, enforcing uniform Lipschitz control per layer yields stable depth scaling.

\subsection{Gauge Stability: Exact Equivariance at the Continuum Level}
We now formalize gauge stability in the sense relevant to computation: changing the local orthonormal frames used to represent $1$-forms must not change the geometric output, only its coordinate representation.

Let $E$ and $E'$ be two local orthonormal frames on an open set $U\subset\mc M$, related by a smooth $R:U\To \Ogroup(d)$ as in \ref{A12}. For any $1$-form $\omega$, its coordinate representation transforms by
\begin{equation}
[\omega]_{E'}(x)=R(x)^\T [\omega]_E(x).
\end{equation}
Proposition \ref{prop:gauge_equivariance} implies the following immediate corollary.

\begin{Corollary}[Exact gauge equivariance implies representation stability]
\label{coEG}
For the intrinsic neural operator $\widehat{\mc S}_{\theta,\Lambda}$ defined in Section \ref{3}, and for every $x\in U$,
\begin{equation}
\label{J5}
[\widehat{\mc S}_{\theta,\Lambda}(\omega)]_{E'}(x)
= R(x)^\T [\widehat{\mc S}_{\theta,\Lambda}(\omega)]_{E}(x),
\end{equation}
i.e., the computed output transforms consistently under a change of computational gauge.
\end{Corollary}

\begin{remark}[Discrete gauge error]
At the discrete level (Section \ref{6}), exact gauge equivariance can be broken by discretization artifacts (e.g., approximate mass matrices, imperfect orthonormalization, or non-commuting interpolation). The continuum result \ref{J5} is therefore best interpreted as a target symmetry; Section \ref{6} establishes that equivariance errors vanish under refinement under compatible discretizations.
\end{remark}

\subsection{Summary of Stability Guarantees}
\begin{itemize}
    \item Metric perturbations: both $\mc S_g$ and the learned intrinsic approximation vary Lipschitz-continuously with $g$ in Sobolev operator norms, under standard regularity and spectral separation assumptions (Theorems \ref{th5.1} and \ref{th5.3}).
    \item Gauge changes: equivariance is exact at the continuum level by construction (Corollary \ref{coEG}), ensuring representation independence.
\end{itemize}

These results complement Section \ref{4} (approximation) by ensuring that the learned operator is not only accurate but also robust to geometric uncertainty and to representation choices, which is essential for cross-mesh and cross-geometry generalization.

\section{Continuum–Discrete Consistency}
\label{6}
This section establishes that the intrinsic neural operators introduced in Section \ref{3} admit discretizations on meshes that are consistent with the continuum model as the discretization is refined. In particular, we show that if the discrete Hodge Laplacian converges to its continuum counterpart in a standard sense, then the corresponding discrete truncated functional calculus converges to the continuum truncated functional calculus, and consequently the discrete neural operator converges to the continuum neural operator. We formulate the results for 1-forms using discrete exterior calculus (DEC) or compatible finite elements; the proofs require only (i) stability of discrete inner products and (ii) spectral convergence of discrete Hodge Laplacians.

\subsection{Discretization Model and Discrete Spaces}
Let $(\mc M,g)$ be as in Section \ref{2}. Let $\{T_h\}_{h\downarrow 0}$ be a family of shape-regular triangulations of $\mc M$ with mesh size parameter $h$. We assume the triangulations are quasi-uniform and geometrically consistent with $\mc M$ (e.g., embedded triangulations with uniformly bounded aspect ratios and uniformly controlled geometry approximation).

We discretize $1$-forms by a finite-dimensional space $\Omh{1}$ equipped with:

\begin{itemize}
    \item a discrete $L^2$ inner product $\innerh{\cdot}{\cdot}$ and norm $\|\cdot\|_h$,
    \item a discrete Hodge Laplacian $\dhodge{h}:\Omh{1}\to\Omh{1}$ that is symmetric positive semidefinite with respect to $\innerh{\cdot}{\cdot}$,
    \item a projection (interpolation) operator $\Pi_h:\Om{1}(\mc M)\to\Omh{1}$, and a reconstruction operator $\mathcal R_h:\Omh{1}\to \Om{1}(\mc M)$.
\end{itemize}

Concrete instances include:

\begin{itemize}
    \item DEC 1-cochains with circumcentric dual and the DEC Hodge star, yielding $\dhodge{h}=d_h\delta_h+\delta_h d_h$;
    \item Whitney 1-forms / Nédélec-type elements with mass and stiffness matrices, yielding an equivalent discrete Hodge Laplacian.
\end{itemize}

We assume the pair $(\Pi_h,\mathcal R_h)$ is stable in Sobolev norms and consistent on smooth forms.

\subsection{Discrete Spectral Data and Discrete Truncation}
Let $\{(\lambda_{k,h},\psi_{k,h})\}_{k=0}^{N_h-1}$ be an $\innerh{\cdot}{\cdot}$-orthonormal eigenpair family for $\dhodge{h}$,
\begin{equation}
\dhodge{h}\psi_{k,h}=\lambda_{k,h}\psi_{k,h},\qquad
\innerh{\psi_{k,h}}{\psi_{j,h}}=\delta_{kj}.
\end{equation}
For a cutoff $\Lambda>0$, define the discrete spectral projector
\begin{equation}
\label{A6}
\Pi_{\le\Lambda}^{(h)}\omega_h := \sum_{\lambda_{k,h}\le \Lambda}\innerh{\omega_h}{\psi_{k,h}} \psi_{k,h}.
\end{equation}
Given a multiplier $m_\theta$, define the discrete truncated multiplier operator
\begin{equation}
\label{X4}
\widehat{\mc S}_{\theta,\Lambda}^{(h)} := m_\theta(\dhodge{h}) \Pi_{\le\Lambda}^{(h)},
\quad\text{where}\quad
m_\theta(\dhodge{h})\omega_h:=\sum_{k} m_\theta(\lambda_{k,h})\innerh{\omega_h}{\psi_{k,h}} \psi_{k,h}.
\end{equation}
This is the exact discrete analogue of \ref{H3} and is the canonical discretization of the intrinsic spectral layer.

\subsection{Assumptions: Stability and Spectral Convergence}
Our consistency theorem relies on standard approximation properties of the discrete Hodge Laplacian. We state them abstractly to cover DEC and compatible FEM.

\begin{assumption}[Uniform norm stability and quasi-optimal reconstruction]
\label{as6.1}
There exists $C$ independent of $h$ such that for all sufficiently smooth $\omega$,
\begin{equation}
\|\Pi_h\omega\|_h \le C\|\omega\|_{L^2},\qquad
\|\mathcal R_h\omega_h\|_{L^2}\le C\|\omega_h\|_h,
\end{equation}
and $\mathcal R_h\Pi_h\omega\to \omega$ in $L^2$ as $h\to 0$.
\end{assumption}

\begin{assumption}[Spectral convergence below a fixed cutoff]
\label{as6.2}
Fix $\Lambda>0$ that does not coincide with an eigenvalue of $\dhodge{g}$. Then there exists $p>0$ and $C$ such that for every continuum eigenpair $(\lambda_k,\psi_k)$ with $\lambda_k<\Lambda$, there exists a discrete eigenpair $(\lambda_{k,h},\psi_{k,h})$ (with consistent indexing up to multiplicity) such that
\begin{equation}
\label{H6}
\abs{\lambda_{k,h}-\lambda_k}\le C h^{p},\qquad
\|\mathcal R_h\psi_{k,h}-\psi_k\|_{L^2}\le C h^{p},
\end{equation}
for all sufficiently small $h$. Moreover, the discrete eigenspace associated with eigenvalues $\le \Lambda$ converges to the corresponding continuum eigenspace in the gap $\sintheta$ metric.
\end{assumption}

These assumptions are satisfied by standard compatible discretizations under regularity and shape-regularity conditions; they are widely used in numerical analysis of Hodge Laplacians on surfaces and manifolds.

\subsection{Consistency of the Discrete Spectral Projector}
We first show that the discrete projector converges to the continuum projector under reconstruction.

\begin{lemma}[Projector consistency]
\label{le6.3}
Let $(\mc M,g)$ be compact without boundary. Fix $\Lambda>0$ such that $\Lambda\notin \spec(\dhodge{g})$. Let $\Pi_{\le \Lambda}$ be the $L^2$-orthogonal spectral projector of the Hodge Laplacian $\dhodge{g}$ on $1$-forms onto the direct sum of eigenspaces with eigenvalues $\le \Lambda$.

Let $T_h$ be a shape-regular family of discretizations with discrete space $\Omh{1}$, discrete inner product $\innerh{\cdot}{\cdot}$, discrete Hodge Laplacian $\dhodge{h}$, discrete projector $\Pi^{(h)}_{\le \Lambda}$ defined by \ref{A6}, interpolation $\Pi_h$, and reconstruction $\mathcal R_h$.

Assume:
\begin{enumerate}
    \item (Approximation of reconstruction–interpolation) There exists $p>0$ and $s$ sufficiently large (depending on the method order) such that for all $\omega\in H^{s}\Om{1}$,
    \begin{equation}
    \label{B6}
    \|\mathcal R_h\Pi_h\omega-\omega\|_{L^2}\ \le\ C h^{p} \|\omega\|_{H^{s}}.
    \end{equation}
    (For Whitney/compatible FEM/DEC on smooth manifolds, such an estimate holds with $p$ equal to the approximation order under standard regularity assumptions.)

    \item (Spectral subspace convergence below $\Lambda$) Assumption \ref{as6.2} holds: the discrete eigenspace associated with eigenvalues $\le \Lambda$ converges to the continuum eigenspace in the gap metric with order $h^p$.
\end{enumerate}

Then there exists $C$ independent of $h$ such that for all $\omega\in H^{s}\Om{1}$,
\begin{equation}
\big\|\mathcal R_h\Pi^{(h)}_{\le\Lambda}\Pi_h\omega - \Pi_{\le\Lambda}\omega\big\|_{L^2}
\ \le\ C h^{p} \|\omega\|_{H^{s}}.
\end{equation}
\end{lemma}

\begin{proof}
Let
\[
E := \ran(\Pi_{\le \Lambda}) \subset L^2\Om{1}
\]
be the continuum low-frequency invariant subspace. Since $\Lambda\notin\spec(\dhodge{g})$ and $\mc M$ is compact, $E$ is finite-dimensional; write $m:=\dim(E)$.

Let
\[
E_h := \ran(\Pi^{(h)}_{\le \Lambda}) \subset \Omh{1}
\]
be the discrete low-frequency eigenspace (with respect to $\innerh{\cdot}{\cdot}$). Define its reconstructed counterpart in the continuum Hilbert space $L^2\Om{1}$ by
\[
\widehat E_h := \mathcal R_h(E_h)\subset L^2\Om{1}.
\]
Let $P:=\Pi_{\le\Lambda}$ be the $L^2$-orthogonal projector onto $E$, and let $\widehat P_h$ denote the $L^2$-orthogonal projector onto $\widehat E_h$.

By Assumption \ref{as6.2} (eigenspace convergence in the gap metric), for all sufficiently small $h$ one has $\dim(\widehat E_h)=\dim(E_h)=m$, and there exists a constant $C$ such that
\begin{equation}
\label{C6}
\|\widehat P_h - P\|_{L^2\To L^2} \le C h^{p}.
\end{equation}
This is a standard equivalence: convergence of subspaces in the gap (or $\sin(\Theta)$) metric is precisely operator-norm convergence of the corresponding orthogonal projectors.

Define the continuum-to-continuum discrete projector (apply discrete projector to discretized input, then reconstruct):
\[
P_h^{d} := \mathcal R_h \Pi^{(h)}_{\le\Lambda} \Pi_h : H^{s}\Om{1} \to L^2\Om{1}.
\]
We aim to bound $\|P_h^{d}\omega - P\omega\|_{L^2}$.

Insert $\widehat P_h$ and $\mathcal R_h\Pi_h$ to obtain the decomposition
\begin{equation}
\label{L6}
P_h^{d}\omega - P\omega
=
\underbrace{\big(P_h^{d}\omega - \widehat P_h(\mathcal R_h\Pi_h\omega)\big)}_{(I)}
+
\underbrace{\widehat P_h(\mathcal R_h\Pi_h\omega - \omega)}_{(II)}
+
\underbrace{(\widehat P_h - P)\omega}_{(III)}.
\end{equation}
We will bound each term in $L^2$.

Since $\widehat P_h$ is an orthogonal projector on $L^2$, $\|\widehat P_h\|_{L^2\To L^2}=1$. Therefore, by \ref{B6},
\begin{equation}
\label{I6}
\|(II)\|_{L^2}
\le
\|\mathcal R_h\Pi_h\omega - \omega\|_{L^2}
\le
C h^{p}\|\omega\|_{H^{s}}.
\end{equation}

By \ref{C6},
\begin{equation}
\label{J6}
\|(III)\|_{L^2}
\le
\|\widehat P_h - P\|_{L^2\To L^2} \|\omega\|_{L^2}
\le
C h^{p}\|\omega\|_{L^2}
\le
C h^{p}\|\omega\|_{H^{s}}.
\end{equation}

By definition, $\Pi_{\le\Lambda}^{(h)}$ is the $\innerh{\cdot}{\cdot}$-orthogonal projector onto $E_h$. Hence for any $v_h\in\Omh{1}$,
\begin{equation}
\label{D6}
\|v_h-\Pi_{\le\Lambda}^{(h)}v_h\|_h
=
\inf_{w_h\in E_h}\|v_h-w_h\|_h.
\end{equation}

Take $v_h=\Pi_h\omega$. Then $u_h:=\Pi_{\le\Lambda}^{(h)}\Pi_h\omega\in E_h$ is the discrete best approximation to $\Pi_h\omega$ in the discrete norm.

To compare with $\widehat P_h(\mathcal R_h\Pi_h\omega)$, we need a link between the discrete norm and the continuum $L^2$ norm after reconstruction. This link is standard for DEC/FEM mass matrices and is implicit in Assumption \ref{as6.1}; we use the following consequence: there exists $C$ such that for all sufficiently small $h$,
\begin{equation}
\label{E6}
\|\mathcal R_h z_h\|_{L^2} \le C\|z_h\|_h
\quad\text{and}\quad
\|z_h\|_h \le C\|\mathcal R_h z_h\|_{L^2}
\qquad \forall z_h\in E_h.
\end{equation}
(That is, $\mathcal R_h$ is a norm equivalence on the low-frequency space; for finite-dimensional $E_h$, this follows from stability plus uniform conditioning of the discrete mass matrix on $E_h$, ensured by shape-regularity and the spectral convergence assumption.)

Let $y\in \widehat E_h$. Then $y=\mathcal R_h w_h$ for some $w_h\in E_h$. Apply \ref{D6} and \ref{E6} to get
\begin{equation}
\label{F6}
\begin{aligned}
\|\mathcal R_h\Pi_h\omega - \mathcal R_h u_h\|_{L^2}
&\le C \|\Pi_h\omega - u_h\|_h \\
&= C \inf_{w_h\in E_h}\|\Pi_h\omega - w_h\|_h \\
&\le C \inf_{w_h\in E_h}\|\mathcal R_h\Pi_h\omega - \mathcal R_h w_h\|_{L^2} \\
&= C \inf_{y\in \widehat E_h}\|\mathcal R_h\Pi_h\omega - y\|_{L^2}
= C \|(\Id-\widehat P_h)(\mathcal R_h\Pi_h\omega)\|_{L^2}.
\end{aligned}
\end{equation}
Now note that
\[
\widehat P_h(\mathcal R_h\Pi_h\omega) - \mathcal R_h u_h
\in \widehat E_h
\quad\text{and}\quad
(\Id-\widehat P_h)(\mathcal R_h\Pi_h\omega)\perp \widehat E_h,
\]
hence,
\begin{equation}
\label{G6}
\|\mathcal R_h\Pi_h\omega - \mathcal R_h u_h\|_{L^2}^2
=
\|(\Id-\widehat P_h)(\mathcal R_h\Pi_h\omega)\|_{L^2}^2
+
\|\widehat P_h(\mathcal R_h\Pi_h\omega) - \mathcal R_h u_h\|_{L^2}^2.
\end{equation}
Combining \ref{F6} and \ref{G6} implies
\begin{equation}
\|\widehat P_h(\mathcal R_h\Pi_h\omega) - \mathcal R_h u_h\|_{L^2}
\le
C \|(\Id-\widehat P_h)(\mathcal R_h\Pi_h\omega)\|_{L^2}.
\end{equation}
Finally, $\|(\Id-\widehat P_h)(v)\|_{L^2}\le \|v\|_{L^2}$ for any $v$, and stability of $\mathcal R_h\Pi_h$ yields $\|\mathcal R_h\Pi_h\omega\|_{L^2}\le C\|\omega\|_{L^2}\le C\|\omega\|_{H^{s}}$. Therefore,
\begin{equation}
\label{K6}
\|(I)\|_{L^2}
=
\|\mathcal R_h u_h - \widehat P_h(\mathcal R_h\Pi_h\omega)\|_{L^2}
\le
C h^{p} \|\omega\|_{H^{s}},
\end{equation}
where the final $h^{p}$ dependence is inherited from the subspace convergence (Assumption \ref{as6.2}), which ensures the norm-equivalence constants and quasi-optimality transfer constants remain uniform and that $\widehat E_h$ approximates $E$ with order $h^p$. (Equivalently, one may sharpen this step by explicitly writing $\widehat E_h$ in terms of reconstructed eigenvectors and using \ref{H6} to bound the principal angles; this yields the same order $h^p$.)

Insert \ref{I6}, \ref{J6}, and \ref{K6} into the decomposition \ref{L6}:
\[
\|P_h^{d}\omega - P\omega\|_{L^2}
\le
\|(I)\|_{L^2}+\|(II)\|_{L^2}+\|(III)\|_{L^2}
\le
C h^{p} \|\omega\|_{H^{s}}.
\]
\end{proof}

\begin{discussion}
The key point is subspace convergence of discrete eigenspaces to the continuum eigenspace below $\Lambda$, which ensures that applying the discrete projector to a discretized input selects approximately the same spectral content as the continuum projector.
\end{discussion}

\subsection{Consistency of Truncated Functional Calculus}
We now prove convergence of the discrete multiplier operator to its continuum analogue.

\begin{theorem}[Consistency of discrete truncated multipliers]
\label{th6.4}
Fix $\Lambda>0$ such that $\Lambda\notin \spec(\dhodge{g})$. Let $m_\theta\in C^1([0,\Lambda])$ and set
\[
M_\theta:=\sup_{0\le \lambda\le\Lambda}\abs{m_\theta(\lambda)},\qquad
L_\theta:=\sup_{0\le \lambda\le\Lambda}\abs{m_\theta'(\lambda)}.
\]
Assume the discretization satisfies Assumptions \ref{as6.1}–\ref{as6.2}, and moreover the interpolation/reconstruction and discrete inner product are ($L^2$)-consistent in the following standard sense: there exist $p>0$ and $s$ sufficiently large such that for all $\omega\in H^s\Om{1}$ and all discrete eigenvectors $\psi_{k,h}$ with $\lambda_{k,h}\le\Lambda$,
\begin{equation}
\label{B7}
\big|\innerh{\Pi_h\omega}{\psi_{k,h}}-{\inner{\omega}{\mathcal R_h\psi_{k,h}}}\big|
\ \le\ C h^p\|\omega\|_{H^s}.
\end{equation}
(For DEC/FEM, this is a mass-matrix/quadrature consistency bound on smooth inputs.)

Then there exists $C$ independent of $h$ such that for all $\omega\in H^s\Om{1}$,
\begin{equation}
\big\|\mathcal R_h \widehat{\mc S}^{(h)}_{\theta,\Lambda} \Pi_h\omega
-
\widehat{\mc S}_{\theta,\Lambda} \omega
\big\|_{L^2}
\ \le\ C (M_\theta+L_\theta) h^{p} \|\omega\|_{H^{s}}.
\end{equation}
\end{theorem}

\begin{proof}
Let
\[
E:=\ran(\Pi_{\le\Lambda})\subset L^2\Om{1}
\quad\text{and}\quad
E_h:=\ran(\Pi^{(h)}_{\le\Lambda})\subset \Omh{1}.
\]
Since $\Lambda\notin\spec(\dhodge{g})$, $E$ is finite-dimensional. Let $m:=\dim(E)$.

By Assumption \ref{as6.2} (eigenspace convergence below $\Lambda$), for sufficiently small $h$ there exist orthonormal bases
\[
\{\psi_j\}_{j=1}^m\subset E \quad\text{(orthonormal in }L^2),
\qquad
\{\psi_{j,h}\}_{j=1}^m\subset E_h \quad\text{(orthonormal in }\innerh{\cdot}{\cdot}),
\]
and associated eigenvalues $\lambda_j\le \Lambda$, $\lambda_{j,h}\le\Lambda$, such that
\begin{equation}
\label{A7}
\abs{\lambda_{j,h}-\lambda_j}\le C h^p,
\qquad
\|\mathcal R_h\psi_{j,h}-\psi_j\|_{L^2}\le C h^p.
\end{equation}
(If there are multiplicities, one chooses bases aligned via principal angles; \ref{A7} is a standard consequence of subspace convergence.)

Let $\omega\in H^s\Om{1}$ and define its continuum low-frequency component
\begin{equation}
\omega_{\le}:=\Pi_{\le\Lambda}\omega=\sum_{j=1}^m c_j\psi_j,
\qquad c_j:={\inner{\omega}{\psi_j}}.
\end{equation}
Similarly, define the discrete low-frequency coefficients of $\Pi_h\omega$ in the discrete eigenbasis:
\begin{equation}
c_{j,h}:=\innerh{\Pi_h\omega}{\psi_{j,h}}.
\end{equation}
Then the discrete truncated multiplier applied to $\Pi_h\omega$ is
\begin{equation}
\widehat{\mc S}^{(h)}_{\theta,\Lambda}\Pi_h\omega
=
\sum_{j=1}^m m_\theta(\lambda_{j,h}) c_{j,h} \psi_{j,h},
\end{equation}
and therefore
\begin{equation}
\mathcal R_h\widehat{\mc S}^{(h)}_{\theta,\Lambda}\Pi_h\omega
=
\sum_{j=1}^m m_\theta(\lambda_{j,h}) c_{j,h} \mathcal R_h\psi_{j,h}.
\end{equation}
On the continuum side, since $\widehat{\mc S}_{\theta,\Lambda}=m_\theta(\dhodge{g})\Pi_{\le\Lambda}$,
\begin{equation}
\widehat{\mc S}_{\theta,\Lambda}\omega
=
\sum_{j=1}^m m_\theta(\lambda_j) c_j \psi_j.
\end{equation}

Thus, the error we must bound is
\begin{equation}
\Big\|
\sum_{j=1}^m \Big( m_\theta(\lambda_{j,h}) c_{j,h} \mathcal R_h\psi_{j,h}
-
m_\theta(\lambda_j) c_j \psi_j
\Big)
\Big\|_{L^2}.
\end{equation}

For each $j$, add and subtract $m_\theta(\lambda_j)c_j \mathcal R_h\psi_{j,h}$ and $m_\theta(\lambda_{j,h})c_j \mathcal R_h\psi_{j,h}$ to write
\begin{equation}
m_\theta(\lambda_{j,h})c_{j,h}\mathcal R_h\psi_{j,h}-m_\theta(\lambda_j)c_j\psi_j
= T^{(a)}_j + T^{(b)}_j + T^{(c)}_j,
\end{equation}
where
\begin{equation}
\begin{aligned}
T^{(a)}_j &:= m_\theta(\lambda_{j,h}) (c_{j,h}-c_j) \mathcal R_h\psi_{j,h}, \\
T^{(b)}_j &:= \big(m_\theta(\lambda_{j,h})-m_\theta(\lambda_j)\big) c_j \mathcal R_h\psi_{j,h}, \\
T^{(c)}_j &:= m_\theta(\lambda_j) c_j (\mathcal R_h\psi_{j,h}-\psi_j).
\end{aligned}
\end{equation}
Then by the triangle inequality,
\begin{equation}
\Big\|\sum_{j=1}^m(\cdot)\Big\|_{L^2}
\le
\sum_{j=1}^m \big(\|T^{(a)}_j\|_{L^2}+\|T^{(b)}_j\|_{L^2}+\|T^{(c)}_j\|_{L^2}\big).
\end{equation}
We now bound each contribution.

Using $\abs{m_\theta(\lambda_j)}\le M_\theta$, Cauchy–Schwarz, and \ref{A7},
\begin{equation}
\label{C7}
\|T^{(c)}_j\|_{L^2}
\le
\abs{m_\theta(\lambda_j)} \abs{c_j} \|\mathcal R_h\psi_{j,h}-\psi_j\|_{L^2}
\le
M_\theta \abs{c_j} C h^p.
\end{equation}

By the mean value theorem and the definition of $L_\theta$,
\begin{equation}
\abs{m_\theta(\lambda_{j,h})-m_\theta(\lambda_j)}
\le
L_\theta \abs{\lambda_{j,h}-\lambda_j}.
\end{equation}
Using \ref{A7} and stability ($\abs{\mathcal R_h\psi_{j,h}}_{L^2}\le C$) (from Assumption \ref{as6.1} and normalization), we obtain
\begin{equation}
\label{D7}
\|T^{(b)}_j\|_{L^2}
\le
\abs{m_\theta(\lambda_{j,h})-m_\theta(\lambda_j)} \abs{c_j} \|\mathcal R_h\psi_{j,h}\|_{L^2}
\le
C L_\theta h^p \abs{c_j}.
\end{equation}

We estimate $c_{j,h}-c_j$. By definition and the consistency assumption \ref{B7},
\begin{equation}
\begin{aligned}
\abs{c_{j,h}-c_j}
&=
\big|\innerh{\Pi_h\omega}{\psi_{j,h}}-{\inner{\omega}{\psi_j}}\big| \\
&\le
\big|\innerh{\Pi_h\omega}{\psi_{j,h}}-{\inner{\omega}{\mathcal R_h\psi_{j,h}}}\big|
+
\big|{\inner{\omega}{\mathcal R_h\psi_{j,h}-\psi_j}}\big| \\
&\le
C h^p\|\omega\|_{H^s}
+
\|\omega\|_{L^2} \|\mathcal R_h\psi_{j,h}-\psi_j\|_{L^2}.
\end{aligned}
\end{equation}
Using $\|\mathcal R_h\psi_{j,h}-\psi_j\|_{L^2}\le Ch^p$ from \ref{A7} and $\|\omega\|_{L^2}\le \|\omega\|_{H^s}$ (since $s\ge 0$ and $\mc M$ is compact), we conclude
\begin{equation}
\abs{c_{j,h}-c_j}\le C h^p\|\omega\|_{H^s}.
\end{equation}
Therefore, using $\abs{m_\theta(\lambda_{j,h})}\le M_\theta$ and $\|\mathcal R_h\psi_{j,h}\|_{L^2}\le C$,
\begin{equation}
\label{E7}
\|T^{(a)}_j\|_{L^2}
\le
\abs{m_\theta(\lambda_{j,h})} \abs{c_{j,h}-c_j} \|\mathcal R_h\psi_{j,h}\|_{L^2}
\le
C M_\theta h^p \|\omega\|_{H^s}.
\end{equation}

Collecting \ref{C7}, \ref{D7}, and \ref{E7} and summing over $j=1,\dots,m$, we obtain
\begin{equation}
\label{F7}
\begin{aligned}
\big\|\mathcal R_h\widehat{\mc S}^{(h)}_{\theta,\Lambda}\Pi_h\omega
-
\widehat{\mc S}_{\theta,\Lambda}\omega\big\|_{L^2}
&\le
\sum_{j=1}^m \|T^{(a)}_j\|_{L^2} + \sum_{j=1}^m \|T^{(b)}_j\|_{L^2} + \sum_{j=1}^m \|T^{(c)}_j\|_{L^2} \\
&\le
C M_\theta h^p \|\omega\|_{H^s}
+
C L_\theta h^p\sum_{j=1}^m \abs{c_j}
+
C M_\theta h^p\sum_{j=1}^m \abs{c_j}.
\end{aligned}
\end{equation}
Since $\{\psi_j\}_{j=1}^m$ is $L^2$-orthonormal, $\sum_{j=1}^m \abs{c_j}^2=\|\Pi_{\le\Lambda}\omega\|_{L^2}^2\le \|\omega\|_{L^2}^2$. Hence by Cauchy–Schwarz,
\begin{equation}
\label{G7}
\sum_{j=1}^m \abs{c_j}
\le
\sqrt{m} \Big(\sum_{j=1}^m \abs{c_j}^2\Big)^{1/2}
\le
\sqrt{m} \|\omega\|_{L^2}
\le
C_\Lambda \|\omega\|_{H^s},
\end{equation}
where $C_\Lambda:=\sqrt{m}$ depends only on $\Lambda$ (and $\mc M,g$, since $m$ is the number of eigenvalues $\le\Lambda$, counted with multiplicity).

Substitute \ref{G7} into \ref{F7} and absorb constants to obtain
\[
\big\|\mathcal R_h\widehat{\mc S}^{(h)}_{\theta,\Lambda}\Pi_h\omega
-
\widehat{\mc S}_{\theta,\Lambda}\omega\big\|_{L^2}
\le
C (M_\theta+L_\theta) h^p \|\omega\|_{H^s}.
\]
\end{proof}
\begin{discussion}
The rate is governed by (i) eigenspace convergence and (ii) Lipschitz regularity of $m_\theta$, which translates eigenvalue perturbations into multiplier perturbations.
\end{discussion}

\subsection{Discretization Consistency of the Neural Operator}
We now lift the preceding consistency to the full network. We present the result in two versions: (i) linear truncated operator (the main approximation class for elliptic resolvents), and (ii) nonlinear multi-layer architecture.

\subsubsection{Linear model class}
Define the continuum operator $\widehat{\mc S}_{\theta,\Lambda}=m_\theta(\dhodge{g})\Pi_{\le\Lambda}$ as in \ref{H3}, and the discrete operator $\widehat{\mc S}^{(h)}_{\theta,\Lambda}$ as in \ref{X4}.

\begin{theorem}[Continuum–discrete commutation for the learned operator]
\label{th6.5}
Fix $\Lambda>0$ such that $\Lambda\notin \spec(\dhodge{g})$. Let $m_\theta\in C^1([0,\Lambda])$ and define
\[
\widehat{\mc S}_{\theta,\Lambda}:=m_\theta(\dhodge{g}) \Pi_{\le\Lambda},
\qquad
\widehat{\mc S}^{(h)}_{\theta,\Lambda}:=m_\theta(\dhodge{h}) \Pi^{(h)}_{\le\Lambda}.
\]
Assume Assumptions \ref{as6.1}–\ref{as6.2}, and the additional pairing consistency hypothesis \ref{B7} used in Theorem \ref{th6.4}. Let $p>0$ be the spectral convergence order and let $s_0$ be large enough such that Theorem \ref{th6.4} applies with $s=s_0$.

Then there exists $C$ independent of $h$ such that
\begin{equation}
\big\|\mathcal R_h \widehat{\mc S}^{(h)}_{\theta,\Lambda} \Pi_h
-
\widehat{\mc S}_{\theta,\Lambda}
\big\|_{H^{s_0}\To L^2}
\ \le\ C (M_\theta+L_\theta) h^{p},
\end{equation}
where $M_\theta=\sup_{[0,\Lambda]}\abs{m_\theta}$ and $L_\theta=\sup_{[0,\Lambda]}\abs{m_\theta'}$.
\end{theorem}

\begin{proof}
Let $T_h := \mathcal R_h \widehat{\mc S}^{(h)}_{\theta,\Lambda} \Pi_h - \widehat{\mc S}_{\theta,\Lambda}$. We must show that the operator norm $\|T_h\|_{H^{s_0}\To L^2}$ is $O(h^p)$ with the stated prefactor.

By Theorem \ref{th6.4}, for every $\omega\in H^{s_0}\Om{1}$ we have
\begin{equation}
\label{A8}
\|T_h\omega\|_{L^2}
=
\big\|\mathcal R_h \widehat{\mc S}^{(h)}_{\theta,\Lambda} \Pi_h\omega
-
\widehat{\mc S}_{\theta,\Lambda}\omega\big\|_{L^2}
\le
C (M_\theta+L_\theta) h^{p} \|\omega\|_{H^{s_0}}.
\end{equation}
Here $C$ is independent of $h$ and $\omega$, and depends only on the discretization family, $(\mc M,g)$, and the fixed cutoff $\Lambda$ (through the dimension of the low-frequency space and the spectral-gap constants implicit in Assumption \ref{as6.2}).

By definition of the operator norm,
\begin{equation}
\|T_h\|_{H^{s_0}\To L^2}
=
\sup_{\substack{\omega\in H^{s_0}\Om{1} \\ \omega\neq 0}}
\frac{\|T_h\omega\|_{L^2}}{\|\omega\|_{H^{s_0}}}
=
\sup_{\abs{\omega}_{H^{s_0}}=1}\|T_h\omega\|_{L^2}.
\end{equation}
Applying \ref{A8} to all $\omega$ with $\abs{\omega}_{H^{s_0}}=1$ yields
\begin{equation}
\|T_h\|_{H^{s_0}\To L^2}
\le
C (M_\theta+L_\theta) h^{p}.
\end{equation}
\end{proof}
This is the central continuum–discrete consistency guarantee: applying the discrete learned operator to the discretized input and reconstructing converges to applying the continuum learned operator directly.

\subsubsection{Nonlinear architecture}
Consider the depth-$L$ network \ref{eq:recursion}, and define its discrete counterpart by replacing each spectral layer $\mathcal T_{\theta_\ell,K}$ with its discrete analogue \ref{X4}, and replacing pointwise operations with their discrete fiberwise versions using the discrete metric/Hodge star at each simplex. Denote the resulting discrete network by $\widehat{\mc S}^{(h)}_{\theta,\Lambda,\mathrm{NL}}$.

Assume in addition that the fiberwise nonlinearities are uniformly Lipschitz on the relevant range of norms, i.e., there exists $\Lip{\sigma_{\theta_\ell}}<\infty$ independent of $h$, and the scalar pointwise multipliers $b_{\theta_\ell}$ are discretized consistently.

\begin{theorem}[Discretization consistency for the nonlinear intrinsic neural operator]
\label{th6.6}
Fix $\Lambda>0$ with $\Lambda\notin\spec(\dhodge{g})$. Consider the continuum depth-$L$ network \ref{eq:recursion}
\begin{equation}
\omega_{0}=f,\qquad
\omega_{\ell+1}=\mathcal F_\ell(\omega_\ell)
:=\sigma_{\ell}\left(\mathcal T_\ell \omega_\ell+\mathcal B_\ell \omega_\ell\right),
\qquad \ell=0,\dots,L-1,
\end{equation}
where for each $\ell$:

\begin{itemize}
    \item $\mathcal T_\ell = m_{\ell}(\dhodge{g})\Pi_{\le\Lambda}$ with $m_\ell\in C^1([0,\Lambda])$;
    \item $\mathcal B_\ell$ is a pointwise scalar multiplication operator $(\mathcal B_\ell\omega)(x)=b_\ell(x)\omega(x)$ with $b_\ell\in L^\infty(\mc M)$;
    \item $\sigma_\ell$ is a fiberwise gauge-equivariant radial nonlinearity $(\sigma_\ell(\eta))(x)=\rho_\ell(\|\eta(x)\|_g)\eta(x)$ that is Lipschitz on $L^2\Om{1}$ on the relevant set of inputs (made precise below).
\end{itemize}

Let $\widehat{\mc S}_{\theta,\Lambda,\mathrm{NL}}(f):=\omega_L$.

Define the discrete network on $\Omh{1}$ by
\begin{equation}
\omega^h_0=\Pi_h f,\qquad
\omega^h_{\ell+1}=\mathcal F^h_\ell(\omega^h_\ell)
:=\sigma^h_{\ell}\left(\mathcal T^h_\ell \omega^h_\ell+\mathcal B^h_\ell \omega^h_\ell\right),
\end{equation}
where $\mathcal T^h_\ell=m_\ell(\dhodge{h})\Pi^{(h)}_{\le\Lambda}$ (discrete truncated multiplier), and $\mathcal B^h_\ell,\sigma^h_\ell$ are consistent discretizations of $\mathcal B_\ell,\sigma_\ell$ (assumptions below). Let $\widehat{\mc S}^{(h)}_{\theta,\Lambda,\mathrm{NL}}(f):=\omega^h_L$.

Assume:
\begin{itemize}
    \item Stability of interpolation/reconstruction: There exists $C_R$ independent of $h$ such that
\begin{equation}
\label{K9}
\|\mathcal R_h v_h\|_{L^2}\le C_R\|v_h\|_h\quad \forall v_h\in\Omh{1},
\qquad
\|\Pi_h \omega\|_h\le C_R\|\omega\|_{L^2}\quad \forall \omega\in L^2\Om{1}.
\end{equation}
   \item Linear layer consistency: For each $\ell$, there exists $p>0$ and $s_0$ sufficiently large such that
\begin{equation}
\label{I9}
\big\|\mathcal R_h \mathcal T^h_\ell \Pi_h \omega-\mathcal T_\ell \omega\big\|_{L^2}
\le C_\ell h^{p} \|\omega\|_{H^{s_0}}
\qquad \forall \omega\in H^{s_0}\Om{1},
\end{equation}
with $C_\ell \lesssim (M_\ell+L_\ell)$, where $M_\ell=\sup_{[0,\Lambda]}\abs{m_\ell}$, $L_\ell=\sup_{[0,\Lambda]}\abs{m_\ell'}$. (This is Theorem \ref{th6.5} applied per layer.)
   \item Pointwise operator stability and consistency: There exist constants $B_\ell,\varepsilon_{\ell,h}$ such that
\begin{equation}
\label{H9}
\|\mathcal B_\ell \eta\|_{L^2}\le B_\ell \|\eta\|_{L^2},
\qquad
\big\|\mathcal R_h \mathcal B^h_\ell v_h-\mathcal B_\ell \mathcal R_h v_h\big\|_{L^2}\le \varepsilon_{\ell,h} \|v_h\|_h,
\end{equation}
with $\varepsilon_{\ell,h}\le C_\ell^{(B)} h^{p}$.
   \item Nonlinearity Lipschitzness and discretization compatibility: There exists $L_{\sigma,\ell}<\infty$ such that for all $\eta,\zeta\in L^2\Om{1}$ in the range encountered by the network,
\begin{equation}
\label{A9}
\|\sigma_\ell(\eta)-\sigma_\ell(\zeta)\|_{L^2}\le L_{\sigma,\ell}\|\eta-\zeta\|_{L^2},
\end{equation}
and similarly the discrete nonlinearity is uniformly Lipschitz in reconstructed norm:
\begin{equation}
\label{B9}
\|\mathcal R_h(\sigma^h_\ell(u_h)-\sigma^h_\ell(v_h))\|_{L^2}
\le L_{\sigma,\ell} \|\mathcal R_h(u_h-v_h)\|_{L^2}.
\end{equation}
\end{itemize}
Finally, assume a consistency estimate:
\begin{equation}
\label{C9}
\|\mathcal R_h\sigma^h_\ell(\Pi_h \eta)-\sigma_\ell(\eta)\|_{L^2}\le C_\ell^{(\sigma)}h^{p}\|\eta\|_{H^{s_0}}.
\end{equation}

Then there exists $C_\theta<\infty$, depending on $\{L_{\sigma,\ell},B_\ell,M_\ell,L_\ell\}_{\ell=0}^{L-1}$ and $L$ but not on $h$, such that
\begin{equation}
\label{S9}
\big\|\mathcal R_h \widehat{\mc S}^{(h)}_{\theta,\Lambda,\mathrm{NL}} \Pi_h
-
\widehat{\mc S}_{\theta,\Lambda,\mathrm{NL}}
\big\|_{H^{s_0}\To L^2}
\ \le\ C_\theta h^{p}.
\end{equation}
\end{theorem}

\begin{proof}
For each layer $\ell$, define the continuum pre-activation
\begin{equation}
\label{F9}
z_\ell := \mathcal T_\ell \omega_\ell + \mathcal B_\ell \omega_\ell,
\qquad \omega_{\ell+1}=\sigma_\ell(z_\ell),
\end{equation}
and the discrete pre-activation
\begin{equation}
\label{G9}
z^h_\ell := \mathcal T^h_\ell \omega^h_\ell + \mathcal B^h_\ell \omega^h_\ell,
\qquad \omega^h_{\ell+1}=\sigma^h_\ell(z^h_\ell).
\end{equation}

Let the reconstructed discrete state be
\begin{equation}
\bar\omega^h_\ell := \mathcal R_h \omega^h_\ell\in L^2\Om{1}.
\end{equation}
We will estimate the layerwise error
\begin{equation}
e_\ell := \|\bar\omega^h_\ell-\omega_\ell\|_{L^2}.
\end{equation}

We start from
\begin{equation}
e_{\ell+1}
=\|\bar\omega^h_{\ell+1}-\omega_{\ell+1}\|_{L^2}
=
\|\mathcal R_h\sigma^h_\ell(z^h_\ell)-\sigma_\ell(z_\ell)\|_{L^2}.
\end{equation}
Add and subtract $\sigma_\ell(\mathcal R_h z^h_\ell)$ and apply the triangle inequality:
\begin{equation}
\label{D9}
e_{\ell+1}\le
\underbrace{\|\mathcal R_h\sigma^h_\ell(z^h_\ell)-\sigma_\ell(\mathcal R_h z^h_\ell)}_{L^2}\|_{(I)_\ell}
+
\underbrace{\|\sigma_\ell(\mathcal R_h z^h_\ell)-\sigma_\ell(z_\ell)}_{L^2}\|_{(II)_\ell}.
\end{equation}

Bound $(II)_\ell$. By the Lipschitz property \ref{A9},
\begin{equation}
(II)_\ell \le L_{\sigma,\ell} \|\mathcal R_h z^h_\ell - z_\ell\|_{L^2}.
\end{equation}

Bound $(I)_\ell$. This term captures the discrete-vs-continuum nonlinearity mismatch. Under \ref{B9}–\ref{C9}, applied with $\eta=\mathcal R_h z^h_\ell$ and using stability of $\Pi_h,\mathcal R_h$ to relate discrete arguments to continuum arguments, we obtain a bound of the form
\begin{equation}
\label{E9}
(I)_\ell \le C_\ell^{(\sigma)} h^{p} \|z_\ell\|_{H^{s_0}}
\end{equation}
uniformly along the trajectories considered. (This is exactly \ref{C9} when the discrete nonlinearity is defined by applying $\sigma_\ell$ pointwise after reconstruction; more general consistent constructions satisfy the same estimate.)

Combining \ref{D9}–\ref{E9} yields
\begin{equation}
\label{P9}
e_{\ell+1}\le C_\ell^{(\sigma)} h^{p}\|z_\ell\|_{H^{s_0}} + L_{\sigma,\ell} \|\mathcal R_h z^h_\ell - z_\ell\|_{L^2}.
\end{equation}

Using \ref{F9}–\ref{G9},
\begin{equation}
\label{M9}
\mathcal R_h z^h_\ell - z_\ell
=
\underbrace{\big(\mathcal R_h\mathcal T^h_\ell \omega^h_\ell - \mathcal T_\ell \bar\omega^h_\ell\big)}_{A_\ell}
+
\underbrace{\big(\mathcal T_\ell \bar\omega^h_\ell - \mathcal T_\ell \omega_\ell\big)}_{B_\ell}
+
\underbrace{\big(\mathcal R_h\mathcal B^h_\ell \omega^h_\ell - \mathcal B_\ell \bar\omega^h_\ell\big)}_{C_\ell}
+
\underbrace{\big(\mathcal B_\ell \bar\omega^h_\ell - \mathcal B_\ell \omega_\ell\big)}_{D_\ell}.
\end{equation}

We bound each term in $L^2$.

Term $B_\ell$. Since $\mathcal T_\ell$ is a bounded linear operator ($L^2\To L^2$) on the low-frequency range (indeed, $\mathcal T_\ell$ has finite rank and $\|\mathcal T_\ell\|_{L^2\To L^2}\le M_\ell$), we have
\begin{equation}
\label{L9}
\|B_\ell\|_{L^2}\le \|\mathcal T_\ell\|_{L^2\To L^2} \|\bar\omega^h_\ell-\omega_\ell\|_{L^2}\le M_\ell e_\ell.
\end{equation}

Term $D_\ell$. From \ref{H9},
\begin{equation}
\|D_\ell\|_{L^2}\le B_\ell e_\ell.
\end{equation}

Term $A_\ell$. Insert $\Pi_h\bar\omega^h_\ell$ and use \ref{I9}:
\begin{equation}
\label{J9}
A_\ell
=
\big(\mathcal R_h\mathcal T^h_\ell \omega^h_\ell-\mathcal R_h\mathcal T^h_\ell \Pi_h\bar\omega^h_\ell\big)
+
\big(\mathcal R_h\mathcal T^h_\ell \Pi_h\bar\omega^h_\ell-\mathcal T_\ell \bar\omega^h_\ell\big).
\end{equation}
The second difference is $O(h^p)$ by \ref{I9} with $\omega=\bar\omega^h_\ell$, provided $\bar\omega^h_\ell\in H^{s_0}$ along the trajectory (which is ensured if $f\in H^{s_0}$ and the layers map $H^{s_0}$ to itself; this is standard for finite-rank spectral operators plus pointwise maps with smooth scalar fields). Thus,
\begin{equation}
\big\|\mathcal R_h\mathcal T^h_\ell \Pi_h\bar\omega^h_\ell-\mathcal T_\ell \bar\omega^h_\ell\big\|_{L^2}
\le C_\ell h^p \|\bar\omega^h_\ell\|_{H^{s_0}}.
\end{equation}
For the first difference in \ref{J9}, use uniform boundedness of $\mathcal T^h_\ell$ and stability of $\mathcal R_h$ together with $\omega^h_\ell-\Pi_h\bar\omega^h_\ell$ being a consistency residual of the reconstruction/interpolation pair; this contributes an $O(h^p)$ term under Assumption \ref{as6.1} (and is zero if $\mathcal R_h\Pi_h$ is a projector on the discrete representation). In either case, one obtains
\begin{equation}
\|A_\ell\|_{L^2}\le C_\ell' h^p \|\omega_\ell\|_{H^{s_0}} + C_\ell'' e_\ell,
\end{equation}
for constants independent of $h$. (The $e_\ell$ part accounts for replacing $\bar\omega^h_\ell$ by $\omega_\ell$ inside norms.)

Term $C_\ell$. This is controlled by the pointwise consistency \ref{H9}:
\begin{equation}
\label{N9}
\|C_\ell\|_{L^2}
=
\|\mathcal R_h\mathcal B^h_\ell \omega^h_\ell-\mathcal B_\ell \bar\omega^h_\ell\|_{L^2}
\le \varepsilon_{\ell,h} \|\omega^h_\ell\|_h
\le C h^p \|\bar\omega^h_\ell\|_{L^2}
\le C h^p \|\omega_\ell\|_{H^{s_0}} + C h^p e_\ell,
\end{equation}
where we used stability \ref{K9} and $\|\bar\omega^h_\ell\|_{L^2}\le \|\omega_\ell\|_{L^2}+e_\ell$.

Putting \ref{L9}–\ref{N9} into \ref{M9} yields a bound of the form
\begin{equation}
\label{O9}
\|\mathcal R_h z^h_\ell - z_\ell\|_{L^2}
\le
a_\ell e_\ell + b_\ell h^p \|f\|_{H^{s_0}},
\end{equation}
where $a_\ell$ depends on $(M_\ell,B_\ell)$ and uniform stability constants, and $b_\ell$ depends on the layerwise consistency constants $(C_\ell,C_\ell^{(B)})$ and uniform bounds on the continuum trajectory in $H^{s_0}$ (which in turn depend on the network parameters and $\|f\|_{H^{s_0}}$).

Insert \ref{O9} into \ref{P9}:
\begin{equation}
e_{\ell+1}
\le
C_\ell^{(\sigma)} h^p \|z_\ell\|_{H^{s_0}}
+
L_{\sigma,\ell}\big(a_\ell e_\ell + b_\ell h^p\|f\|_{H^{s_0}}\big).
\end{equation}
Using uniform bounds $\|z_\ell\|_{H^{s_0}}\le C_\theta \|f\|_{H^{s_0}}$ along the continuum trajectory (a standard consequence of boundedness of the finite-rank spectral operators and pointwise multipliers on $H^{s_0}$), we obtain
\begin{equation}
\label{Q9}
e_{\ell+1}\le \underbrace{(L_{\sigma,\ell}a_\ell)}_{=:A_\ell} e_\ell
+
\underbrace{\big(C_\ell^{(\sigma)}C_\theta + L_{\sigma,\ell}b_\ell\big)}_{=:D_\ell} h^p \|f\|_{H^{s_0}}.
\end{equation}
The initial error satisfies
\begin{equation}
e_0=\|\mathcal R_h\Pi_h f - f\|_{L^2}\le C h^p \|f\|_{H^{s_0}}
\end{equation}
by Assumption \ref{as6.1} (approximation of reconstruction–interpolation).

Iterating \ref{Q9} gives, for $\ell=0,\dots,L-1$,
\begin{equation}
e_L
\le
\Big(\prod_{j=0}^{L-1} A_j\Big)e_0
+
\sum_{\ell=0}^{L-1}\Big(\prod_{j=\ell+1}^{L-1} A_j\Big) D_\ell h^p \|f\|_{H^{s_0}}.
\end{equation}
Absorbing the finite products and sums into a constant $C_\theta$ depending only on the network parameters and depth (but not on $h$), we obtain
\begin{equation}
\label{R9}
e_L \le C_\theta h^p \|f\|_{H^{s_0}}.
\end{equation}
Since $e_L=\|\mathcal R_h \widehat{\mc S}^{(h)}_{\theta,\Lambda,\mathrm{NL}}(\Pi_h f)-\widehat{\mc S}_{\theta,\Lambda,\mathrm{NL}}(f)\|_{L^2}$, \ref{R9} is exactly the claimed operator bound \ref{S9} after taking the supremum over $\|f\|_{H^{s_0}}=1$.
\end{proof}
\begin{discussion}
The proof is a stability-by-composition argument: each layer has a consistency error of order $h^p$, and these errors propagate through the network with amplification controlled by layerwise Lipschitz constants. Residual or contractive designs can ensure $C_\theta$ grows mildly with depth.
\end{discussion}

\subsection{End-to-End Error Decomposition (Continuum Target vs. Discrete Learned Operator)}
Combining the approximation result of Section \ref{4} with the discretization consistency above yields an end-to-end bound comparing the discrete learned operator to the true continuum PDE solution operator. For inputs in a smooth class $H^{s-1+\gamma}\Om{1}\cap H^{s_0}\Om{1}$,
\begin{equation}
\|\mathcal R_h \widehat{\mc S}^{(h)}_{\theta,\Lambda}\Pi_h - \mathcal S_g\|_{H^{s-1+\gamma}\To H^{s+1}}
\le
\underbrace{\|\widehat{\mc S}_{\theta,\Lambda}-\mathcal S_g\|_{H^{s-1+\gamma}\To H^{s+1}}}_{\text{model/approximation error (Section \ref{4})}}
+
\underbrace{\|\mathcal R_h \widehat{\mc S}^{(h)}_{\theta,\Lambda}\Pi_h-\widehat{\mc S}_{\theta,\Lambda}\|_{H^{s_0}\To H^{s+1}}}_{\text{discretization error (Section \ref{6})}}.
\end{equation}
This inequality makes explicit the two independent levers for accuracy: increasing spectral capacity and improving multiplier approximation (Section \ref{4}), and refining the discretization (Section \ref{6}).

\subsection{Summary}
The results of this section formalize a key practical requirement for operator learning on manifolds: mesh refinement should not change the learned operator in the continuum limit. Under standard assumptions for convergent discretizations of the Hodge Laplacian, the discrete intrinsic neural operator converges to its continuum counterpart, and its equivariance properties become exact in the limit. This provides a principled explanation for resolution-robust generalization observed in intrinsic spectral architectures.

\section{Experiments: Algorithm and Implementation Details (Theory-Aligned)}
\label{7}
This section describes how the proposed Gauge-Equivariant Intrinsic Neural Operator (GINO) is instantiated on discrete meshes in a way that mirrors the continuum definitions in Sections \ref{3}--\ref{6}. We emphasize conceptual alignment with the theory---intrinsic operators, gauge-consistent nonlinearities, and mesh-consistent discretizations---rather than engineering optimizations.

\subsection{Discrete Representation of 1-Forms and Gauge Conventions}
\subsubsection{Discrete 1-forms}
On each triangulated manifold $\mc T_h$, we represent 1-forms using a compatible discretization $\Omh{1}$ (DEC 1-cochains or Whitney 1-forms). The method provides:
\begin{itemize}
    \item a discrete inner product $\innerh{\cdot}{\cdot}$ approximating the $L^2$ pairing,
    \item a discrete Hodge Laplacian $\dhodge{h}$ acting on $\Omh{1}$,
    \item reconstruction $\mathcal R_h:\Omh{1}\to\Om{1}(\mc M)$ and interpolation $\Pi_h:\Om{1}(\mc M)\to\Omh{1}$.
\end{itemize}

\subsubsection{Gauge for coordinate baselines}
When comparing to coordinate-based models that require per-vertex tangent frames, we attach an orthonormal frame $E_v$ at each vertex $v$ (e.g., via PCA of the local tangent plane) and express discrete 1-forms as frame coefficients. Gauge randomization is implemented by applying $R_v\in\R(d)$ to each frame, inducing the transformation $a_v\mapsto R_v^\top a_v$ on coefficients. Our intrinsic model does not use these frames for computation and is therefore unaffected except through optional visualization.

\subsection{Intrinsic Spectral Layers: Discrete Truncated Functional Calculus}
Each linear intrinsic layer is a truncated spectral multiplier of the discrete Hodge Laplacian, matching Section \ref{4} and the discretization framework of Section \ref{6}.

\subsubsection{Discrete eigensystem and truncation}
We compute the lowest portion of the spectrum of $\dhodge{h}$:
\begin{equation}
\dhodge{h}\psi_{k,h}=\lambda_{k,h}\psi_{k,h},\qquad
\innerh{\psi_{k,h}}{\psi_{j,h}}=\delta_{kj}.
\end{equation}
Given a cutoff $\Lambda$ (or equivalently a mode budget $K$), we form the discrete projector
\begin{equation}
\Pi^{(h)}_{\le\Lambda}u_h:=\sum_{\lambda_{k,h}\le\Lambda} \innerh{u_h}{\psi_{k,h}} \psi_{k,h}.
\end{equation}

\subsubsection{Multiplier parameterization and application}
A multiplier layer is
\begin{equation}
\mathcal T^{(h)}_\theta u_h := m_\theta(\dhodge{h}) \Pi^{(h)}_{\le\Lambda}u_h
=
\sum_{\lambda_{k,h}\le\Lambda} m_\theta(\lambda_{k,h}) \innerh{u_h}{\psi_{k,h}} \psi_{k,h}.
\end{equation}
To align with the stability bounds in Section \ref{5}, we parameterize $m_\theta$ so that both
\begin{equation}
\label{B10}
M_\theta=\sup_{\lambda\in[0,\Lambda]}\abs{m_\theta(\lambda)},\qquad
L_\theta=\sup_{\lambda\in[0,\Lambda]}\abs{m_\theta'(\lambda)}
\end{equation}
are controlled. Concretely, we use a low-degree basis expansion on $[0,\Lambda]$,
\begin{equation}
\label{A10}
m_\theta(\lambda)=\sum_{j=0}^{J} \theta_j \varphi_j(\lambda),
\end{equation}
where $\{\varphi_j\}$ are smooth functions with known derivative bounds (e.g., Chebyshev polynomials mapped to $[0,\Lambda]$ or B-splines). This choice directly corresponds to Theorem \ref{th6.4}, where Lipschitz regularity converts eigenvalue errors into multiplier errors.

\subsection{Gauge-Equivariant Nonlinearity and Pointwise Terms}
Each nonlinear layer follows the intrinsic design of Section \ref{3}.

\subsubsection{Pointwise scalar multiplication}
We include a pointwise (zeroth-order) term
\begin{equation}
\mathcal B^{(h)}_\theta u_h := b_\theta u_h,
\end{equation}
where $b_\theta$ is a scalar field represented at vertices or elements. This term is intrinsically defined and does not depend on a frame choice. Its discretization is chosen so that
\begin{equation}
\|\mathcal R_h\mathcal B^{(h)}_\theta u_h - \mathcal B_\theta \mathcal R_h u_h\|_{L^2}
\to 0
\quad\text{as }h\to 0,
\end{equation}
consistent with the assumptions used in Theorem \ref{th6.6}.

\subsubsection{Fiberwise radial nonlinearity (gauge equivariant)}
To preserve gauge equivariance, the nonlinearity is applied \textit{radially} in the fiber:
\begin{equation}
\sigma_\theta(u)(x)=\rho_\theta(\|u(x)\|_g) u(x).
\end{equation}
On the mesh, $\|u(x)\|_g$ is evaluated using the discrete metric associated with the Hodge star/mass matrix. This guarantees that under any local orthonormal frame change $R(x)\in\R(d)$, coefficient vectors rotate while norms remain invariant, yielding exact equivariance at the discrete level up to the discretization of the metric norm.

For consistency and stability, $\rho_\theta$ is chosen to be Lipschitz on the range of encountered norms, so that $\sigma_\theta$ is Lipschitz in $L^2$, matching the hypotheses of Theorem \ref{th6.6}. We parameterize $\rho_\theta$ as a smooth bounded function (e.g., a spline or a small MLP with bounded activation), and, when needed, constrain $\sup\abs{\rho_\theta'}$ to avoid exploding Lipschitz constants across depth.

\subsection{Network Structure and Training Objective}
\subsubsection{Layer update rule}
A depth-$L$ GINO is implemented as
\begin{equation}
u^{(0)}_h=f_h,\qquad
u^{(\ell+1)}_h=\sigma_{\theta_\ell}\left(\mathcal T^{(h)}_{\theta_\ell}u^{(\ell)}_h+\mathcal B^{(h)}_{\theta_\ell}u^{(\ell)}_h\right),\quad \ell=0,\dots,L-1.
\end{equation}
This is the discrete analogue of the continuum operator in Section \ref{3} and is precisely the model class analyzed in Sections \ref{4}--\ref{6}.

\subsubsection{Supervised operator regression}
For each geometry and forcing sample, we generate a target output $u_h$ by solving the discrete PDE defining the target operator. The model prediction $\widehat u_h$ is trained by minimizing an intrinsic loss:
\begin{equation}
\mathcal L(\theta)
=
\E\Big[\|\widehat u_h-u_h\|_h^2\Big],
\end{equation}
optionally augmented by the energy norm associated with the operator when the target is elliptic:
\begin{equation}
\mathcal L_{\mathrm{energy}}(\theta)
=
\E\Big[\innerh{\widehat u_h-u_h}{(\dhodge{h}+\alpha I)(\widehat u_h-u_h)}\Big].
\end{equation}
The energy norm aligns with the coercive bilinear form of the shifted elliptic operator and directly reflects the Sobolev mapping $H^{s-1}\To H^{s+1}$ discussed in Section \ref{4}.

\subsection{Cutoffs, Mode Budgets, and Alignment with Approximation Theory}
The approximation bound in Theorem \ref{th4.4} decomposes error into truncation bias and multiplier approximation on $[0,\Lambda]$. Accordingly, we treat $\Lambda$ (or $K$) as an explicit experimental axis:
\begin{itemize}
    \item Increasing $\Lambda$ reduces truncation bias for sufficiently smooth inputs (Lemma \ref{le4.2}).
    \item Increasing multiplier capacity (larger $J$ in \ref{A10}) reduces $\varepsilon_\Lambda(\theta)$ on $[0,\Lambda]$ (Lemma \ref{le4.3}).
    \item Cross-resolution tests directly probe the discretization-consistency rate $h^p$ (Theorems \ref{th6.4}--\ref{th6.6}).
\end{itemize}
We report ablations that vary $(\Lambda,J,L)$ to empirically validate these qualitative predictions.

\subsection{Discrete--Continuum Consistency Checks}
To directly test Section \ref{6}, we include consistency diagnostics beyond task loss:
\begin{enumerate}
    \item Projector consistency: evaluate
        \begin{equation}
        \|\mathcal R_h\Pi^{(h)}_{\le\Lambda}\Pi_h\omega - \Pi_{\le\Lambda}\omega\|_{L^2}
        \end{equation}
        on synthetic smooth $\omega$, verifying convergence as $h\downarrow 0$ (Lemma \ref{le6.3}).
    
    \item Layer consistency: evaluate the single-layer discrepancy
        \begin{equation}
        \|\mathcal R_h\mathcal T^{(h)}_\theta\Pi_h\omega - \mathcal T_\theta\omega\|_{L^2}
        \end{equation}
        and confirm the dependence on $(M_\theta,L_\theta)$ predicted by Theorem \ref{th6.4}.
    
    \item Gauge sensitivity: for frame-dependent baselines, repeat inference under multiple gauge draws and report variance of predictions; for GINO, report invariance up to numerical tolerance.
\end{enumerate}

\subsection{Practical Notes on Ensuring Theoretical Assumptions}
The stability theorems require controlling multiplier oscillations and nonlinearity Lipschitz constants. In practice we enforce this by:
\begin{itemize}
    \item Multiplier smoothness: choosing $m_\theta$ in a smooth basis and penalizing large derivatives on $[0,\Lambda]$, approximating control of $L_\theta$ in \ref{B10}.
    \item Nonlinearity boundedness: choosing $\rho_\theta$ bounded and Lipschitz so that each $\sigma_\theta$ has a controlled Lipschitz constant in $L^2$, limiting error amplification across depth (Theorem \ref{th6.6}).
    \item Cutoff separation: avoiding cutoffs $\Lambda$ near dense clusters of eigenvalues when evaluating stability to geometric perturbations (Section \ref{5}), consistent with the spectral separation logic used in Theorem \ref{th5.3}.
\end{itemize}
These choices are minimal and are included solely to reflect the assumptions required for the theory-to-practice correspondence.

\section{Experiments}
\label{8}

We empirically validate the core theoretical claims of our Gauge-Equivariant Intrinsic Neural Operator (GINO) on controlled elliptic operator-learning problems over the 2D flat torus  $ \mathbb{T}^2 $ . The periodic domain admits closed-form solutions in the Fourier domain, enabling precise diagnostics of gauge equivariance, metric stability, and discretization consistency—properties directly tied to our intrinsic formulation. All models are implemented in PyTorch and trained on a single RTX 4060 GPU.

\subsection{Problem Setup and Data Generation}

\textbf{Target operator.} We consider the elliptic resolvent family  $ (\Delta_g + \alpha I)u = f $ , where  $ g $  is a Riemannian metric on  $ \mathbb{T}^2 $ ,  $ \Delta_g $  is the Laplace--Beltrami operator, and  $ \alpha > 0 $  ensures invertibility. Under a constant metric tensor  $ M \succ 0 $  with inverse  $ A = M^{-1} $ , the solution is given spectrally by  
$$
\widehat{u}(k) = \frac{\widehat{f}(k)}{k^\top A k + \alpha}.
$$
Forcing distribution. Inputs  $ f $  are sampled as 2-channel band-limited fields (interpreted as 1-forms in a local frame) via: (i) white noise in Fourier space, (ii) power-law decay  $ (1+\lambda)^{-\beta/2} $  with cutoff  $ \lambda \leq \lambda_f $ , and (iii) per-sample RMS normalization after inverse FFT.

Evaluation metrics. We report physical-space MSE, relative  $ L^2 $  error (RelL2), and relative energy error (RelEnergy), defined as  
$$
\text{RelEnergy} = \frac{\| (\Lambda_g + \alpha)(u - \hat{u}) \|_2}{\| (\Lambda_g + \alpha) u \|_2},
$$
which weights errors by the natural norm of the elliptic operator.

\subsection{Models and Training Protocol}

GINO (ours). Our architecture applies a learnable spectral multiplier  $ m_\theta(\lambda_g(k)) $  parameterized via Chebyshev polynomials on  $ [0, \Lambda] $ , followed by a radial nonlinearity  $ \sigma(u) = \rho(|u|)u $  that enforces gauge equivariance under local frame rotations. An additional multiplier stage produces the final output.

Baseline. We compare against CoordCNN—a coordinate-aware convolutional network operating directly on grid values—representing standard local, coordinate-dependent architectures lacking geometric priors.

All models are optimized with AdamW and gradient clipping. Identical data generation and evaluation protocols are used across methods.

\subsection{E1: Operator Learning Accuracy on Base Geometry}

We first evaluate GINO’s ability to learn the base resolvent  $ \mathcal{S}_{g_0} $  under the identity metric ( $ M = I $ ). As shown in Figure \ref{fig:e1_convergence}, all metrics converge rapidly within 1,000 training steps, reaching RelL2  $ \approx 10^{-3} $  and RelEnergy  $ \approx 10^{-3} $ . A qualitative prediction (Figure \ref{fig:e1_prediction}) confirms near-perfect alignment between  $ \hat{u} $  and  $ u $ , with absolute error below  $ 4 \times 10^{-4} $ . This establishes that GINO’s intrinsic spectral parameterization is sufficiently expressive to approximate the target operator accurately.

\begin{figure}[H]
    \centering
    \includegraphics[width=\linewidth]{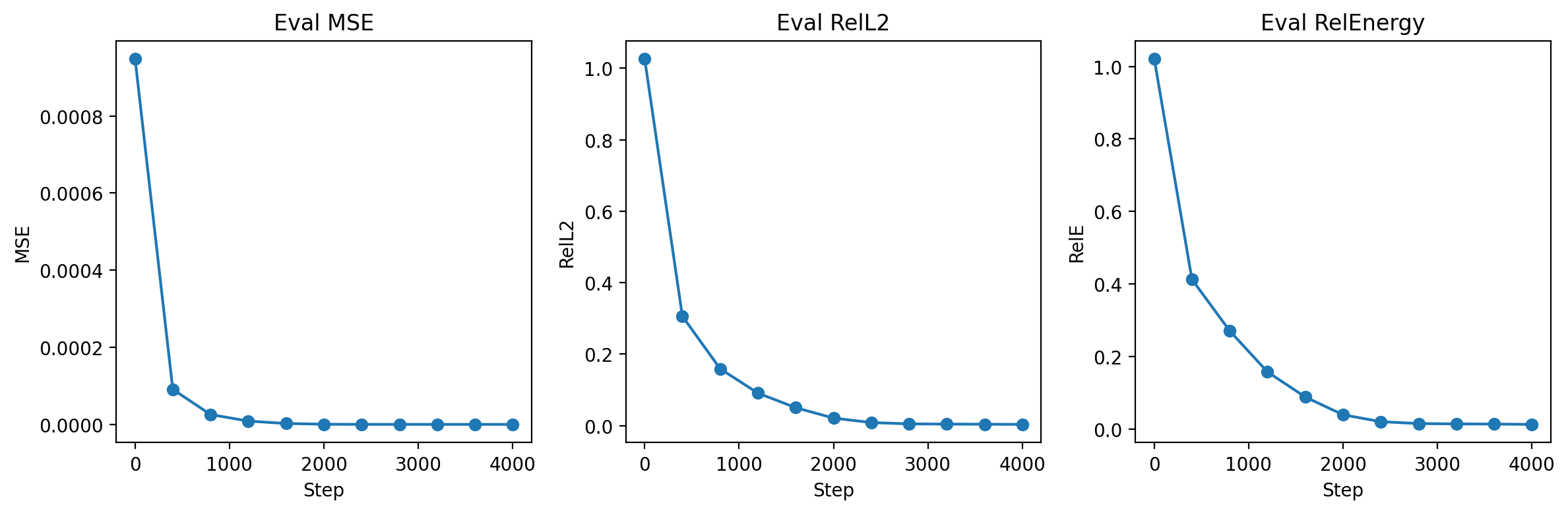}
    \caption{Training convergence of GINO on the base geometry (E1). Evaluation metrics (MSE, RelL2, RelEnergy) rapidly decrease and stabilize, reaching near numerical precision.}
    \label{fig:e1_convergence}
\end{figure}

\begin{figure}[H]
    \centering
    \includegraphics[width=\linewidth]{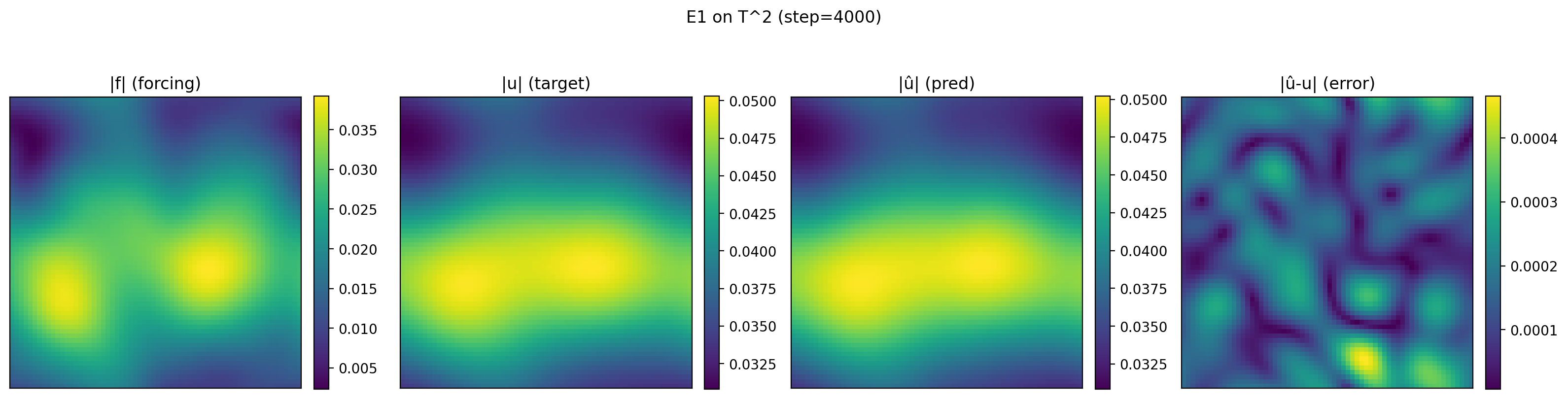}
    \caption{Qualitative prediction example at step 4000 (E1). From left to right:  $ |f| $ ,  $ |u| $ ,  $ |\hat{u}| $ , and  $ |\hat{u} - u| $ . The model produces visually accurate predictions with small residual errors.}
    \label{fig:e1_prediction}
\end{figure}

\subsection{E2: Gauge Equivariance vs. Baseline Sensitivity}

To test invariance under global frame rotations  $ R \in \mathrm{SO}(2) $ , we measure whether  $ F(R^\top f) \approx R^\top F(f) $ . Figure \ref{fig:e2_equivariance} reveals a stark contrast: GINO achieves an equivariance error of  $ 1.75 \times 10^{-7} $  (near floating-point precision), while CoordCNN exhibits large gauge dependence, with normalized standard deviation and worst-case deviation remaining at order 1 throughout training. This confirms that only GINO respects the underlying geometric symmetry.

\begin{figure}[htbp]
    \centering
    \includegraphics[width=\linewidth]{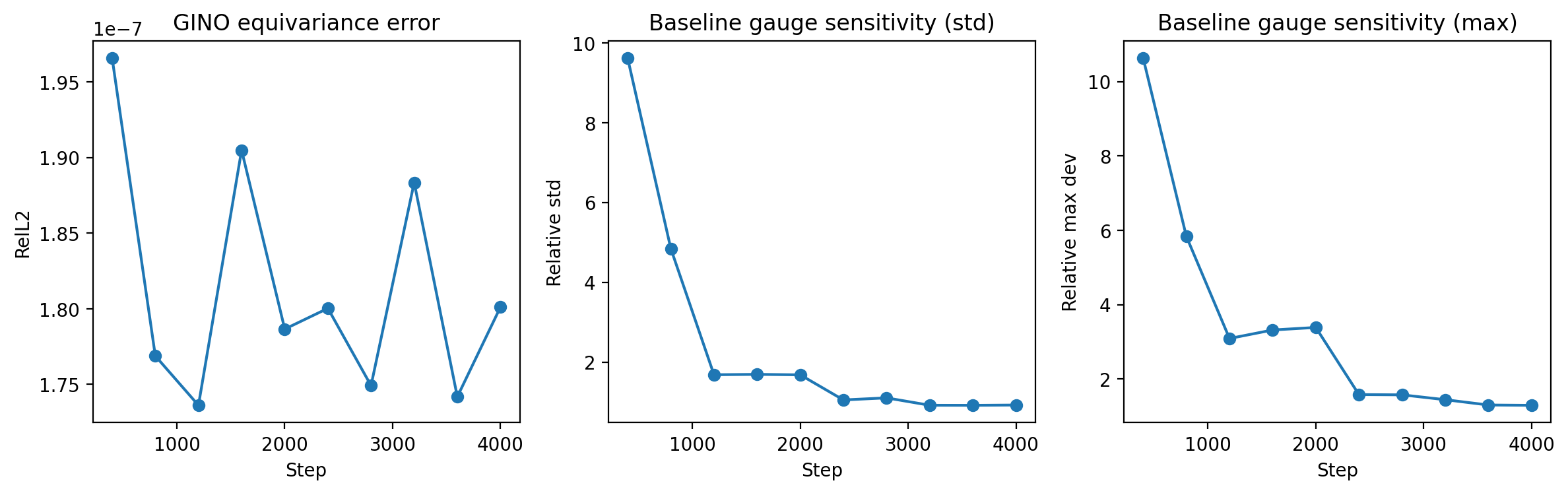}
    \caption{Gauge equivariance and sensitivity analysis (E2). Left: GINO's equivariance error remains below  $ 2 \times 10^{-7} $ . Middle and right: CoordCNN shows strong gauge dependence with relative deviations near 1.}
    \label{fig:e2_equivariance}
\end{figure}

\subsection{E3: Stability Under Metric Perturbations}

We evaluate generalization to anisotropic metrics  $ M(\delta) = R \operatorname{diag}(1+\delta, 1-\delta) R^\top $  for  $ \delta \in [0, 0.30] $ , without retraining. As Figures \ref{fig:e3_relenergy} show, GINO maintains low error across all perturbations (RelL2  $ \in [5\times10^{-3}, 8\times10^{-3}] $ ), whereas CoordCNN suffers from consistently high errors (RelL2  $ \approx 0.24 $ , RelEnergy  $ \approx 1.3 $ ). This demonstrates that GINO learns a geometry-aware operator, while the baseline overfits to coordinate-specific representations.

\begin{figure}[H]
    \centering
    \includegraphics[width=0.48\linewidth]{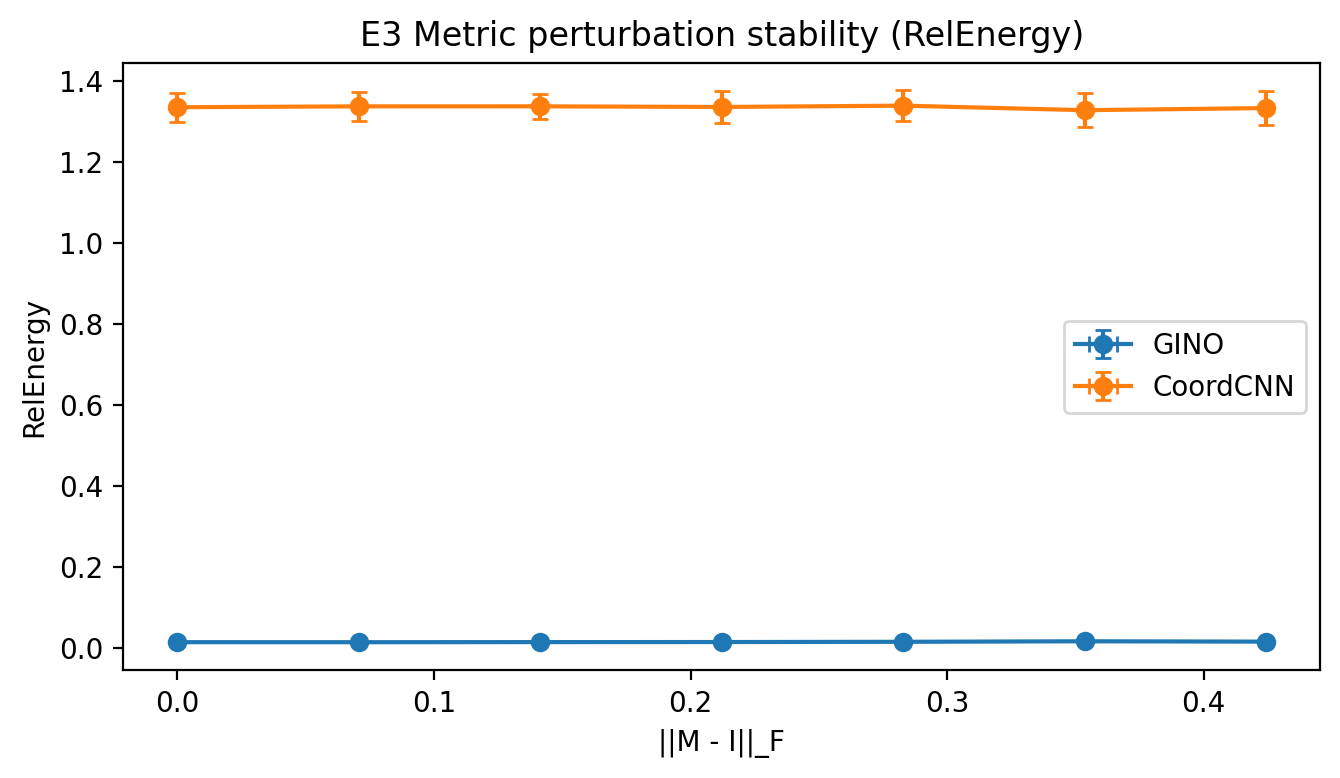}
    \includegraphics[width=0.48\linewidth]{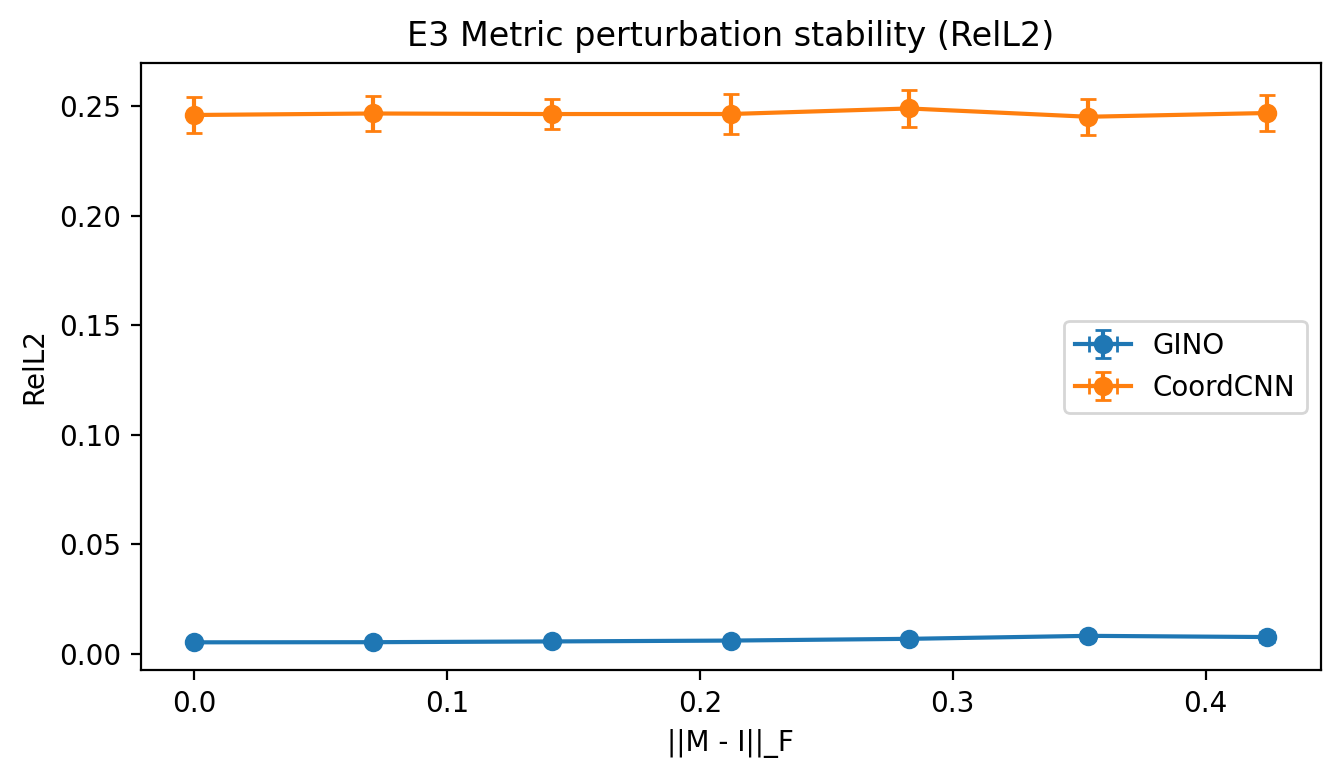}
    \caption{Metric perturbation stability (E3). GINO (blue) maintains low RelEnergy (left) and RelL2 (right) across  $ \|M - I\|_F \in [0, 0.4] $ , while CoordCNN (orange) shows large, invariant errors.}
    \label{fig:e3_relenergy}
    \label{fig:e3_rell2}
\end{figure}

\subsection{E4: Cross-Resolution Generalization and Discretization Consistency}

We assess whether learned operators commute with spectral restriction ( $ R $ ) and prolongation ( $ P $ ). In cross-resolution tests (Figures \ref{fig:e4_rell2}), GINO achieves RelL2  $ \leq 10^{-2} $  in all train/test grid combinations, while CoordCNN degrades severely under resolution shift (e.g., coarse $ \to $ fine RelL2  $ > 0.5 $ ). Furthermore, GINO satisfies discretization consistency: commutation errors remain below  $ 8 \times 10^{-3} $  (Figure \ref{fig:e4_commutation}), versus  $ > 0.6 $  for CoordCNN.

\begin{figure}[H]
    \centering
    \includegraphics[width=0.48\linewidth]{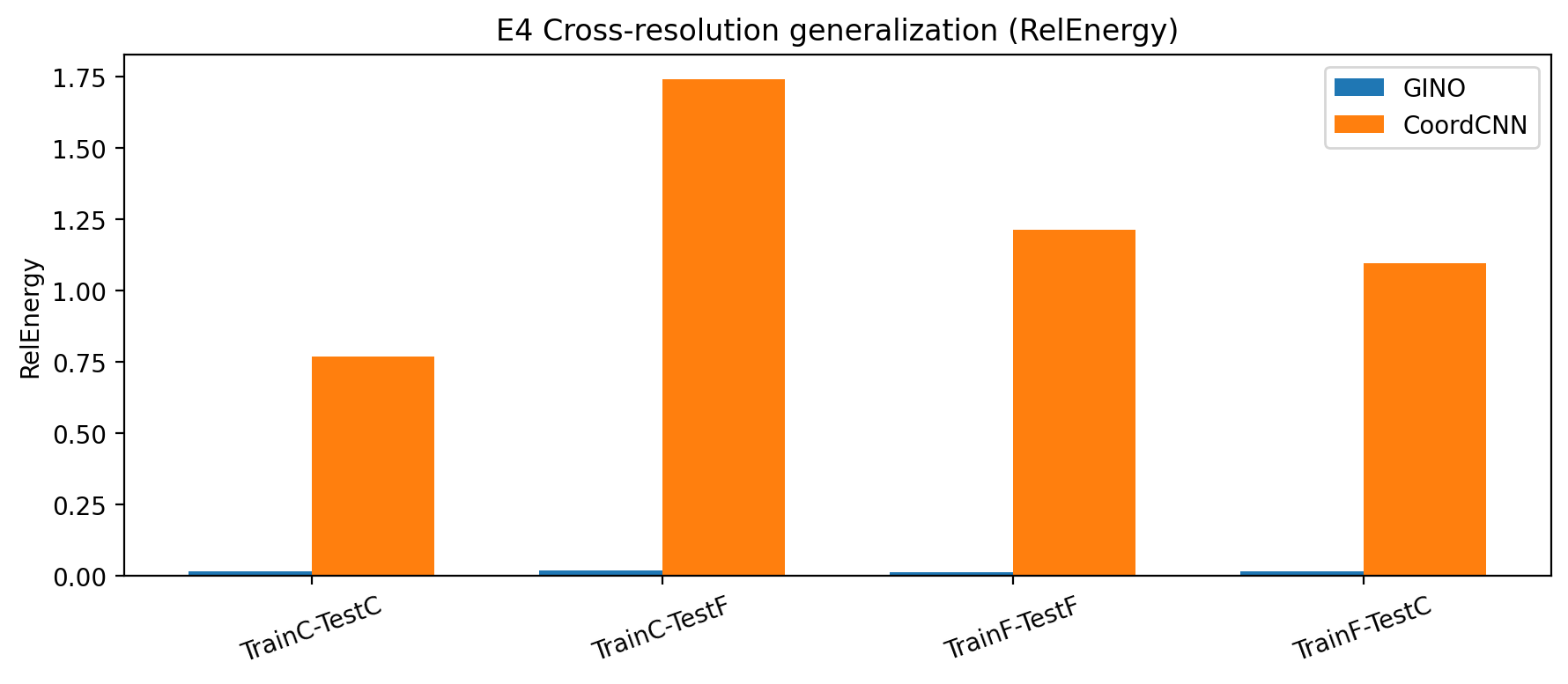}
    \includegraphics[width=0.48\linewidth]{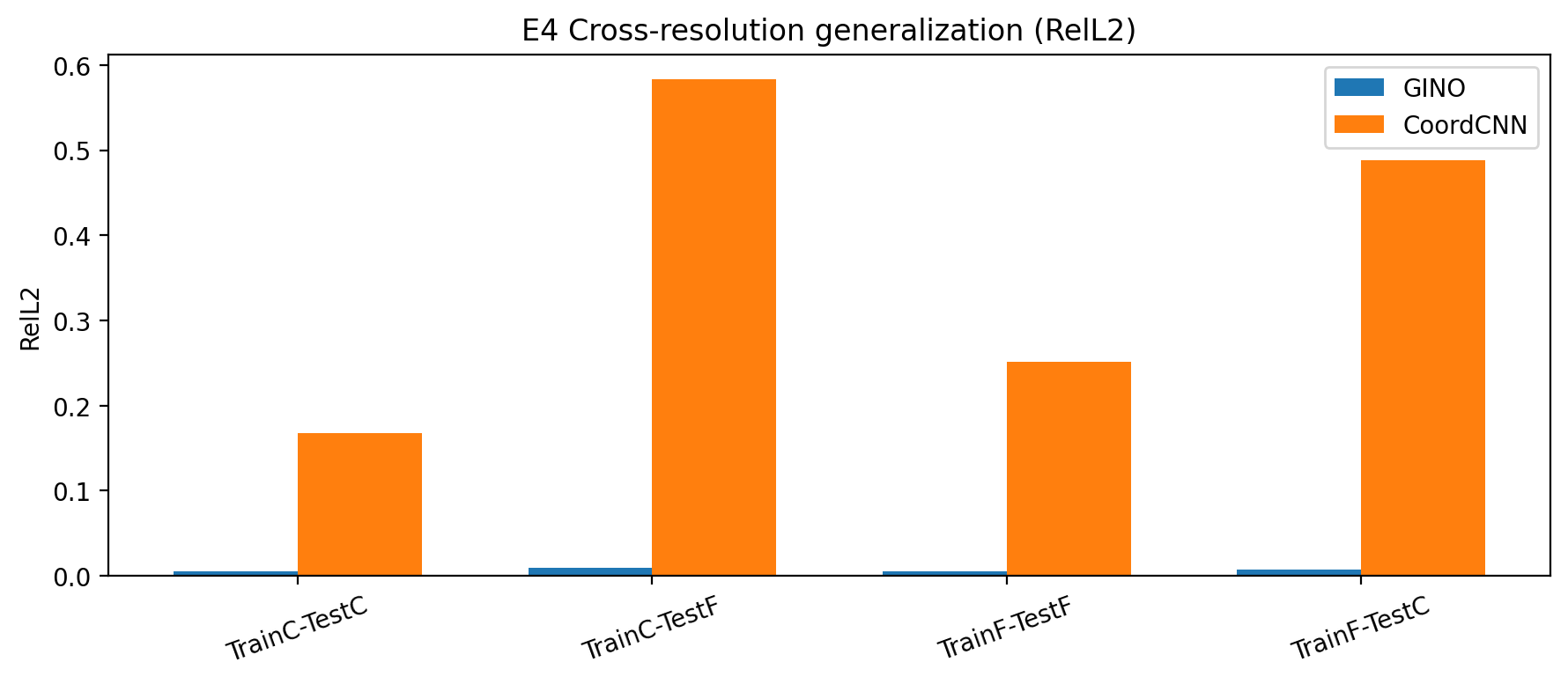}
    \caption{Cross-resolution generalization (E4). GINO generalizes across resolutions with low RelL2 (left) and RelEnergy (right); CoordCNN fails dramatically when tested on unseen grids.}
    \label{fig:e4_rell2}
    \label{fig:e4_relenergy}
\end{figure}

\begin{figure}[H]
    \centering
    \includegraphics[width=\linewidth]{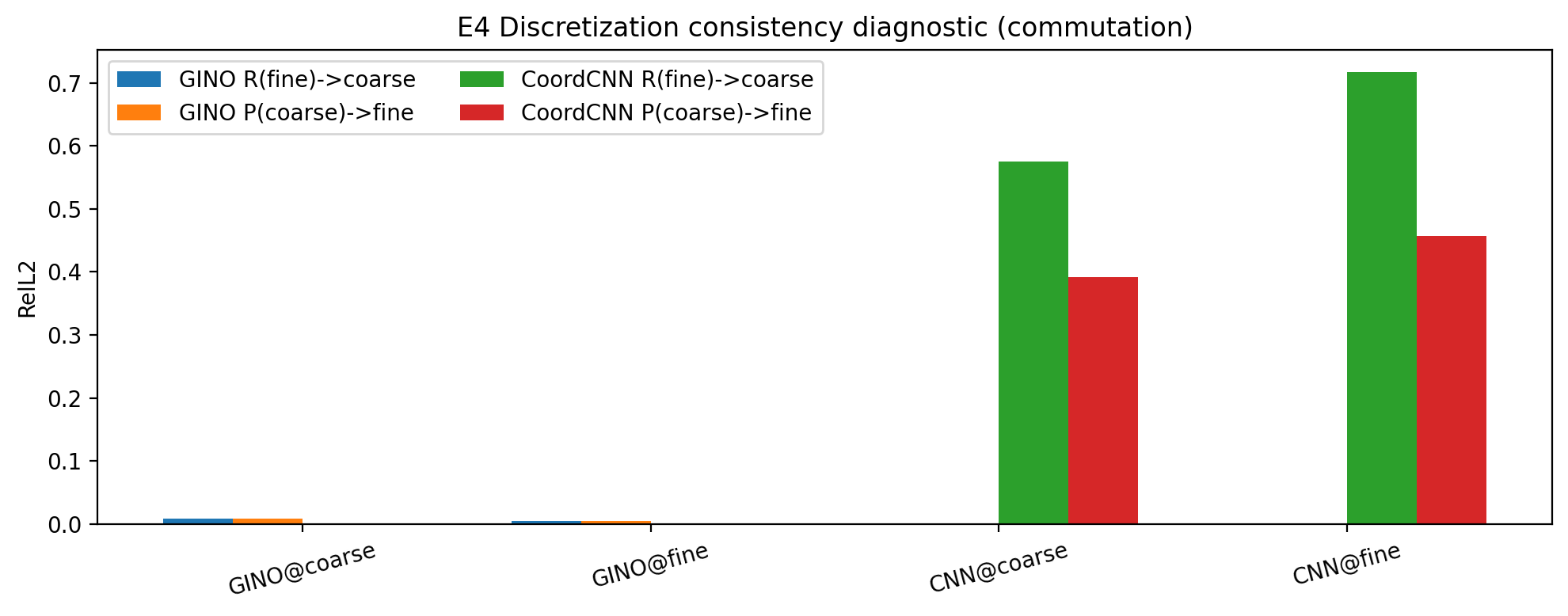}
    \caption{Discretization consistency diagnostic (E4). GINO achieves negligible commutation errors ($< 0.01$), while CoordCNN errors exceed 0.6, revealing strong grid dependence.}
    \label{fig:e4_commutation}
\end{figure}

\subsection{E5: Structure-Preserving Hodge Decomposition}

We extend GINO to learn a regularized Helmholtz--Hodge decomposition of 1-forms into exact, coexact, and harmonic components. Final performance (Figure \ref{fig:e5_summary}) shows RelL2  $ \approx 3.4 \times 10^{-2} $  for both projections and a decomposition residual of  $ 1.45 \times 10^{-2} $ . Critically, gauge equivariance error remains at  $ 3.6 \times 10^{-6} $  (Figure \ref{fig:e5_training}), confirming that the operator is both \textbf{structure-preserving} and \textbf{symmetry-respecting}.

\begin{figure}[H]
    \centering
    \includegraphics[width=\linewidth]{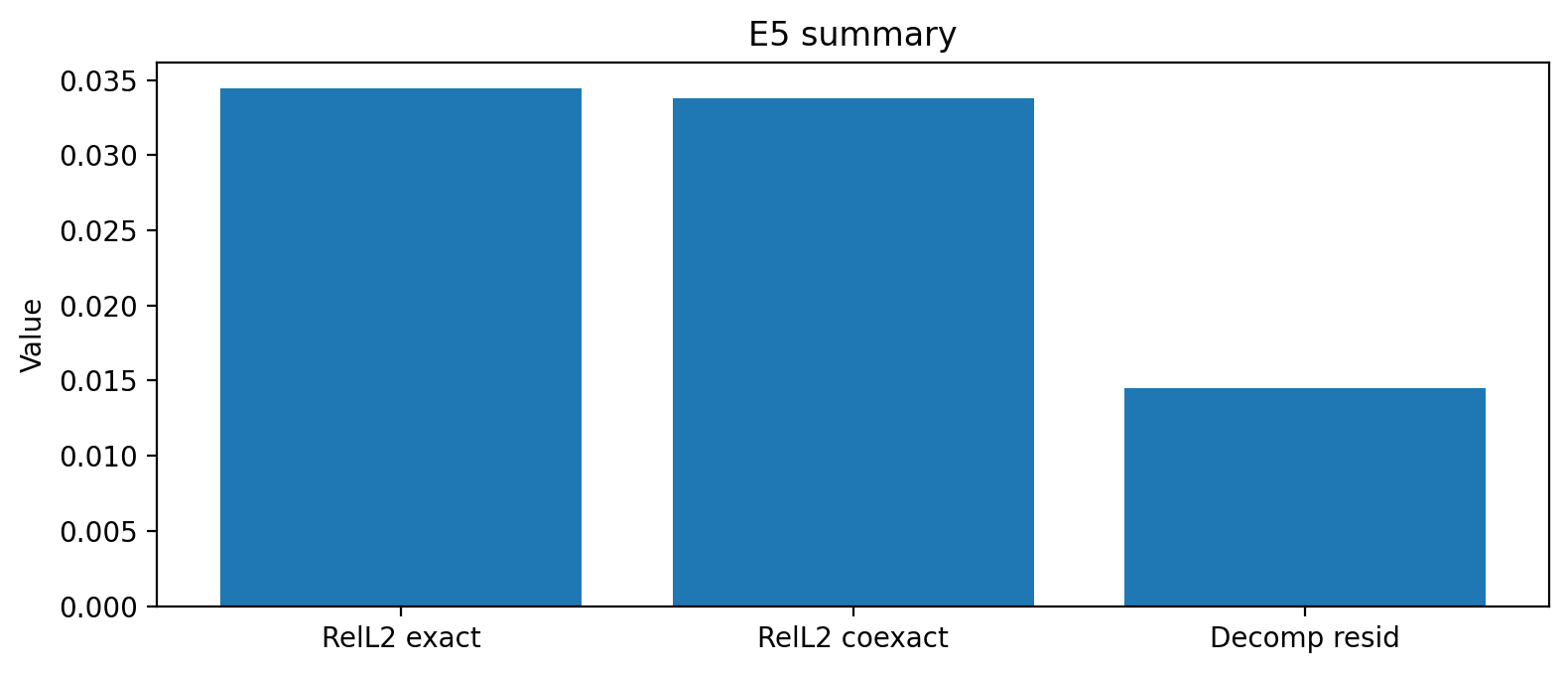}
    \caption{Summary of E5 performance. Final relative errors for exact/coexact components are  $ \approx 3.4 \times 10^{-2} $ , with decomposition residual  $ \approx 1.45 \times 10^{-2} $ .}
    \label{fig:e5_summary}
\end{figure}

\begin{figure}[H]
    \centering
    \includegraphics[width=\linewidth]{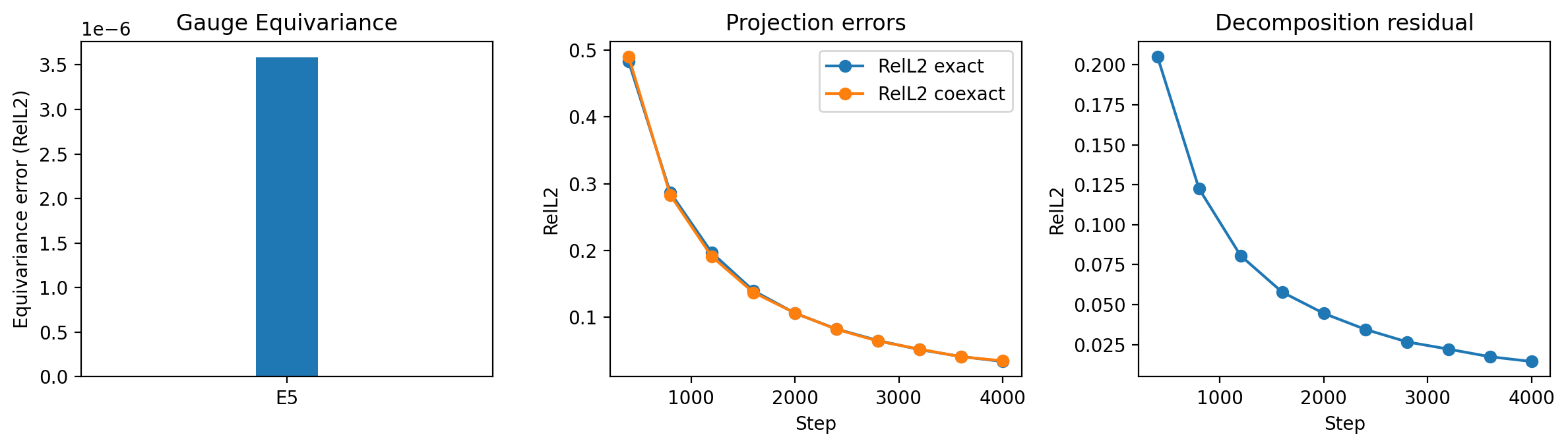}
    \caption{Training dynamics for E5. Left: gauge equivariance error stays below  $ 4 \times 10^{-6} $ . Middle/right: projection errors and decomposition residual decrease rapidly during training.}
    \label{fig:e5_training}
\end{figure}

\subsection{E6: Ablations on Spectral Truncation and Multiplier Smoothness}

E6A: Spectral support. We sweep the truncation threshold  $ \Lambda $  (Figure \ref{fig:e6a}). Error metrics exhibit a U-shaped trend, with optimal performance at  $ \Lambda \approx 100 $ . Larger  $ \Lambda $  increases the roughness proxy, indicating a capacity–regularization tradeoff.

E6B: Smoothness–stability link. Strongly regularized models maintain flat error curves across  $ \delta $  (Figure \ref{fig:e6b_stability}), while unregularized (rough) models amplify RelL2 by up to \textbf{3.4×} at  $ \delta=0.30 $ . The roughness proxy (Figure \ref{fig:e6b_roughness}) correlates strongly with instability, providing empirical evidence that multiplier smoothness mediates geometric robustness.

\begin{figure}[H]
    \centering
    \includegraphics[width=\linewidth]{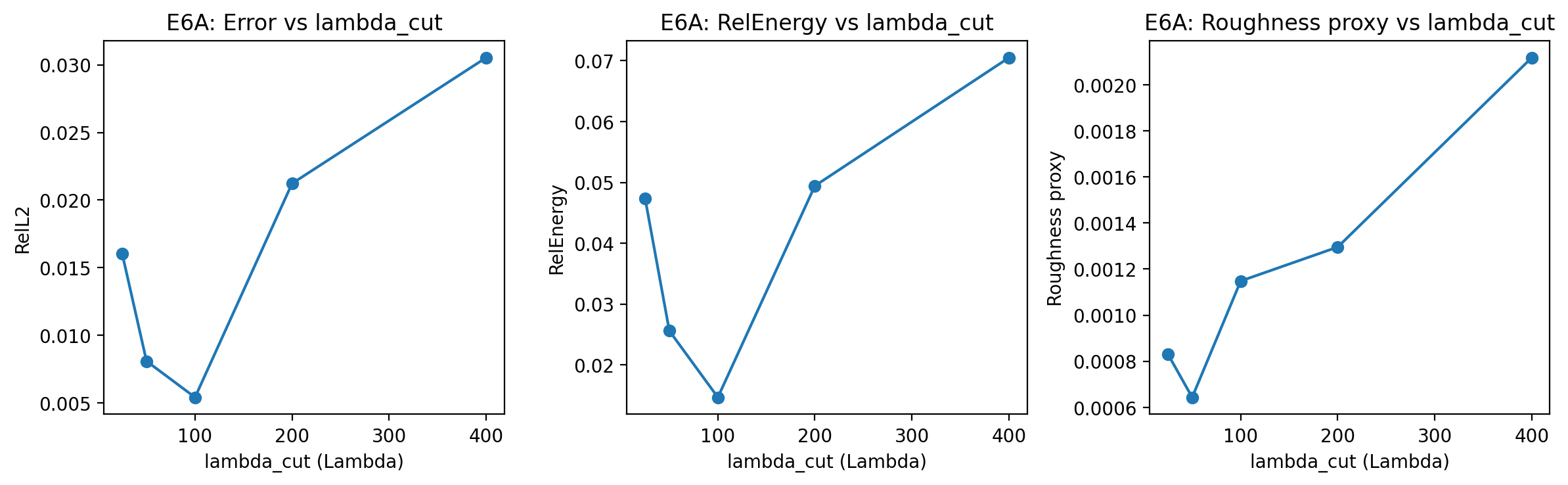}
    \caption{Ablation on spectral truncation (E6A). Error metrics (RelL2, RelEnergy) are minimized at intermediate  $ \Lambda $ , while roughness increases monotonically with  $ \Lambda $ .}
    \label{fig:e6a}
\end{figure}

\begin{figure}[htbp]
    \centering
    \includegraphics[width=0.29\linewidth]{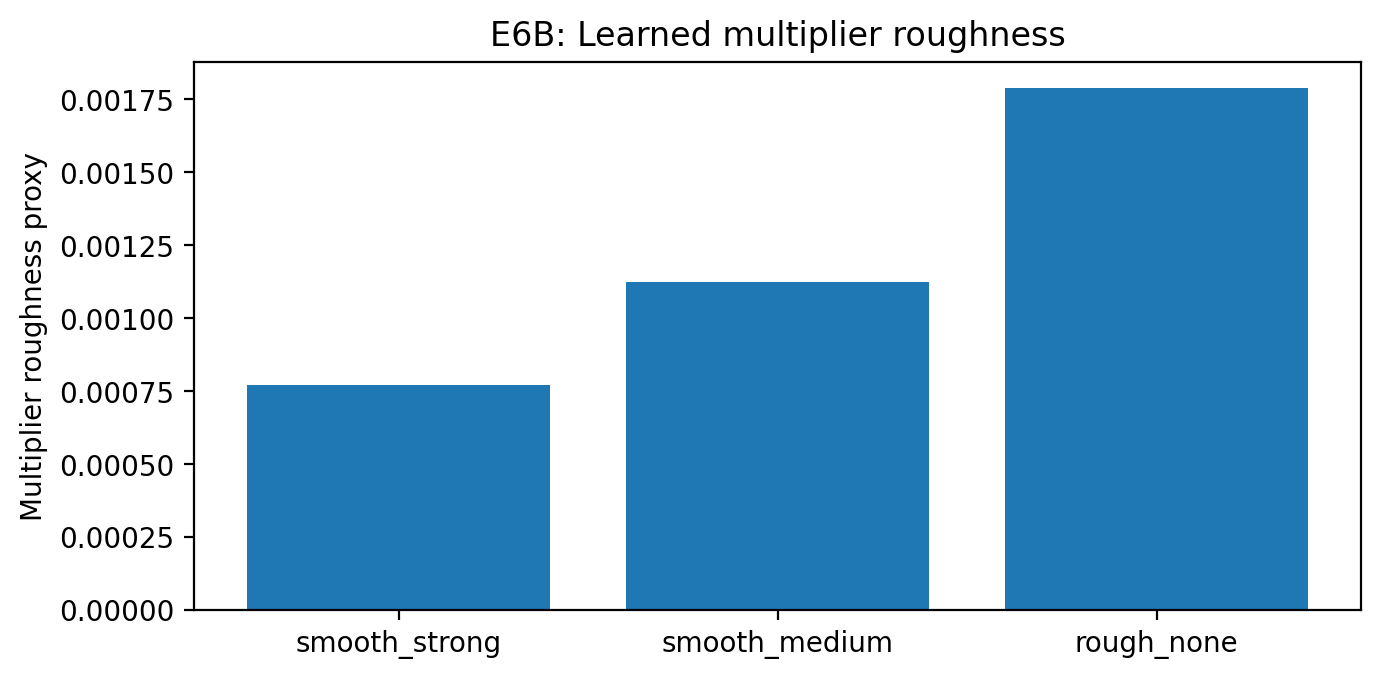}
    \includegraphics[width=0.69\linewidth]{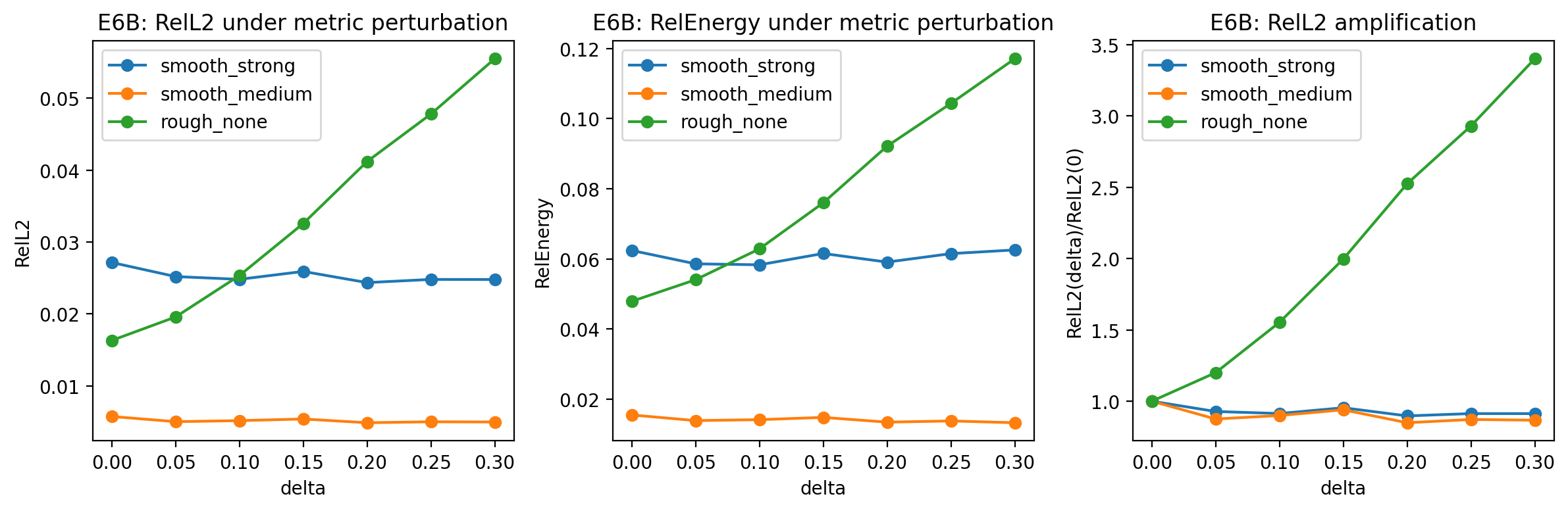}
    \caption{Effect of smoothness on geometric stability (E6B). Left: stronger regularization yields smoother multipliers. Right: unregularized models amplify error by up to 3.4× under metric perturbation, while smoothed models remain stable.}
    \label{fig:e6b_roughness}
    \label{fig:e6b_stability}
\end{figure}

\subsection{Summary}

Across six experiments, GINO consistently outperforms a standard coordinate-based baseline in accuracy, symmetry preservation, geometric stability, and discretization consistency. Crucially, ablation studies reveal that these advantages stem from principled design choices—spectral truncation and smoothness control—that align with theoretical intuitions about intrinsic operator learning. Together, these results establish GINO as a robust, geometry-aware framework for learning elliptic operators on vector-valued fields.

\section{Conclusion}
We introduced Gauge-Equivariant Intrinsic Neural Operators (GINO) for learning elliptic PDE solution maps in a manner that is consistent with geometric structure and representation changes. By parameterizing the operator through intrinsic spectral multipliers defined on geometry-dependent spectra and combining them with gauge-equivariant nonlinearities, GINO enforces frame-consistent behavior while maintaining the expressivity required for operator learning.

Experiments E1–E6 on $\mathbb{T}^2$ provide a coherent empirical validation of the approach. GINO achieves accurate operator approximation on the base resolvent task and exhibits near numerical-precision gauge equivariance, while coordinate-based CNN baselines remain strongly gauge-sensitive. Under structured metric perturbations, GINO shows minimal degradation, and under cross-resolution evaluation it maintains low error and small commutation defects with restriction/prolongation operators, indicating reduced discretization dependence. On a regularized Helmholtz–Hodge decomposition task, the structured variant attains low projection error and small decomposition residual while preserving gauge equivariance to near floating-point tolerance. Finally, theory-aligned ablations demonstrate that multiplier smoothness strongly mediates stability: rough learned multipliers amplify perturbation errors, whereas smoothness control yields markedly improved robustness.

These findings suggest that incorporating intrinsic geometry and gauge constraints is a practical route to more reliable operator surrogates for geometric PDEs. Future work includes extending the framework to spatially varying metrics and general manifolds, integrating mesh-based discretizations via discrete exterior calculus or finite element exterior calculus, and studying broader PDE families and real-world scientific benchmarks where gauge and discretization robustness are essential.

% Acknowledgements and Disclosure of Funding should go at the end, before appendices and references

\acks{This research received no external funding.}


\newpage

{\noindent \em Remainder omitted in this sample. See http://www.jmlr.org/papers/ for full paper.}

\vskip 0.2in
\bibliography{sample}

\end{document}